%% file: 0_iclr2022_main.tex
\documentclass{article} 
\usepackage{iclr2022_conference,times}

\input{math_commands.tex}

\usepackage[colorlinks=true,linkcolor=blue,citecolor=green,urlcolor=blue]{hyperref}

\newcommand{\mytitle}
{Transfer RL across Observation Feature Spaces via Model-Based Regularization}
\input{packages}
\include{macros}
\title{\mytitle}

\author{
Yanchao Sun\textsuperscript{\dag}\thanks{The work was done while the author was an intern at Unity Technologies. } 
\quad
Ruijie Zheng\textsuperscript{\dag} 
\quad
Xiyao Wang\textsuperscript{\dag} 
\quad
Andrew Cohen\textsuperscript{\ddag}
\quad
Furong Huang\textsuperscript{\dag}  \\
\textsuperscript{\dag} \text{University of Maryland, College Park} \quad
  \textsuperscript{\ddag} Unity Technologies \\
  \textsuperscript{\dag}\texttt{\{ycs,rzheng12,xywang,furongh\}@umd.edu}
  \quad
  \textsuperscript{\ddag}\texttt{andrew.cohen@unity3d.com}
}

\iclrfinalcopy 
\begin{document}
\maketitle

\begin{abstract} 
In many reinforcement learning (RL) applications, the observation space is specified by human developers and restricted by physical realizations, and may thus be subject to dramatic changes over time (e.g. increased number of observable features). However, when the observation space changes, the previous policy will likely fail due to the mismatch of input features, and another policy must be trained from scratch, which is inefficient in terms of computation and sample complexity. Following theoretical insights, we propose a novel algorithm which extracts the latent-space dynamics in the source task, and transfers the dynamics model to the target task to use as a model-based regularizer. Our algorithm works for drastic changes of observation space (e.g. from vector-based observation to image-based observation), without any inter-task mapping or any prior knowledge of the target task. Empirical results show that our algorithm significantly improves the efficiency and stability of learning in the target task.
\end{abstract}
\input{s1-intro}

\input{s2-prelim}
\input{s3-setup}
\input{s4-method}

\input{s5-related}

\input{s6-exp}
\input{s7-conclusion}
\section*{Acknowledgements}
This work is supported by Unity Technologies, National Science Foundation IIS-1850220 CRII Award 030742-00001, DOD-DARPA-Defense Advanced Research Projects Agency Guaranteeing AI Robustness against Deception (GARD), and Adobe, Capital One and JP Morgan faculty fellowships.
\input{s8-ethics}

\bibliography{references}
\bibliographystyle{iclr2022_conference}

\newpage
\appendix
{\centering{\Large Appendix: \mytitle}}
\input{a0-prelim}
\input{a3-representation}
\input{a1-proofs}
\input{a4-avi}
\input{a2-exp}

\end{document}

%% file: math_commands.tex
\usepackage{amsmath,amsfonts,bm}

\def\eqref#1{equation~\ref{#1}}

\def\1{\bm{1}}

\DeclareMathAlphabet{\mathsfit}{\encodingdefault}{\sfdefault}{m}{sl}
\SetMathAlphabet{\mathsfit}{bold}{\encodingdefault}{\sfdefault}{bx}{n}



%% file: packages.tex
\usepackage{hyperref}
\usepackage{url}
\usepackage{natbib}
\usepackage{graphicx}
\usepackage{amsfonts}
\usepackage{amsmath}
\usepackage{amsthm}
\usepackage{bbm}
\usepackage{xcolor}
\usepackage{algorithm}
\usepackage{scalerel}
\usepackage[noend]{algpseudocode}
\usepackage{natbib}
\usepackage[outdir=./]{epstopdf}
\usepackage{graphicx,float,pgfplots,wrapfig,sidecap,lipsum}

\usepackage{colortbl}
\usepackage{tablefootnote}
\usepackage[font=small,labelfont=bf]{caption}
\usepackage{subcaption}
\usepackage{subfloat}
\usepackage{tikz}
\usetikzlibrary{fit}
\usetikzlibrary{calc,shapes}
\usetikzlibrary{decorations.pathmorphing} 
\usetikzlibrary{fit}					
\usetikzlibrary{backgrounds}	
\usetikzlibrary{pgfplots.groupplots}

\usepackage[utf8]{inputenc}
\usepackage{pgfplots}
\DeclareUnicodeCharacter{2212}{−}
\usepgfplotslibrary{groupplots,dateplot}
\usetikzlibrary{patterns,shapes.arrows}
\pgfplotsset{compat=newest}

%% file: macros.tex
\newtheorem{theorem}{Theorem}
\newtheorem{lemma}[theorem]{Lemma}
\newtheorem{assumption}[theorem]{Assumption}
\newtheorem{definition}[theorem]{Definition}

\newtheorem{proposition}[theorem]{Proposition}

\newcommand{\source}{^{\scaleto{(S)}{5pt}}}
\newcommand{\target}{^{\scaleto{(T)}{5pt}}}
\newcommand{\rep}{\Phi}

\newcommand{\mdp}{M}
\newcommand{\states}{\mathcal{O}}
\newcommand{\actions}{\mathcal{A}}
\newcommand{\bellman}{\mathcal{T}}
\newcommand{\apx}{\mathcal{H}_{\phi}}
\newcommand{\apf}{h}
\newcommand{\reppolicy}{\Pi_{\phi}^D}
\newcommand{\lipsvalue}{K_{\phi,V}}
\newcommand{\lipsapx}{K_{\phi,h}}

\newcommand{\qpi}{Q_\pi}
\newcommand{\vpi}{V_\pi}
\newcommand{\ppi}{P_\pi}
\newcommand{\rpi}{R_\pi}
\newcommand{\pact}{P_a}
\newcommand{\ract}{R_a}
\newcommand{\appi}{\hat{P}_\pi}
\newcommand{\arpi}{\hat{R}_\pi}
\newcommand{\apa}{\hat{P}_a}
\newcommand{\ara}{\hat{R}_a}
\newcommand{\avpi}{\hat{V}_\pi}

\definecolor{sourcecolor}{rgb}{0.5,1,0.5}
\definecolor{ourcolor}{rgb}{1,0,0}
\definecolor{singlecolor}{rgb}{0,0,1}
\definecolor{auxcolor}{rgb}{0.54,0.17,0.89}
\definecolor{linearcolor}{rgb}{0.172549019607843,0.627450980392157,0.172549019607843}
\definecolor{randomcolor}{rgb}{1,0.498039215686275,0.0549019607843137}
\definecolor{tunecolor}{rgb}{0.9568627450980393, 0.8156862745098039,0}
\definecolor{aligncolor}{rgb}{0,0.5,0}

%% file: s1-intro.tex
\section{Introduction}
\label{sec:intro}

Deep Reinforcement Learning (DRL) has the potential to be used in many large-scale applications such as robotics, gaming and automotive.
In these real-life scenarios, it is an essential ability for agents to utilize the knowledge learned in past tasks to facilitate learning in unseen tasks, which is known as Transfer RL (TRL).
Most existing TRL works~\citep{taylor2009transfer,zhu2020transfer} focus on tasks with the same state-action space but different dynamics/reward.
However, these approaches do not apply to the case where the observation space changes significantly. 
Observation change is common in practice as in the following scenarios.
(1) Incremental environment development. RL is used to train non-player characters (NPC) in games~\citep{juliani2018unity}, which may be frequently updated. 
When there are new scenes, characters, or obstacles added to the game, the agent's observation space will change accordingly.
(2) Hardware upgrade/replacement. For robots with sensory observations~\citep{bohez2017sensor}, the observation space could change (e.g. from text to audio, from lidar to camera) as the sensor changes.
(3) Restricted data access. In some RL applications~\citep{ganesh2019reinforcement}, agent observation contains sensitive data (e.g. inventory) which may become unavailable in the future due to data restrictions. 
In these cases, the learner may have to discard the old policy and train a new policy from scratch, as the policy has a significantly different input space, even though the underlying dynamics are similar. But training an RL policy from scratch can be expensive and unstable. Therefore, there is a crucial need for a technique that transfers knowledge across tasks with similar dynamics but different observation spaces.

Besides these existing common applications, there are more benefits of across-observation transfer. For example, observations in real-world environments are usually rich and redundant, so that directly learning a policy is hard and expensive. If we can transfer knowledge from low-dimensional and informative vector observations (usually available in a simulator) to richer observations, the learning efficiency can be significantly improved. Therefore, an effective transfer learning method enables many novel and interesting applications, such as curriculum learning via observation design.

\begin{figure}[!htbp]
\centering
    \includegraphics[width=0.9\textwidth]{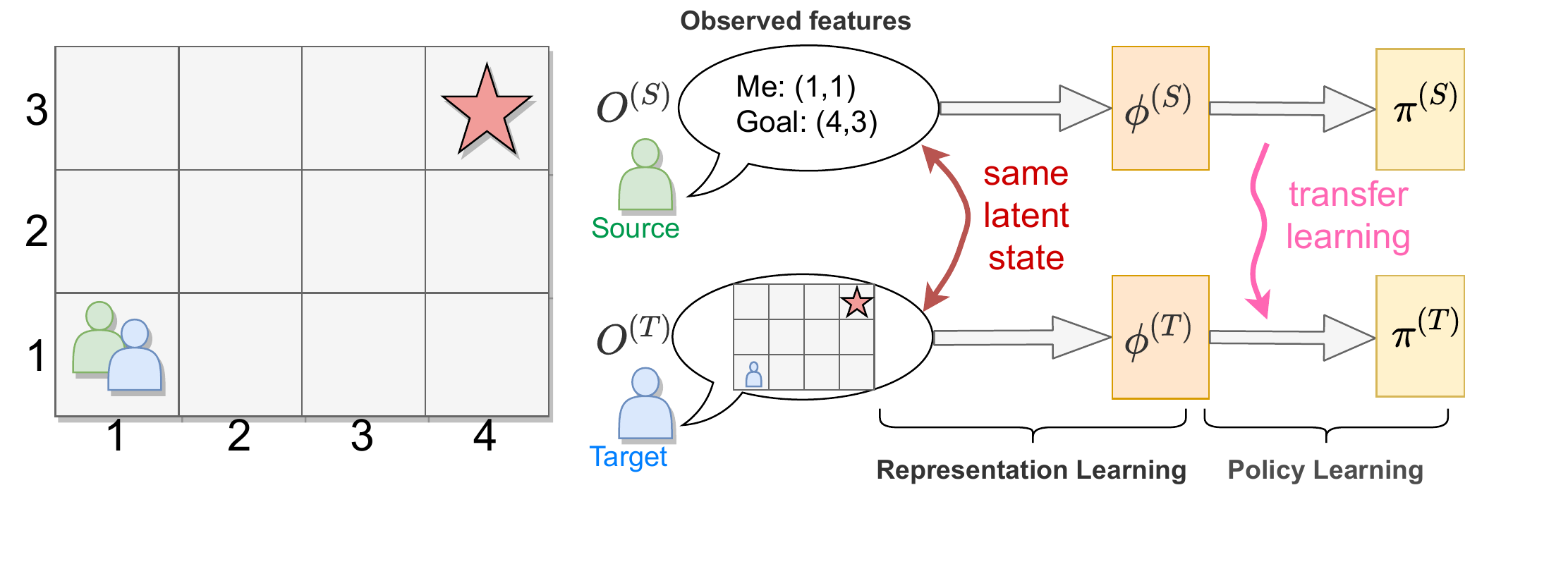}
\vspace{-2em}
\caption{\small{An example of the transfer problem with changed observation space.
The source-task agent observes the x-y coordinates of itself and the goal, while the target-task agent observes a top-down view/image of the whole maze. The two observation spaces are drastically different, but the two tasks are structurally similar. Our goal is to transfer knowledge from the source task to accelerate learning in the target task, without knowing or learning any inter-task mapping.
}}
\label{fig:example}
\vspace{-1em}
\end{figure}




In this paper, we aim to fill the gap and propose a new algorithm that can automatically transfer knowledge from the old environment to facilitate learning in a new environment with a (drastically) different observation space.
In order to meet more practical needs, we focus on the challenging setting where the observation change is: 
\textbf{(1) unpredictable} (there is no prior knowledge about how the observations change), 
\textbf{(2) drastic} (the source and target tasks have significantly different observation feature spaces, e.g., vector to image), and
\textbf{(3) irretrievable} (once the change happens, it is impossible to query the source task, so that the agent can not interact with both environments simultaneously).
Note that different from many prior works ~\citep{taylor2007transfer,mann2013directed}, we do not assume the knowledge of any inter-task mapping. That is, the agent does not know which new observation feature is corresponding the which old observation feature.
To remedy the above challenges and achieve knowledge transfer, we make a key observation that, if only the observation features change, the source and target tasks share the same latent space and dynamics
(e.g. in Figure~\ref{fig:example}, $\states\source$ and $\states\target$ can be associated to the same latent state).
Therefore, we first disentangle representation learning from policy learning, and then accelerate the target-task agent by regularizing the representation learning process with the latent dynamics model learned in the source task. We show by theoretical analysis and empirical evaluation that the target task can be learned more efficiently with our proposed transfer learning method than from scratch.


\textbf{Summary of Contributions.} 
(1) To the best of our knowledge, we are the first to discuss the transfer problem where the source and target tasks have drastically different observation feature spaces, and there is no prior knowledge of an inter-task mapping.
(2) We theoretically characterize what constitutes a ``good representation'' and analyze the sufficient conditions the representation should satisfy.
(3) Theoretical analysis shows that a model-based regularizer enables efficient representation learning in the target task. Based on this, we propose a novel algorithm that automatically transfers knowledge across observation representations.
(4) Experiments in 7 environments show that our proposed algorithm significantly improves the learning performance of RL agents in the target task.

%% file: s2-prelim.tex
\vspace{-1em}
\section{Preliminaries and Background}
\label{sec:prelim}

\textbf{Basic RL Notations.}
An RL task can be modeled by a Markov Decision Process (MDP)~\citep{puterman2014markov}, defined as a tuple $\mdp=\langle \states,\actions,P,R,\gamma \rangle$, where $\states$ is the state/observation space, $\actions$ is the action space, $P$ is the transition kernel, $R$ is the reward function and $\gamma$ is the discount factor. 
At timestep $t$, the agent observes state $o_t$, takes action $a_t$ based on its \textit{policy} $\pi: \states \to \Delta(\actions)$(where $\Delta(\cdot)$ denotes the space of probability distributions), and receives reward $r_t=R(o_t,a_t)$. The environment then proceeds to the next state $o_{t+1} \sim P(\cdot|o_t, a_t)$.
The goal of an RL agent is to find a policy $\pi$ in the policy space $\Pi$ with the highest cumulative reward, which is characterized by the value functions.
The \textit{value} of a policy $\pi$ for a state $o\in\states$ is defined as $V^\pi(o) = \mathbb{E}_{\pi,P}[\sum_{t=0}^\infty \gamma^t r_t|o_{0}=o]$. The \textit{Q value} of a policy $\pi$ for a state-action pair $(o,a) \in \states \times \actions$ is defined as $Q^\pi(o,a) = \mathbb{E}_{\pi,P}[\sum_{t=0}^\infty \gamma^t r_t|o_0=o, a_0=a]$. Appendix~\ref{app:prelim} provides more background of RL.


\textbf{Representation Learning in RL.}
Real-world applications usually have large observation spaces for which function approximation is needed to learn the value or the policy. 
However, directly learning a policy over the entire observation space could be difficult, as there is usually redundant information in the observation inputs. A common solution is to map the large-scale observation into a smaller representation space via a \textit{non-linear encoder} (also called a \textit{representation mapping}) $\phi: \states \to \mathbb{R}^d$, where $d$ is the representation dimension, 
and then learn the policy/value function over the representation space $\phi(\states)$. In DRL, the encoder and the policy/value are usually jointly learned.

%% file: s3-setup.tex
\vspace{-1.5em}
\section{Problem Setup: Transfer Across Different Observation Spaces}
\label{sec:setup}
\vspace{-0.5em}

We aim to transfer knowledge learned from a source MDP to a target MDP, whose observation spaces are different while dynamics are structurally similar.
Denote the source MDP as $\mathcal{M}\source = \langle \states\source, \actions, P\source, R\source, \gamma \rangle$, and the target MDP as $\mathcal{M}\target = \langle \states\target, \actions, P\target, R\target, \gamma \rangle.$ 
Note that $\states\source$ and $\states\target$ can be significantly different, such as $\states\source$ being a low-dimensional vector space and $\states\target$ being a high-dimensional pixel space, which is challenging for policy transfer since the source target policy have different input shapes and would typically be very different architecturally. 

In this work, as motivated in the Introduction, we focus on the setting wherein the dynamics (($P\source, R\source$) and ($P\target, R\target$)) of the two MDPs between which we transfer knowledge are defined on different observation spaces but share structural similarities. Specifically, we make the assumption that there exists a mapping between the source and target observation spaces such that the transition dynamics under the mapping in the target task share the same transition dynamics as in the source task. We formalize this in Assumption~\ref{assum:similarity}:

\begin{assumption}
\label{assum:similarity}
There exists a function $f: \states\target \to \states\source$ such that $\forall o_i\target, o_j\target \in \states\target, \forall a\in\actions$,
\begin{equation*}
    P\target(o_j\target|o_i\target,a) = P\source(f(o_j\target) | f(o_i\target), a), \quad 
    R\target(o_i\target,a) = R\source(f(o_i\target),a).
\end{equation*}
\end{assumption}
\textbf{Remarks.} 
(1) Assumption~\ref{assum:similarity} is mild as many real-world scenarios fall under this assumption. For instance, when upgrading the cameras of a patrol robot to have higher resolutions, such a mapping $f$ can be a down-sampling function.
(2) $f$ is a general function without extra restrictions. $f$ can be a many-to-one mapping, i.e., more than one target observations can be related to the same observation in the source task. $f$ can be non-surjective, i.e., there could exist source observations that do not correspond to any target observation.


Many prior works~\citep{mann2013directed,brys2015policy} have similar assumptions, but require prior knowledge of such an inter-task mapping to achieve knowledge transfer. However, such a mapping might not be available in practice. 
As an alternative, we propose a novel transfer algorithm in the next section that \textit{does not} assume any prior knowledge of the mapping $f$. The proposed algorithm learns a latent representation of the observations and a dynamics model in this latent space, and then the dynamics model is transferred to speed up learning in the target task. 



%% file: s4-method.tex
\section{Methodology: Transfer with Regularized Representation}
\label{sec:method}

In this section, we first formally characterize \textit{``what a good representation is for RL''} in Section~\ref{sec:representation}, then introduce our proposed transfer algorithm based on representation regularization  in Section~\ref{sec:algo}, and next provide theoretical analysis of the algorithm in Section~\ref{sec:theory}.

\subsection{Characterizing Conditions for Good Representations}
\label{sec:representation} 

As discussed in Section~\ref{sec:prelim}, real-world applications usually have rich and redundant observations, where learning a good representation~\citep{jaderberg2016reinforcement,dabney2020the} is essential for efficiently finding an optimal policy.
However, the properties that constitute a good representation for an RL task are still an open question.
Some prior works~\citep{bellemare2019geometric,dabney2020the,gelada2019deepmdp} have discussed the representation quality in DRL, but we take a different perspective and focus on characterizing the sufficient properties of representation for learning a task.

Given a representation mapping $\phi$, the Q value of any $(o,a)\in\states\times\actions$ can be approximately represented by a function of $\phi(o)$, i.e., $\hat{Q}(o, a) = h(\phi(o); \theta_a)$, where $h$ is a function parameterized by $\theta_a$.
To study the relation between representation quality and approximation quality,
we define an \textit{approximation operator} $\apx$, which finds the best Q-value approximation based on $\phi$. Formally, let $\Theta$ denote the parameter space of function $h\in\mathcal{H}$, then $\forall a\in\actions$, $\apx Q(o,a) := h(\phi(o);\theta^*_a)$, where $\theta^*_a=\mathrm{argmin}_{\theta\in\Theta} \mathbb{E}_o [\| h(\phi(o);\theta) - Q(\phi(o),a) \| ]$. Such a function $h$ can be realized by neural networks as universal function approximators~\citep{hornik1989multilayer}.
Therefore, the value approximation error $\|Q-\apx Q\|$ only depends on the representation quality, i.e., whether we can represent the Q value of any state $o$ as a function of the encoded state $\phi(o)$.

The quality of the encoder $\phi$ is crucial for learning an accurate value function or learning a good policy. The ideal encoder $\phi$ should discard irrelevant information in the raw observation but keep essential information. 
In supervised or self-supervised representation learning~\citep{chen2020simple,achille2018emergence}, it is believed that a good representation $\phi(X)$ of input $X$ should contain minimal information of $X$ which maintaining sufficient information for predicting the label $Y$.
However, in RL, it is difficult to identify whether a representation is sufficient, since there is no label corresponding to each input. 
The focus of an agent is to estimate the value of each input $o\in\states$, which is associated with some policy. Therefore, we point out that the representation quality in RL is \emph{policy-dependent}. 
Below, we formally characterize the sufficiency of a representation mapping in terms of a fixed policy and learning a task. 

\textbf{Sufficiency for A Fixed Policy.} If the agent is executing a fixed policy, and its goal is to estimate the expected future return from the environment, then a representation is sufficient for the policy as long as it can encode the policy value $\vpi$. A formal definition is provided by Definition~\ref{def:suf_policy} in Appendix~\ref{app:repre}.


\textbf{Sufficiency for Learning A Task.} The goal of RL is to find an optimal policy. Therefore, it is not adequate for the representation to only fit one policy.
Intuitively, a representation mapping is sufficient for learning if we are able to find an optimal policy over the representation space $\phi(\states)$, which requires multiple iterations of policy evaluation and policy improvement.
Definition~\ref{def:reppolicy} below defines a set of ``important'' policies for learning with $\phi(\states)$.
\begin{definition}[Encoded Deterministic Policies]
\label{def:reppolicy}
For a given representation mapping $\phi(\cdot)$, define an encoded deterministic policy set $\reppolicy$ as the set of policies that are deterministic and take the same actions for observations with the same representations. Formally,
\begin{equation}
    \reppolicy := \{ \pi \in \Pi \quad | \quad \exists \tilde{\pi}: \phi(\states) \to \actions \text{ s.t. } \forall o\in\states, \pi(o) = \tilde{\pi}(\phi(o))  \},
\end{equation}
where $\tilde{\pi}$ is a mapping from the representation space to the action space.
\end{definition}
A policy $\pi$ is in $\reppolicy$ if it does not distinguish $o_1$ and $o_2$ when $\phi(o_1)=\phi(o_2)$. Therefore, $\reppolicy$ can be regarded as deterministic policies that make decisions for encoded observations. 
Now, we define the concept of sufficient representation for learning in an MDP.
\begin{definition}[Sufficient Representation for Learning]
\label{def:suf_learn}
A representation mapping $\phi$ is \textbf{sufficient} for a task $\mdp$ w.r.t. approximation operator $\apx$ if $\apx \qpi = \qpi$ for all $\pi\in\reppolicy$. Furthermore, \\
$\bullet$ $\phi$ is \textbf{linearly-sufficient} for learning $\mdp$ if $\exists\theta_a$ s.t. $\qpi(o,a)=\phi(o)^\top\theta_a$, $\forall a\in\actions, \pi\in\reppolicy$. \\
$\bullet$ $\phi$ is $\boldsymbol{\epsilon}$\textbf{-sufficient} for learning $\mdp$ if $\|\apx\qpi - \qpi\| \leq \epsilon$, $\forall \pi\in\reppolicy$.
\end{definition}

Definition~\ref{def:suf_learn} suggests that the representation is sufficient for learning a task as long as it is sufficient for policies in $\reppolicy$.
Then, the lemma below justifies that a nearly sufficient representation can ensure that approximate policy iteration converges to a near-optimal solution. (See Appendix~\ref{app:avi} for analysis on approximate value iteration.)

\begin{lemma}[Error Bound for Approximate Policy Iteration]
\label{lem:bound_pi}
If $\phi$ is $\epsilon$-sufficient for task $\mdp$ (with $\ell_\infty$ norm), then the approximated policy iteration with approximation operator $\apx$ starting from any initial policy that is encoded by $\phi$ ($\pi_0\in\reppolicy$) satisfies
\setlength\abovedisplayskip{-2pt}
\setlength\belowdisplayskip{2pt}
\begin{equation}
    \limsup_{k\to\infty} \|Q^* - Q^{\pi_k} \|_{\infty} \leq \frac{2\gamma^2\epsilon}{(1-\gamma)^2},
\end{equation}
where $\pi_k$ is the policy in the $k$-th iteration. 
\end{lemma}
Lemma~\ref{lem:bound_pi}, proved in Appendix~\ref{app:proofs}, is extended from the error bound provided by~\citet{BertsekasTsitsiklis96}. For simplicity, we consider the bound in $\ell_\infty$, but tighter bounds can be derived with other norms~\citep{munos2005error}, although a tighter bound is not the focus of this paper.

\textbf{How Can We Learn A Sufficient Representation?}
So far we have provided a principle to define whether a given representation is sufficient for learning. In DRL, the representation is learned together with the policy or value function using neural networks, but the quality of the representation may be poor~\citep{dabney2020the}, which makes it hard for the agent to find an optimal policy.
Based on Definition~\ref{def:suf_learn}, a natural method to learn a good representation is to let the representation fit as many policy values as possible as auxiliary tasks, which matches the ideas in other works. For example, \citet{bellemare2019geometric} propose to fit a set of representative policies (called \textit{adversarial value functions}). 
\citet{dabney2020the} choose to fit the values of all past policies (along the \textit{value improvement path}), which requires less computational resource.
Different from these works that directly fit the value functions of multiple policies, in Section~\ref{sec:algo}, we propose to \emph{fit and transfer an auxiliary policy-independent dynamics model}, which is an efficient way to achieve sufficient representation for learning and knowledge transfer, as theoretically justified in Section~\ref{sec:theory}.

\input{algo}

\subsection{Algorithm: Learning and Transferring Model-based Regularizer}
\label{sec:algo}
Our goal is to use the knowledge learned in the source task to learn a good representation in the target task, such that the agent learns the target task more easily than learning from scratch.
Since we focus on developing a generic transfer mechanism, the base learner can be any DRL algorithms. We use $L_{\text{base}}$ to denote the loss function of the base learner. 

As motivated in Section~\ref{sec:representation}, we propose to learn policy-independent dynamics models for producing high-quality representations: (1) $\hat{P}$ which predicts the representation of the next state based on current state representation and action, and (2) $\hat{R}$ which predicts the immediate reward based on current state representation and action. 
For a batch of $N$ transition samples $\{o_i, a_i, o^\prime_i, r_i\}_{i=1}^N$, define the transition loss and the reward loss as:
\setlength\abovedisplayskip{0pt}
\setlength\belowdisplayskip{2pt}
\begin{equation}
    L_P(\phi,\hat{P}) = \frac{1}{N} \sum_{i=1}^N ( \hat{P}(\phi(o_i), a_i) - \bar{\phi}(o^\prime_i) )^2, \quad
    L_R(\phi,\hat{R}) = \frac{1}{N} \sum_{i=1}^N ( \hat{R}(\phi(o_i), a_i) - r_i )^2 \label{eq:model_loss} 
\end{equation}
where $\bar{\phi}(o^\prime_i)$ denotes the representation of the next state $o^\prime_i$ with stop gradients. 
In order to fit a more diverse state distribution, transition samples are drawn from an off-policy buffer, which stores shuffled past trajectories.

\begin{wrapfigure}{r}{0.45\textwidth}
\vspace{-1em}
  \centering
  \includegraphics[width=0.45\textwidth]{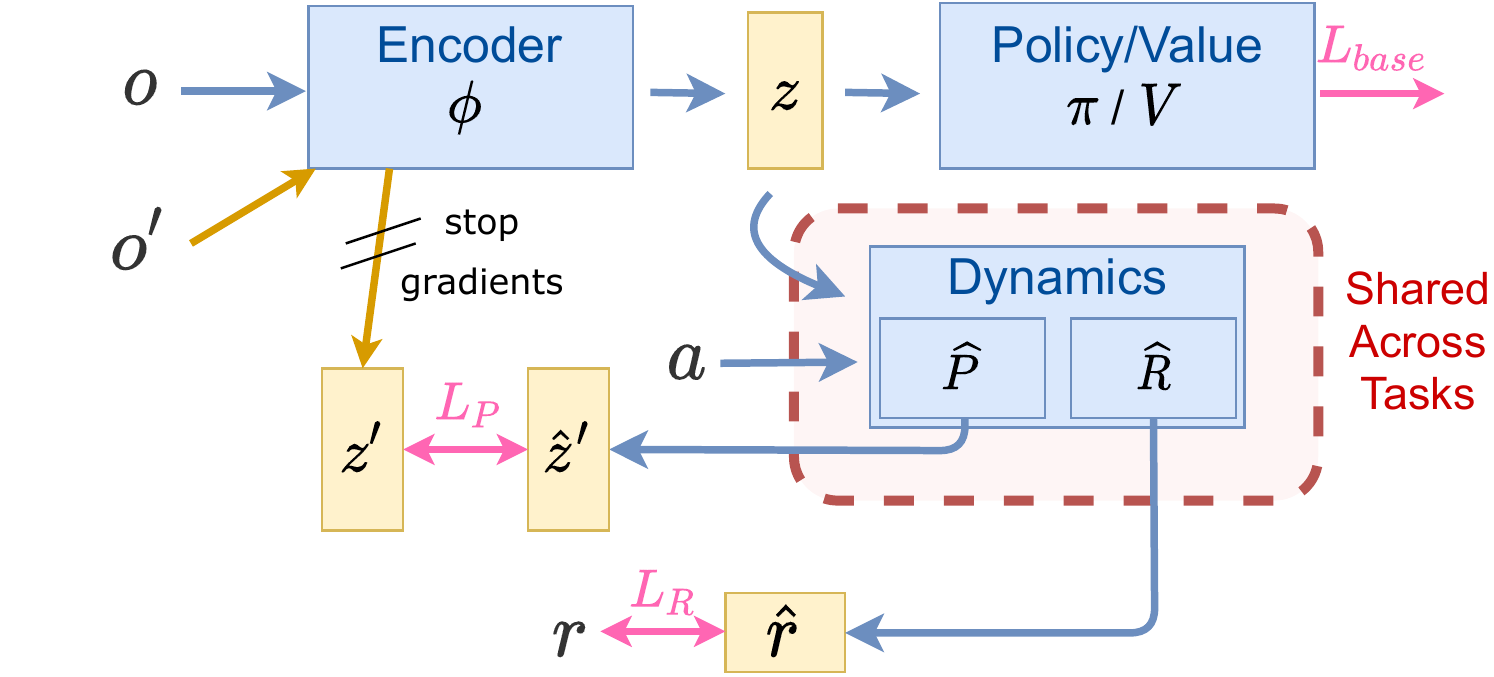}
  \vspace{-1.5em}
  \caption{\small{The architecture of proposed method. $\hat{P}$ and $\hat{R}$ are learned in the source task, then transferred to the target task and fixed during training.}}
  \label{fig:model}
\vspace{-1em}
\end{wrapfigure}

The learning procedures for the source task and the target task are illustrated in Algorithm~\ref{alg:source} and Algorithm~\ref{alg:target}, respectively. Figure~\ref{fig:model} depicts the architecture of the learning model for both source and target tasks. 
$z=\phi(o)$ and $z^\prime=\bar{\phi}(o^\prime)$ are the encoded observation and next observation.
Given the current encoding $z$ and the action $a$, the dynamics models $\hat{P}$ and $\hat{R}$ return the predicted next encoding $\hat{z}^\prime=\hat{P}(z, a)$ and predicted reward $\hat{r}=\hat{R}(z, a)$.
Then the transition loss is the mean squared error (MSE) between $z^\prime$ and $\hat{z}^\prime$ in a batch; the reward loss is the MSE between $r$ and $\hat{r}$ in a batch.

\textbf{In the source task (Algorithm~\ref{alg:source}):} 
dynamics models $\hat{P}$ and $\hat{R}$ are learned by minimizing $L_P$ and $L_R$, which are computed based on a recent copy of encoder called stable encoder $\hat{\phi}\source$ (Line 5).
The computation of the stable encoder is to help the dynamics models converge, as the actual encoder $\phi\source$ changes at every step. Note that a stable copy of the network is widely used in many DRL algorithms (e.g. the target network in DQN), which can be directly regarded as $\hat{\phi}\source$ without maintaining an extra network. 
The actual encoder $\phi\source$ is regularized by the auxiliary dynamics models $\hat{P}$ and $\hat{R}$ (Line 6).

\textbf{In the target task (Algorithm~\ref{alg:target}):} 
dynamics model $\hat{P}$ and $\hat{R}$ are transferred from the source task and fixed during learning. Therefore, the learning of $\phi\target$ is regularized by static dynamics models, which leads to faster and more stable convergence than naively learning an auxiliary task.

\textbf{Relation and Difference with Model-based RL and Bisimulation Metrics.}
Learning a dynamics model is a common technique in model-based RL~\citep{kipf2019contrastive,grimm2020value}, whose goal is to learn an accurate world model and use the model for planning. The dynamics model could be learned on either raw observations or representations. In our framework, we also learn a dynamics model, but the model serves as an auxiliary task, and learning is still performed by the model-free base learner with $L_{\mathrm{base}}$. Bisimulation methods~\citep{castro2020scalable,zhang2020learning} aim to approximate the bisimulation distances among states by learning dynamics models, whereas we do not explicitly measure the distance among states. 
Note that we also do not require a reconstruction loss that is common in literature~\citep{lee2019stochastic}.

\subsection{Theoretical Analysis: Benefits of Transferable Dynamics Model}
\label{sec:theory}

The algorithms introduced in Section~\ref{sec:algo} consist of two designs: learning a latent dynamics model as an auxiliary task, and transferring the dynamics model to the target task.
In this section, we show theoretical justifications and practical advantages of our proposed method. 
We aim to answer the following two questions: \textit{(1) How does learning an auxiliary dynamics model help with representation learning? (2) Is the auxiliary dynamics model transferable?}


For notational simplicity, let $\pact$ and $\ract$ denote the transition and reward functions associated with action $a\in\actions$.
Note that $\pact$ and $\ract$ are independent of any policy. 
We then define the sufficiency of a representation mapping w.r.t. dynamics models as below.
\begin{definition}[Policy-independent Model Sufficiency]
\label{def:model_suf}
For an MDP $\mdp$, a representation mapping $\phi$ is sufficient for its dynamics $(\pact,\ract)_{a\in\actions}$ if $\forall a \in \actions$, there exists functions $\apa: \mathbb{R}^d \to \mathbb{R}^d$ and $\ara: \mathbb{R}^d \to \mathbb{R}$ such that $\forall o\in\states$, $\apa(\phi(o))=\mathbb{E}_{o^\prime\sim \pact(o)}[\phi(o^\prime)]$, $\ara(\phi(o)) = \ract(o)$.
\end{definition}
\vspace{-0.5em}
\textbf{Remarks.} 
(1) $\phi$ is exactly sufficient for dynamics $(\pact,\ract)_{a\in\actions}$ when the transition function $P$ is deterministic. 
(2) If $P$ is stochastic, but we have $\max_{o,a}\| \mathbb{E}_{o^\prime\sim \pact(o)}[\phi(o^\prime)] - \apa(\phi(o))\|\leq\epsilon_P$ and $\max_{o,a}|\ract(o)-\ara(\phi(o))|\leq\epsilon_R$, then $\phi$ is \textbf{$\bm{(\epsilon_P,\epsilon_R)}$-sufficient} for the dynamics of $\mdp$.

Next we show by Proposition~\ref{prop:nonlin_suf} and Theorem~\ref{thm:model_learn_error} that learning sufficiency can be achieved via ensuring model sufficiency.

\begin{proposition}[Learning Sufficiency Induced by Policy-independent Model Sufficiency]
\label{prop:nonlin_suf}
Consider an MDP $\mdp$ with deterministic transition function $P$ and reward function $R$. If $\phi$ is sufficient for $(\pact, \ract)_{a\in\actions}$, then it is sufficient (but not necessarily linearly sufficient) 
for learning in $\mdp$.
\end{proposition}
Proposition~\ref{prop:nonlin_suf} shows that, if the transition is deterministic and the model errors $L_P,L_R$ are zero, then $\phi$ is exactly sufficient for learning.
More generally, if the transition function $P$ is not deterministic, and model fitting is not perfect, the learned representation can still be nearly sufficient for learning as characterized by Theorem~\ref{thm:model_learn_error} below, which is extended from a variant of the value difference bound derived by~\citet{gelada2019deepmdp}. 
Proposition~\ref{prop:nonlin_suf} and Theorem~\ref{thm:model_learn_error} justify that learning the latent dynamics model as an auxiliary task encourages the representation to be sufficient for learning. The model error $L_P$ and $L_R$ defined in Section~\ref{sec:algo} can indicate how good the representation is.
\begin{theorem}
\label{thm:model_learn_error}
For an MDP $\mdp$, if representation mapping $\phi$ is $(\epsilon_P,\epsilon_R)$-sufficient for the dynamics of $\mdp$,
then approximate policy iteration with approximation operator $\apx$ starting from any initial policy $\pi_0\in\reppolicy$ satisfies 
\vspace{-0.5em}
\begin{equation}
    \limsup_{k\to\infty} \|Q^* - Q^{\pi_k} \|_{\infty} \leq \frac{2\gamma^2}{(1-\gamma)^3}(\epsilon_R+\gamma\epsilon_P\lipsvalue).
\end{equation}
where $\lipsvalue$ is an upper bound of the value Lipschitz constant as defined in Appendix~\ref{app:repre}.
\end{theorem}

\textbf{Transferring Model to Get Better Representation in Target.} 
Although Proposition~\ref{prop:nonlin_suf} shows that learning auxiliary dynamics models benefits representation learning, finding the optimal solution is non-trivial since one still has to learn $\hat{P}$ and $\hat{R}$. Therefore, the main idea of our algorithm is to transfer the dynamics models $\hat{P},\hat{R}$ from the source task to the target task, to ease the learning in the target task. Theorem~\ref{thm:transfer_suf} below guarantees that transferring the dynamics models is feasible. Our experimental result in Section~\ref{sec:exp} verifies that learning with transferred and fixed dynamics models outperforms learning with randomly initialized dynamics models.
\begin{theorem}[Transferable Dynamics Models]
\label{thm:transfer_suf}
Consider a source task $\mdp\source$ and a target task $\mdp\target$ with deterministic transition functions. 
Suppose $\phi\source$ is sufficient for $(\pact\source, \ract\source)_{a\in\actions}$ with functions $\apa, \ara$, then there exists a representation $\phi\target$ satisfying
$\apa(\phi(o))=\mathbb{E}_{o^\prime\sim \pact\target(o)}[\phi(o^\prime)]$, $\ara(\phi(o)) = \ract\target(o)$, for all $o\in\states\target$, and $\phi\target$ is sufficient for learning in $\mdp\target$.
\end{theorem}
Theorem~\ref{thm:transfer_suf} shows that the learned latent dynamics models $\hat{P},\hat{R}$ are transferable from the source task to the target task. 
For simplicity, Theorem~\ref{thm:transfer_suf} focuses on exact sufficiency as in Proposition~\ref{prop:nonlin_suf}, but it can be easily extended to $\epsilon$-sufficiency if combined with Theorem~\ref{thm:model_learn_error}.
Proofs for Proposition~\ref{prop:nonlin_suf}, Theorem~\ref{thm:model_learn_error} and Theorem~\ref{thm:transfer_suf} are all provided in Appendix~\ref{app:proofs}.


\textbf{Trade-off between Approximation Complexity and Representation Complexity.}
As suggested by Proposition~\ref{prop:nonlin_suf}, fitting policy-independent dynamics encourages the representation to be sufficient for learning, but not necessarily linearly sufficient. Therefore, we suggest using a \textit{non-linear policy/value head} following the representation to reduce the approximation error. 
Linear sufficiency can be achieved if $\phi$ is made linearly sufficient for $\ppi$ and $\rpi$ for all $\pi\in\reppolicy$, where $\ppi$ and $\rpi$ are transition and reward functions induced by policy $\pi$ (Proposition~\ref{prop:lin_suf}, Appendix~\ref{app:repre}).
However, using this method for transfer learning is expensive in terms of both computation and memory, as it requires to learn $\ppi$ and $\rpi$ for many different $\pi$'s and store these models for transferring to the target task.
Therefore, there is a trade-off between approximation complexity and representation complexity. Learning a linearly sufficient representation reduces the complexity of the approximation operator. But it requires more complexity in the representation itself as it has to satisfy much more constraints. 
To develop a practical and efficient transfer method, we use a slightly more complex approximation operator (non-linear policy head) while keeping the auxiliary task simple and easy to transfer across tasks. Please see Appendix~\ref{app:repre} for more detailed discussion about linear sufficiency.

%% file: algo.tex
\begin{figure}[t]
\vspace{-2em}
\begin{algorithm}[H]
\caption{Source Task Learning}
\label{alg:source}
\begin{algorithmic}[1] 
    \Require Regularization weight $\lambda$; update frequency $m$ for stable encoder.
    \State Initialize encoder $\phi\source$, stable encoder $\hat{\phi}\source$, policy $\pi\source$, transition prediction network $\hat{P}$ and reward prediction network $\hat{R}$.
    \For{$t=0,1,\cdots$}  
        \State Take action $a_t \sim \pi\source(\phi\source(o\source_t))$, get next observation $o_{t+1}\source$ and reward $r_t$, store to buffer.
        \State Sample a mini-batch $\{o_i,a_i,r_i,o^\prime_i\}_{i=1}^N$ from the buffer.
        \State Update $\hat{P}$ and $\hat{R}$ using one-step gradient descent with $\nabla_{\hat{P}} L_P(\hat{\phi}\source;\hat{P})$ and $\nabla_{\hat{R}} L_R(\hat{\phi}\source;\hat{R})$, where $L_P$ and $L_R$ are defined in Equation~(\ref{eq:model_loss}).
        \State Update encoder and policy by $\min_{\pi\source,\phi\source} L_{\text{base}}(\phi\source, \pi\source) + \lambda \big(L_P(\phi\source;\hat{P}) + L_R(\phi\source;\hat{R})\big).$
        \If{$t \mid m$}  Update the stable encoder $\hat{\phi}\source \leftarrow \phi\source.$
        \EndIf
    \EndFor 
\end{algorithmic}
\end{algorithm}
\vspace{-1.5em}
\end{figure}
\begin{figure}[t]
\vspace{-2em}
\begin{algorithm}[H]
\caption{Target Task Learning with Transferred Dynamics Models}
\label{alg:target}
\begin{algorithmic}[1] 
    \Require Regularization weight $\lambda$; dynamics models $\hat{P}$ and $\hat{R}$ learned in the source task.
    \State Initialize encoder $\phi\target$, policy $\pi\target$
    \For{$t=0,1,\cdots$}  
        \State Take action $a_t \sim \pi\target(\phi\target(o_t\target))$, get next observation $o_{t+1}\target$ and reward $r_t$, store to buffer.
        \State Sample a mini-batch $\{o_i,a_i,r_i,o^\prime_i\}_{i=1}^N$ from the buffer.
        \State Update encoder and policy by $\min_{\phi\target, \pi\target} L_{\text{base}}(\phi\target,\pi\target) + \lambda \big(L_P(\phi\target;\hat{P}) + L_R(\phi\target;\hat{R}))$, where $L_P$ and $L_R$ are defined in Equation~(\ref{eq:model_loss}).
    \EndFor 
\end{algorithmic}
\end{algorithm}
\vspace{-2em}
\end{figure}
\vspace{-0.5em}

%% file: s5-related.tex
\vspace{-0.5em}
\section{Related Work}
\label{sec:related}
\vspace{-0.5em}

\textbf{Transfer RL across Observation Feature Spaces.}
Transferring knowledge between tasks with different observation spaces has been studied for years. 
Many existing approaches\citep{taylor2007transfer,mann2013directed,brys2015policy} require an explicit mapping between the source and target observation spaces, which may be hard to obtain in practice.
\citet{raiman2019neural} introduce network surgery that deals with the change in the input features by determining which components of a neural network model should be transferred and which require retraining. However, it requires knowledge of the input feature maps, and is not designed for drastic changes, e.g. vector to pixel.
\citet{sun2020temple} propose a provably sample-efficient transfer learning algorithm that works for different observation spaces without knowing any inter-task mapping, but the algorithm is mainly designed for tabular RL and model-based RL which uses the model to plan for a policy, different from our setting.
\citet{gupta2017learning} achieve transfer learning between two different tasks by learning an invariant feature space, with a key time-based alignment assumption. We empirically compared this method with our proposed transfer algorithm in Section~\ref{sec:exp}.
Our work is also related to state abstraction in block MDPs, as studied by \citet{zhang2020invariant}. But the problem studied in~\citet{zhang2020invariant} is a multi-task setting where the agent aims to learn generalizable abstract states from a series of tasks.
Another related topic is domain adaptation in RL~\citep{higgins2017darla,eysenbach2020off,zhang2020learning}, where the target observation space (e.g. real world) is different from the source observation (e.g. simulator). However, domain adaptation does not assume drastic observation changes (e.g. changed dimension). Moreover, the aim of domain adaptation is usually zero-shot generalization to new observations, thus prior knowledge or a few samples of the target domain is often needed~\citep{eysenbach2020off}.

\textbf{Representation Learning in RL.}
In environments with rich observations, representation learning is crucial for the efficiency of RL methods. Learning unsupervised auxiliary tasks~\citep{jaderberg2016reinforcement} is shown to be effective for learning a good representation. 
The relationship between learning policy-dependent auxiliary tasks and learning good representations has been studied in some prior works~\citep{bellemare2019geometric,dabney2020the,lyle2021effect}, while our focus is to learn policy-independent auxiliary tasks to facilitate transfer learning. 
Using latent prediction models to regularize representation has been shown to be effective for various types of rich observations~\citep{guo2020bootstrap,lee2019stochastic}.  
\citet{gelada2019deepmdp} theoretically justify that learning latent dynamics model guarantees the quality of the learned representation, while we further characterize the relationship between representation and learning performance, and we utilize dynamics models to improve transfer learning.
\citet{zhang2020learning} use a bisimulation metric to learn latent representations that are invariant to task-irrelevant details in observation. As pointed out by~\citet{achille2018emergence}, invariant and sufficient representation is indeed minimal sufficient, so it is an interesting future direction to combine our method with bisimulation metric to learn minimal sufficient representations.
There is also a line of work using contrastive learning to train an encoder for pixel observations~\citep{srinivas2020curl,yarats2021reinforcement,stooke2021decoupling}, which usually pre-train an encoder based on image samples using self-supervised learning. However, environment dynamics are usually not considered during pre-training.
Our algorithm can be combined with these contrastive learning approaches to further improve learning performance in the target task.

%% file: s6-exp.tex
\vspace{-0.5em}
\section{Experimental Evaluation}
\label{sec:exp}


We empirically evaluate our transfer learning algorithm in various environments and multiple observation-change scenarios.
Detailed experiment setup and hyperparameters are in Appendix~\ref{app:exp}.

\textbf{Baselines.}
To verify the effectiveness of our proposed transfer learning method, we compare our transfer learning algorithm with 4 baselines: 
(1) \textit{Single}: a single-task base learner. 
(2) \textit{Auxiliary}: learns auxiliary models from scratch to regularize representation. 
(3) \textit{Fine-tune}: loads and freezes the source policy head, and retrains an encoder in the target task. 
(4) \textit{Time-aligned}~\citep{gupta2017learning}: supposes the target task and the source task proceed to the same latent state given the same action sequence, and pre-trains a target-task encoder with saved source-task trajectories. More details of baseline implementations are in Appendix~\ref{app:exp_baseline}.

\textbf{Scenarios.}
As motivated in Section~\ref{sec:intro}, there are many scenarios where one can benefit from transfer learning across observation feature spaces. We evaluate our proposed transfer algorithm in 7 environments that fit various scenarios, to simulate real-world applications: \\
(1) \textit{Vec-to-pixel:} a novel and challenging scenario, where the source task has low-dimensional vector observations and the target task has pixel observations. We use 3 vector-input environments CartPole, Acrobot and Cheetah-Run as source tasks, and use the rendered image in the target task. \\
(2) \textit{More-sensor:} another challenging scenario where the target task has a lot more sensors than the source task. We use 3 MuJoCo environments: HalfCheetah, Hopper and Walker2d, whose original observation dimensions are 17, 11 and 17, respectively. We add mass-based inertia and velocity (provided by MuJoCo's API), resulting in 145, 91, 145 dimensions in the corresponding target tasks.  \\
(3) \textit{Broken-sensor:} we use an existing game 3DBall contained in the Unity ML-Agents Toolkit~\citep{juliani2018unity}, which has two different observation specifications that naturally fit our transfer setting: the source observation has 8 features containing the velocity of the ball; the target observation does not have the ball's velocity, thus the agent has to stack the past 9 frames to infer the velocity. Please see Appendix~\ref{app:exp_env} for more detailed descriptions of all the 7 environments.

\textbf{Base DRL Learners.}
What we propose is a transfer learning mechanism that can be combined with any existing DRL methods. For environments with discrete action spaces (CartPole, Acrobot), we use the DQN algorithm~\citep{mnih2015human}, while for environments with continuous action spaces (Cheetah-Run, HalfCheetah, Hopper, Walker2d, 3DBall), we use the SAC algorithm~\citep{haarnoja2018soft}. To ensure a  fair comparison, we use the same base DRL learner with the same hyperparameter settings for all tested methods, as detailed in Appendix~\ref{app:exp_drl}. As is common in prior works, our implementation of the RL algorithms is mostly a proof of concept, thus many advanced training techniques are not included (e.g. Rainbow DQN).

\begin{figure}[!t]
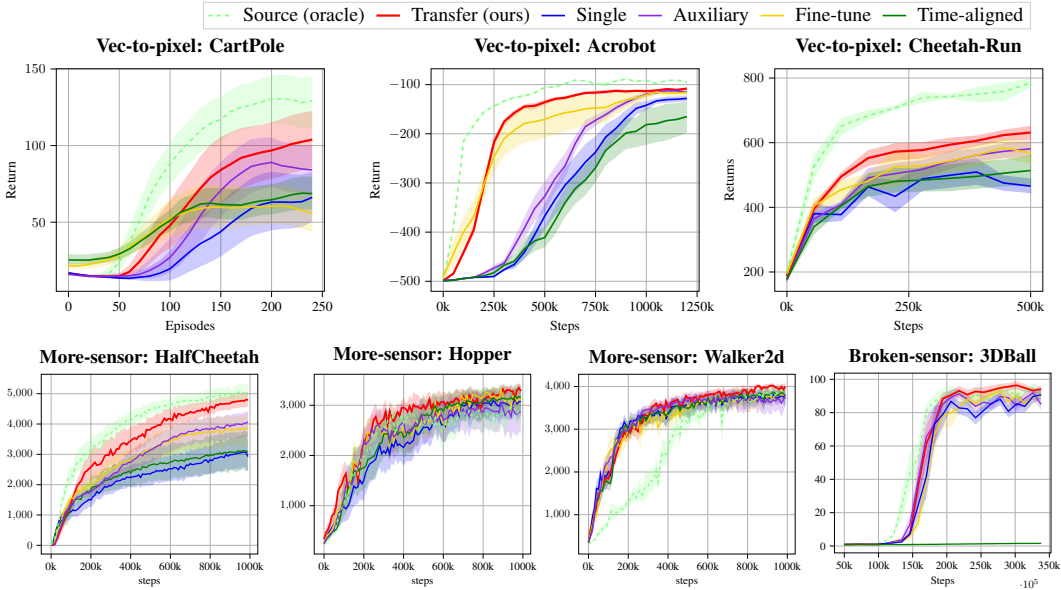

\vspace{-2.5em}
\centering
 \begin{subfigure}[t]{0.01\columnwidth}
  \centering
  \input{figs/cart_new_baseline_smooth}
 \end{subfigure}
 \hfill
 \begin{subfigure}[t]{0.31\columnwidth}
  \centering
  \input{figs/acrobot} 
 \end{subfigure}
  \hfill
  \begin{subfigure}[t]{0.31\columnwidth}
  \centering
  \input{figs/cheetah_run} 
 \end{subfigure}
 \vspace{-1em}
 
 \begin{subfigure}[t]{0.255\columnwidth}
  \centering
  \input{figs/halfcheetah_rebuttal}
 \end{subfigure}
 \hfill
 \begin{subfigure}[t]{0.245\columnwidth}
  \centering
  \input{figs/hopper_rebuttal} 
 \end{subfigure}
 \hfill
 \begin{subfigure}[t]{0.245\columnwidth}
  \centering
  \input{figs/walker2d_rebuttal} 
 \end{subfigure}
 \hfill
 \begin{subfigure}[t]{0.235\columnwidth}
  \centering
  \input{figs/new_ball_withsource} 
 \end{subfigure}
 \vspace{-2em}
 \caption{
 Our proposed transfer method outperforms all baselines in target tasks over all tested scenarios. (The dashed \textcolor{sourcecolor}{\textbf{green}} lines are the learning curves in source tasks.) Results are averaged over 10 random seeds. 
 }
\label{fig:all}
\vspace{-1.5em}
\end{figure}


\textbf{Results.}
Experimental results on all tested environments are shown in Figure~\ref{fig:all}. We can see that our proposed transfer method learns significantly better than the single-task learner, and also outperforms all baselines in the challenging target tasks. 
Our transfer method outperforms Auxiliary since it transfers dynamics model from the source task instead of learning it from scratch, and outperforms Fine-tine since it regularizes the challenging encoder learning with a model-based regularizer.
The Time-aligned method, although requires additional pre-training that is not shown in the figures, does not work better than Single in most environments, because the time-based alignment assumption may not hold as discussed in Appendix~\ref{app:exp_baseline}.
In some environments (e.g. Hopper, Walker2d, 3DBall), our transfer algorithm even achieves better asymptotic performance than the source-task policy, which suggests that our method can be used for improving the policy with incremental observation design.
\textit{To the best of our knowledge, we are the first to achieve effective knowledge transfer from a vector-input environment to a pixel-input environment without any pre-defined mappings.}

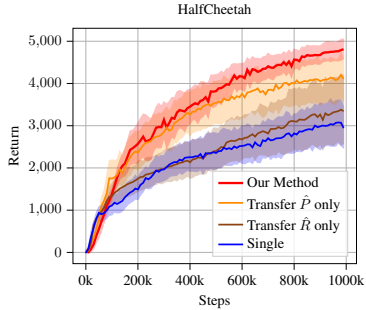
\begin{wrapfigure}{r}{0.35\textwidth}
\vspace{-2em}
  \begin{center}
    \input{figs/halfcheetah_ablation_small}
  \end{center}
  \vspace{-1.5em}
  \caption{Ablation Study}
  \label{fig:ablation_halfcheetah}
\vspace{-1.5em}
\end{wrapfigure}

\textbf{Ablation Study and Hyper-parameter Test.}
To verify the effectiveness of proposed transfer method, we conduct ablation study and compare our method with its two variants: only transferring the transition model $\hat{P}$ and only transferring the reward model $\hat{R}$. Figure~\ref{fig:ablation_halfcheetah} shows the comparison in HalfCheetah, and Appendix~\ref{app:exp_results} demonstrates more results. We find that all the variants of our method can make some improvements, which suggests that \textit{transferring $\hat{P}$ and $\hat{R}$ are both effective designs for accelerating the target task learning.}
Figure~\ref{fig:cart_ablation} in Appendix~\ref{app:exp_results}
shows another ablation study where we investigate different selections of model regularizers and policy heads. 
In Algorithm~\ref{alg:target}, a hyper-parameter $\lambda$ is needed to control the weight of the transferred model-based regularizer. Figure~\ref{fig:hyper} in Appendix~\ref{app:exp_results} shows that, for a wide range of $\lambda$'s, the agent consistently outperforms the single-task learner. 

\textbf{Potential Limitations and Solutions.}
As Figure~\ref{fig:all} shows, in some environments such as HalfCheetah, our transfer algorithm significantly outperforms baselines without transfer. But in Walker2d, the improvement is less significant, although transferring is still better than not transferring.
This phenomenon is common in model-based learning~\citep{nagabandi2018neural}, as state predicting in Walker2d is harder than that in HalfCheetah due to the complexity of the dynamics.
Therefore, we suggest using our method to transfer when the learned models $(\hat{P},\hat{R})$ in the source task are relatively good (error is low).
More techniques of improving model-based learning, such as bisimulation~\citep{zhang2020learning,castro2020scalable}, can be applied to further improve the transfer performance.




%% file: figs/cart_new_baseline_smooth.tex
\begin{tikzpicture}[scale=0.52]


\begin{axis}[
legend columns=6, 
legend cell align={left},
legend style={at={(0, 1.3)}, anchor=north west, fill opacity=0, draw opacity=0, text opacity=0, draw=white!80!black, font=\Large, column sep=0.2cm},
tick align=outside,
tick pos=left,
title={\Large{\textbf{Vec-to-pixel: CartPole}}},
x grid style={white!69.0196078431373!black},
xlabel={Episodes},
xmajorgrids,
xmin=-1.2, xmax=25.2,
xtick={0,5,10,15,20,25},
xticklabels={0,50,100,150,200,250},
xtick style={color=black},
y grid style={white!69.0196078431373!black},
ylabel={Return},
ymajorgrids,
ymin=4.88449960922675, ymax=150.123815244617,
ytick style={color=black}
]
\path [fill=sourcecolor, fill opacity=0.2]
(axis cs:0,16.8934)
--(axis cs:0,15.5197666666667)
--(axis cs:1,14.7664733333333)
--(axis cs:2,14.0025053333333)
--(axis cs:3,13.5232776)
--(axis cs:4,13.4385887466667)
--(axis cs:5,17.9400443306667)
--(axis cs:6,27.5144287978667)
--(axis cs:7,38.9831163716267)
--(axis cs:8,55.212459763968)
--(axis cs:9,66.2479478111744)
--(axis cs:10,74.9358115822729)
--(axis cs:11,82.2986425991516)
--(axis cs:12,88.991080745988)
--(axis cs:13,95.6643045967904)
--(axis cs:14,100.048943677432)
--(axis cs:15,102.521888275279)
--(axis cs:16,108.023643953557)
--(axis cs:17,110.281741829512)
--(axis cs:18,111.720626796943)
--(axis cs:19,113.481081437554)
--(axis cs:20,114.769118483377)
--(axis cs:21,113.674514786701)
--(axis cs:22,113.453991829361)
--(axis cs:23,111.079626796822)
--(axis cs:24,112.209654770791)
--(axis cs:24,145.132212304183)
--(axis cs:24,145.132212304183)
--(axis cs:23,143.897698713562)
--(axis cs:22,145.884573391953)
--(axis cs:21,146.217083406608)
--(axis cs:20,145.709470924926)
--(axis cs:19,143.596663656158)
--(axis cs:18,141.055796236864)
--(axis cs:17,139.745503629413)
--(axis cs:16,136.377304536767)
--(axis cs:15,131.779780670958)
--(axis cs:14,129.227225838698)
--(axis cs:13,124.378065631706)
--(axis cs:12,117.401282039632)
--(axis cs:11,110.211710882873)
--(axis cs:10,101.766271936925)
--(axis cs:9,92.0639649211563)
--(axis cs:8,81.004697818112)
--(axis cs:7,63.89557227264)
--(axis cs:6,47.0077153408)
--(axis cs:5,30.980994176)
--(axis cs:4,20.54074272)
--(axis cs:3,16.5550784)
--(axis cs:2,15.4288146666667)
--(axis cs:1,16.0567266666667)
--(axis cs:0,16.8934)
--cycle;

\path [fill=ourcolor, fill opacity=0.2]
(axis cs:0,17.5200333333333)
--(axis cs:0,15.7766666666667)
--(axis cs:1,14.9273266666667)
--(axis cs:2,14.0971946666667)
--(axis cs:3,13.7443424)
--(axis cs:4,13.3634672533333)
--(axis cs:5,13.1714404693333)
--(axis cs:6,14.2544057088)
--(axis cs:7,17.47462456704)
--(axis cs:8,22.7875129869653)
--(axis cs:9,28.8686170562389)
--(axis cs:10,32.9673269783245)
--(axis cs:11,38.7043549159929)
--(axis cs:12,45.099330599461)
--(axis cs:13,50.7748911462355)
--(axis cs:14,57.206719583655)
--(axis cs:15,62.1070223335907)
--(axis cs:16,66.6016178668725)
--(axis cs:17,71.961474293498)
--(axis cs:18,74.2803261014651)
--(axis cs:19,76.7432742145054)
--(axis cs:20,78.4036593716043)
--(axis cs:21,79.8837074972835)
--(axis cs:22,82.1033526644934)
--(axis cs:23,84.0400287982614)
--(axis cs:24,85.5673497052758)
--(axis cs:24,122.559536332682)
--(axis cs:24,122.559536332682)
--(axis cs:23,121.368053749186)
--(axis cs:22,119.802892186482)
--(axis cs:21,117.431898566436)
--(axis cs:20,115.231406541379)
--(axis cs:19,113.43999984339)
--(axis cs:18,112.656383137571)
--(axis cs:17,111.450328921963)
--(axis cs:16,108.807886152454)
--(axis cs:15,105.709724357234)
--(axis cs:14,100.833255446543)
--(axis cs:13,93.0959693081788)
--(axis cs:12,83.1667533018901)
--(axis cs:11,74.1150082940293)
--(axis cs:10,64.0781020342033)
--(axis cs:9,55.7092192094208)
--(axis cs:8,44.1769573451093)
--(axis cs:7,32.29913001472)
--(axis cs:6,22.0454125184)
--(axis cs:5,16.767315648)
--(axis cs:4,16.0736612266667)
--(axis cs:3,15.9511432)
--(axis cs:2,15.6197373333333)
--(axis cs:1,16.62548)
--(axis cs:0,17.5200333333333)
--cycle;

\path [fill=singlecolor, fill opacity=0.2]
(axis cs:0,17.3934)
--(axis cs:0,16.2499)
--(axis cs:1,15.28858)
--(axis cs:2,14.4008506666667)
--(axis cs:3,13.8652605333333)
--(axis cs:4,13.4781950933333)
--(axis cs:5,12.871209408)
--(axis cs:6,12.3516208597333)
--(axis cs:7,11.8745833544533)
--(axis cs:8,11.4862866835627)
--(axis cs:9,11.6351093468501)
--(axis cs:10,11.9799008108134)
--(axis cs:11,14.5527406486508)
--(axis cs:12,17.8994058522539)
--(axis cs:13,21.0610113484698)
--(axis cs:14,24.0045490787758)
--(axis cs:15,27.6499659296873)
--(axis cs:16,32.2678594104165)
--(axis cs:17,36.8589008616666)
--(axis cs:18,40.9687140226666)
--(axis cs:19,43.0093845514666)
--(axis cs:20,44.6143609745066)
--(axis cs:21,45.5353687796053)
--(axis cs:22,45.9020950236842)
--(axis cs:23,46.9733693522807)
--(axis cs:24,50.2306888151579)
--(axis cs:24,82.7154981642053)
--(axis cs:24,82.7154981642053)
--(axis cs:23,81.0365477052567)
--(axis cs:22,81.2146179649042)
--(axis cs:21,82.1995807894636)
--(axis cs:20,82.9677343201628)
--(axis cs:19,81.5492845668702)
--(axis cs:18,78.6011723752544)
--(axis cs:17,73.7530821357346)
--(axis cs:16,69.1273610030016)
--(axis cs:15,62.605126253752)
--(axis cs:14,58.0979661505234)
--(axis cs:13,54.1891410214876)
--(axis cs:12,47.6212262768595)
--(axis cs:11,39.460474512741)
--(axis cs:10,31.3205348075929)
--(axis cs:9,25.2638018428245)
--(axis cs:8,21.1097189701973)
--(axis cs:7,18.40703204608)
--(axis cs:6,14.9054067242667)
--(axis cs:5,14.3166417386667)
--(axis cs:4,14.7466105066667)
--(axis cs:3,15.1915714666667)
--(axis cs:2,15.6986226666667)
--(axis cs:1,16.3573866666667)
--(axis cs:0,17.3934)
--cycle;

\path [fill=auxcolor, fill opacity=0.2]
(axis cs:0,16.51)
--(axis cs:0,15.3233333333333)
--(axis cs:1,14.65798)
--(axis cs:2,13.9856906666667)
--(axis cs:3,13.8665125333333)
--(axis cs:4,13.63516336)
--(axis cs:5,13.3774040213333)
--(axis cs:6,13.1285165504)
--(axis cs:7,12.97399324032)
--(axis cs:8,13.2511612589227)
--(axis cs:9,14.2843423404715)
--(axis cs:10,17.0893338723772)
--(axis cs:11,22.0150404312351)
--(axis cs:12,29.0709723449881)
--(axis cs:13,37.4358912093238)
--(axis cs:14,45.8956863007924)
--(axis cs:15,53.2935823739672)
--(axis cs:16,59.7153858991738)
--(axis cs:17,64.4622420526724)
--(axis cs:18,69.2150736421379)
--(axis cs:19,71.342045580377)
--(axis cs:20,72.9216164643016)
--(axis cs:21,70.9211665047746)
--(axis cs:22,69.7513532038197)
--(axis cs:23,68.6933158963891)
--(axis cs:24,67.9363927171113)
--(axis cs:24,100.505509826687)
--(axis cs:24,100.505509826687)
--(axis cs:23,100.164137283359)
--(axis cs:22,101.596321604199)
--(axis cs:21,102.576002005249)
--(axis cs:20,105.309260839895)
--(axis cs:19,105.091092716535)
--(axis cs:18,104.241174229002)
--(axis cs:17,100.204692786252)
--(axis cs:16,95.3427409828156)
--(axis cs:15,89.2963095618529)
--(axis cs:14,82.3414119523161)
--(axis cs:13,73.3840066070618)
--(axis cs:12,61.1377332588272)
--(axis cs:11,50.0385249068674)
--(axis cs:10,39.5310228002509)
--(axis cs:9,32.5734535003136)
--(axis cs:8,26.1035335420587)
--(axis cs:7,20.12308359424)
--(axis cs:6,17.5461711594667)
--(axis cs:5,16.486955616)
--(axis cs:4,15.8503528533333)
--(axis cs:3,15.6987077333333)
--(axis cs:2,15.5150013333333)
--(axis cs:1,15.8586933333333)
--(axis cs:0,16.51)
--cycle;

\path [fill=tunecolor, fill opacity=0.2]
(axis cs:0,23.5400666666667)
--(axis cs:0,19.9627666666667)
--(axis cs:1,19.6502133333333)
--(axis cs:2,19.8287973333333)
--(axis cs:3,20.7723245333333)
--(axis cs:4,22.0011662933333)
--(axis cs:5,24.1549197013333)
--(axis cs:6,28.5938157610667)
--(axis cs:7,32.0733526088533)
--(axis cs:8,34.6110220870827)
--(axis cs:9,37.9157376696661)
--(axis cs:10,40.8037501357329)
--(axis cs:11,42.5340734419197)
--(axis cs:12,44.2315587535357)
--(axis cs:13,45.6370403361619)
--(axis cs:14,46.2448989355962)
--(axis cs:15,46.5315458151436)
--(axis cs:16,46.2589433187816)
--(axis cs:17,46.9977946550252)
--(axis cs:18,47.9375090573535)
--(axis cs:19,48.1739805792162)
--(axis cs:20,48.3435911300396)
--(axis cs:21,48.5757395706983)
--(axis cs:22,47.6298716565587)
--(axis cs:23,45.5349706585803)
--(axis cs:24,43.6548231935309)
--(axis cs:24,69.1449849655404)
--(axis cs:24,69.1449849655404)
--(axis cs:23,70.9902978735922)
--(axis cs:22,73.1101306753235)
--(axis cs:21,74.0259133441544)
--(axis cs:20,73.673216680193)
--(axis cs:19,72.8954125169079)
--(axis cs:18,72.5540906461349)
--(axis cs:17,73.1617299743353)
--(axis cs:16,73.4377874679191)
--(axis cs:15,73.7054843348989)
--(axis cs:14,74.4684470852903)
--(axis cs:13,73.5009088566129)
--(axis cs:12,71.7793444040995)
--(axis cs:11,68.4158388384577)
--(axis cs:10,65.0717485480721)
--(axis cs:9,58.7662273517568)
--(axis cs:8,52.3653175230293)
--(axis cs:7,45.6937885704533)
--(axis cs:6,39.5997190464)
--(axis cs:5,33.5531571413333)
--(axis cs:4,29.5101964266667)
--(axis cs:3,27.9591205333333)
--(axis cs:2,25.8313506666667)
--(axis cs:1,23.8203133333333)
--(axis cs:0,23.5400666666667)
--cycle;

\path [fill=aligncolor, fill opacity=0.2]
(axis cs:0,29.0536)
--(axis cs:0,22.5963333333333)
--(axis cs:1,22.4683866666667)
--(axis cs:2,22.2626426666667)
--(axis cs:3,22.9900941333333)
--(axis cs:4,23.8366619733333)
--(axis cs:5,25.772662912)
--(axis cs:6,28.5119436629333)
--(axis cs:7,32.29988826368)
--(axis cs:8,36.3084172776107)
--(axis cs:9,40.8093871554219)
--(axis cs:10,44.1400630576708)
--(axis cs:11,48.0646171128033)
--(axis cs:12,50.358147023576)
--(axis cs:13,52.9042109521941)
--(axis cs:14,53.199335428422)
--(axis cs:15,51.8933950094042)
--(axis cs:16,51.7293826741901)
--(axis cs:17,51.4367194726854)
--(axis cs:18,53.0378089114816)
--(axis cs:19,54.6131071291853)
--(axis cs:20,55.0342257033483)
--(axis cs:21,56.3016805626786)
--(axis cs:22,57.4805711168095)
--(axis cs:23,58.552170226781)
--(axis cs:24,58.2678828480914)
--(axis cs:24,79.5770715910615)
--(axis cs:24,79.5770715910615)
--(axis cs:23,79.8354478221602)
--(axis cs:22,78.7149681110335)
--(axis cs:21,76.4426268054586)
--(axis cs:20,74.8068835068232)
--(axis cs:19,74.0042710501957)
--(axis cs:18,73.179297146078)
--(axis cs:17,71.8153547659308)
--(axis cs:16,71.7775101240802)
--(axis cs:15,72.1245126551002)
--(axis cs:14,72.5431074855419)
--(axis cs:13,71.491351023594)
--(axis cs:12,67.8436554461592)
--(axis cs:11,65.048544307699)
--(axis cs:10,58.2131553846238)
--(axis cs:9,53.2077692307797)
--(axis cs:8,47.6610615384747)
--(axis cs:7,42.24699358976)
--(axis cs:6,37.0219586538667)
--(axis cs:5,33.266589984)
--(axis cs:4,30.8482458133333)
--(axis cs:3,29.9244322666667)
--(axis cs:2,29.1591653333333)
--(axis cs:1,28.8838733333333)
--(axis cs:0,29.0536)
--cycle;

\addplot [very thick, dashed, sourcecolor]
table {%
0 16.2478733333333
1 15.44514
2 14.7270566666667
3 14.847866
4 16.4081841333333
5 23.8736433066667
6 36.6648219786667
7 50.7961069162667
8 67.60991019968
9 78.769666159744
10 88.1231709277952
11 96.1980454089028
12 103.006752993789
13 109.960690395031
14 114.529135649358
15 117.148232519487
16 122.246049348923
17 124.958564812471
18 126.404045849977
19 128.481754013315
20 130.330389877319
21 130.353660568522
22 130.045023788151
23 128.02350303052
24 129.22940375775
};
\addlegendentry{Source}
\addplot [ultra thick, ourcolor]
table {%
0 16.6343166666667
1 15.73184
2 14.820312
3 14.7757656
4 14.5874758133333
5 14.766051984
6 17.6420569205333
7 24.26395286976
8 32.7358176291413
9 41.6176314366464
10 47.9024911493171
11 55.718908252787
12 64.0335872688963
13 72.1972464817837
14 79.228191185427
15 84.0908376150082
16 87.9768107586732
17 91.9655566069386
18 93.8202119522175
19 95.5230728951074
20 96.8409656494192
21 98.7009798528687
22 101.044838548962
23 102.475850172503
24 103.913676138002
};
\addlegendentry{Transfer (ours)}
\addplot [very thick, singlecolor]
table {%
0 16.8155366666667
1 15.8145746666667
2 15.0134324
3 14.4889472533333
4 14.0733404693333
5 13.5136150421333
6 13.3565207003733
7 14.1356932269653
8 14.7872572482389
9 16.7333104652578
10 19.5728037055396
11 25.2558976310983
12 31.030819438212
13 36.0569062172363
14 39.592068973789
15 43.8332931790312
16 49.4307538765583
17 54.4377637679133
18 58.9349870143306
19 61.5066836114645
20 63.0956308891716
21 63.1745733780039
22 62.9611180357365
23 63.4156917619225
24 66.1904720762047
};
\addlegendentry{Single}
\addplot [very thick, auxcolor]
table {%
0 15.9668833333333
1 15.2741913333333
2 14.7272950666667
3 14.74731472
4 14.6274904426667
5 14.7523363541333
6 15.0281457499733
7 15.9464125999787
8 18.5216194133163
9 22.073243530653
10 27.0040161578557
11 34.7627409262846
12 43.8871954076943
13 54.4158356594888
14 63.4261398609244
15 70.7328445554062
16 77.1242023109916
17 82.0273298487933
18 86.3761225457013
19 87.9161433698944
20 89.0288406959155
21 86.8452458900657
22 85.8526460453859
23 84.7025881696421
24 84.1436192023803
};
\addlegendentry{Auxiliary}
\addplot [very thick, tunecolor]
table {%
0 21.62879
1 21.553708
2 22.4175397333333
3 23.9763444533333
4 25.4706802293333
5 28.5177075168
6 33.71485001344
7 38.5543326774187
8 42.9946114752683
9 47.9693911802146
10 52.542062277505
11 55.096479822004
12 57.4657885242699
13 58.9432128194159
14 59.8305215888661
15 59.7213186044262
16 59.3728482168743
17 59.5361959068328
18 59.9331120587995
19 60.2205263137063
20 60.4547497176317
21 60.779784440772
22 59.774000885951
23 57.6089780420941
24 55.7939171003419
};
\addlegendentry{Finetune}
\addplot [very thick, aligncolor]
table {%
0 25.3624733333333
1 25.2618166666667
2 25.2886586666667
3 26.1024509333333
4 27.0382867466667
5 29.2781913973333
6 32.5208271178667
7 37.1501350276267
8 41.7850760221013
9 46.8503434843477
10 51.0819774541448
11 56.3087319633159
12 58.854598903986
13 61.8547051231888
14 62.4111554318844
15 61.5677950121742
16 61.3406020097393
17 61.1526282744581
18 62.6979646195665
19 63.9595436956532
20 64.6070636231892
21 66.1843162318847
22 67.9434603188411
23 68.9025409217395
24 68.666244070725
};
\addlegendentry{Time-aligned}
\end{axis}

\end{tikzpicture}

%% file: figs/acrobot.tex
\begin{tikzpicture}[scale=0.52]


\begin{axis}[
legend columns=6, 
legend cell align={left},
legend style={at={(-0.85, 1.3)}, anchor=north west, fill opacity=1, draw opacity=1, text opacity=1, draw=white!80!black, font=\Large, column sep=0.2cm},
tick align=outside,
tick pos=left,
x grid style={white!69.0196078431373!black},
title={\Large{\textbf{Vec-to-pixel: Acrobot}}},
xlabel={Steps},
xmajorgrids,
xmin=-1.2, xmax=25.2,
xtick style={color=black},
xtick={0,5,10,15,20,25},
xticklabels={0k,250k,500k,750k,1000k,1250k},
y grid style={white!69.0196078431373!black},
ylabel={Return},
ymajorgrids,
ymin=-520.363648805556, ymax=-67.7718470410053,
ytick style={color=black}
]
\path [fill=sourcecolor, fill opacity=0.2]
(axis cs:0,-499.482725661376)
--(axis cs:0,-499.791294179894)
--(axis cs:1,-395.900233333333)
--(axis cs:2,-218.423835)
--(axis cs:3,-185.686483333333)
--(axis cs:4,-157.634086666667)
--(axis cs:5,-145.565133333333)
--(axis cs:6,-135.404966666667)
--(axis cs:7,-127.499268333333)
--(axis cs:8,-127.19382)
--(axis cs:9,-118.2275)
--(axis cs:10,-107.365226666667)
--(axis cs:11,-106.394301666667)
--(axis cs:12,-99.1909983333334)
--(axis cs:13,-91.5249333333333)
--(axis cs:14,-93.2521916666666)
--(axis cs:15,-99.8384533333333)
--(axis cs:16,-104.000466666667)
--(axis cs:17,-94.5652666666666)
--(axis cs:18,-89.372005)
--(axis cs:19,-94.4117116666667)
--(axis cs:20,-92.8951183333333)
--(axis cs:21,-100.843265)
--(axis cs:22,-93.2551666666667)
--(axis cs:23,-91.887665)
--(axis cs:24,-97.0985)
--(axis cs:24,-92.9338583333333)
--(axis cs:24,-92.9338583333333)
--(axis cs:23,-89.7156683333333)
--(axis cs:22,-90.3496666666667)
--(axis cs:21,-95.0716166666667)
--(axis cs:20,-90.1214533333334)
--(axis cs:19,-91.8678116666667)
--(axis cs:18,-88.3442016666667)
--(axis cs:17,-90.954395)
--(axis cs:16,-96.1334983333333)
--(axis cs:15,-96.84173)
--(axis cs:14,-91.60652)
--(axis cs:13,-90.7242033333333)
--(axis cs:12,-96.6491183333333)
--(axis cs:11,-102.844085)
--(axis cs:10,-104.283845)
--(axis cs:9,-115.179466666667)
--(axis cs:8,-118.530946666667)
--(axis cs:7,-124.720668333333)
--(axis cs:6,-131.484938333333)
--(axis cs:5,-141.670383333333)
--(axis cs:4,-153.617313333333)
--(axis cs:3,-179.848166666667)
--(axis cs:2,-212.041376666667)
--(axis cs:1,-387.394995)
--(axis cs:0,-499.482725661376)
--cycle;

\path [fill=singlecolor, fill opacity=0.2]
(axis cs:0,-498.819184424603)
--(axis cs:0,-499.311832123016)
--(axis cs:1,-497.955923333333)
--(axis cs:2,-494.870563333333)
--(axis cs:3,-493.33208)
--(axis cs:4,-492.949006666667)
--(axis cs:5,-492.153956666667)
--(axis cs:6,-479.68761)
--(axis cs:7,-473.831713333333)
--(axis cs:8,-455.168943333333)
--(axis cs:9,-427.404545)
--(axis cs:10,-393.379483333333)
--(axis cs:11,-365.699166666667)
--(axis cs:12,-343.216275)
--(axis cs:13,-318.407031666667)
--(axis cs:14,-296.667225)
--(axis cs:15,-270.406836666667)
--(axis cs:16,-234.369025)
--(axis cs:17,-211.424931666667)
--(axis cs:18,-184.681956666667)
--(axis cs:19,-160.891468333333)
--(axis cs:20,-151.96094)
--(axis cs:21,-138.76427)
--(axis cs:22,-137.124625)
--(axis cs:23,-135.367116666667)
--(axis cs:24,-134.567388333333)
--(axis cs:24,-121.920298333333)
--(axis cs:24,-121.920298333333)
--(axis cs:23,-124.239603333333)
--(axis cs:22,-124.166535)
--(axis cs:21,-127.55327)
--(axis cs:20,-133.105308333333)
--(axis cs:19,-136.976255)
--(axis cs:18,-151.408141666667)
--(axis cs:17,-168.82573)
--(axis cs:16,-180.205525)
--(axis cs:15,-205.073596666667)
--(axis cs:14,-220.580365)
--(axis cs:13,-245.8015)
--(axis cs:12,-267.80091)
--(axis cs:11,-302.29094)
--(axis cs:10,-341.598823333333)
--(axis cs:9,-384.285661666667)
--(axis cs:8,-426.245825)
--(axis cs:7,-456.861045)
--(axis cs:6,-472.827975)
--(axis cs:5,-487.850013333333)
--(axis cs:4,-490.395008333333)
--(axis cs:3,-490.69311)
--(axis cs:2,-493.39152)
--(axis cs:1,-497.190746666667)
--(axis cs:0,-498.819184424603)
--cycle;

\path [fill=ourcolor, fill opacity=0.2]
(axis cs:0,-498.310908174603)
--(axis cs:0,-498.896432063492)
--(axis cs:1,-487.047206666667)
--(axis cs:2,-447.77801)
--(axis cs:3,-406.249073333333)
--(axis cs:4,-316.274106666667)
--(axis cs:5,-228.390771666667)
--(axis cs:6,-183.0457)
--(axis cs:7,-164.016878333333)
--(axis cs:8,-149.735501666667)
--(axis cs:9,-150.237701666667)
--(axis cs:10,-141.406253333333)
--(axis cs:11,-133.259315)
--(axis cs:12,-130.304158333333)
--(axis cs:13,-124.644976666667)
--(axis cs:14,-120.500848333333)
--(axis cs:15,-120.592998333333)
--(axis cs:16,-119.085278333333)
--(axis cs:17,-116.144781666667)
--(axis cs:18,-116.592923333333)
--(axis cs:19,-114.626155)
--(axis cs:20,-116.063946666667)
--(axis cs:21,-113.9547)
--(axis cs:22,-110.36065)
--(axis cs:23,-113.356838333333)
--(axis cs:24,-109.54796)
--(axis cs:24,-106.550178333333)
--(axis cs:24,-106.550178333333)
--(axis cs:23,-108.077715)
--(axis cs:22,-107.74523)
--(axis cs:21,-109.761431666667)
--(axis cs:20,-110.229583333333)
--(axis cs:19,-110.583093333333)
--(axis cs:18,-110.80339)
--(axis cs:17,-109.426616666667)
--(axis cs:16,-110.727983333333)
--(axis cs:15,-112.103368333333)
--(axis cs:14,-113.466523333333)
--(axis cs:13,-118.315005)
--(axis cs:12,-123.144838333333)
--(axis cs:11,-125.257985)
--(axis cs:10,-130.639581666667)
--(axis cs:9,-136.840108333333)
--(axis cs:8,-139.685545)
--(axis cs:7,-151.533776666667)
--(axis cs:6,-168.614396666667)
--(axis cs:5,-205.872446666667)
--(axis cs:4,-291.838295)
--(axis cs:3,-385.64661)
--(axis cs:2,-434.224293333333)
--(axis cs:1,-481.799028333333)
--(axis cs:0,-498.310908174603)
--cycle;

\path [fill=auxcolor, fill opacity=0.2]
(axis cs:0,-497.758329814815)
--(axis cs:0,-498.760129206349)
--(axis cs:1,-497.613053333333)
--(axis cs:2,-495.467923333333)
--(axis cs:3,-493.088126666667)
--(axis cs:4,-486.510551666667)
--(axis cs:5,-478.764715)
--(axis cs:6,-470.193173333333)
--(axis cs:7,-443.370205)
--(axis cs:8,-411.731868333333)
--(axis cs:9,-376.861685)
--(axis cs:10,-358.693803333333)
--(axis cs:11,-319.33722)
--(axis cs:12,-289.89672)
--(axis cs:13,-236.1242)
--(axis cs:14,-198.659101666667)
--(axis cs:15,-184.09708)
--(axis cs:16,-169.182185)
--(axis cs:17,-153.26464)
--(axis cs:18,-142.733383333333)
--(axis cs:19,-133.427833333333)
--(axis cs:20,-124.24978)
--(axis cs:21,-119.987998333333)
--(axis cs:22,-115.20455)
--(axis cs:23,-116.285946666667)
--(axis cs:24,-118.707163333333)
--(axis cs:24,-112.175833333333)
--(axis cs:24,-112.175833333333)
--(axis cs:23,-109.395786666667)
--(axis cs:22,-109.64473)
--(axis cs:21,-112.8838)
--(axis cs:20,-115.25891)
--(axis cs:19,-120.681193333333)
--(axis cs:18,-130.365535)
--(axis cs:17,-141.112698333333)
--(axis cs:16,-152.06198)
--(axis cs:15,-161.495603333333)
--(axis cs:14,-171.885693333333)
--(axis cs:13,-203.422883333333)
--(axis cs:12,-239.577408333333)
--(axis cs:11,-266.939576666667)
--(axis cs:10,-296.947205)
--(axis cs:9,-326.583451666667)
--(axis cs:8,-368.617191666667)
--(axis cs:7,-412.905416666667)
--(axis cs:6,-456.072173333333)
--(axis cs:5,-466.794915)
--(axis cs:4,-480.46819)
--(axis cs:3,-491.178593333333)
--(axis cs:2,-494.004123333333)
--(axis cs:1,-496.804393333333)
--(axis cs:0,-497.758329814815)
--cycle;

\path [fill=aligncolor, fill opacity=0.2]
(axis cs:0,-498.26065)
--(axis cs:0,-499.13385)
--(axis cs:1,-498.178666666667)
--(axis cs:2,-495.172938333333)
--(axis cs:3,-492.659961666667)
--(axis cs:4,-490.805568333333)
--(axis cs:5,-487.968536666667)
--(axis cs:6,-476.278601666667)
--(axis cs:7,-472.210821666667)
--(axis cs:8,-453.431945)
--(axis cs:9,-438.669741666667)
--(axis cs:10,-430.642295)
--(axis cs:11,-405.203881666667)
--(axis cs:12,-369.349588333333)
--(axis cs:13,-342.511773333333)
--(axis cs:14,-324.787185)
--(axis cs:15,-298.964886666667)
--(axis cs:16,-264.428843333333)
--(axis cs:17,-247.432893333333)
--(axis cs:18,-231.115676666667)
--(axis cs:19,-227.604218333333)
--(axis cs:20,-222.121686666667)
--(axis cs:21,-217.291325)
--(axis cs:22,-213.306708333333)
--(axis cs:23,-206.126071666667)
--(axis cs:24,-197.959583333333)
--(axis cs:24,-134.874686666667)
--(axis cs:24,-134.874686666667)
--(axis cs:23,-137.808275)
--(axis cs:22,-138.902783333333)
--(axis cs:21,-145.37085)
--(axis cs:20,-147.191915)
--(axis cs:19,-166.927891666667)
--(axis cs:18,-168.444441666667)
--(axis cs:17,-187.658765)
--(axis cs:16,-203.39392)
--(axis cs:15,-237.447946666667)
--(axis cs:14,-263.207975)
--(axis cs:13,-279.480175)
--(axis cs:12,-309.367041666667)
--(axis cs:11,-346.458906666667)
--(axis cs:10,-390.516405)
--(axis cs:9,-395.821016666667)
--(axis cs:8,-412.654436666667)
--(axis cs:7,-444.655203333333)
--(axis cs:6,-458.10498)
--(axis cs:5,-477.416451666667)
--(axis cs:4,-486.696371666667)
--(axis cs:3,-491.048196666667)
--(axis cs:2,-493.493171666667)
--(axis cs:1,-497.204971666667)
--(axis cs:0,-498.26065)
--cycle;

\path [fill=tunecolor, fill opacity=0.2]
(axis cs:0,-484.624526025132)
--(axis cs:0,-497.262753260582)
--(axis cs:1,-470.52937)
--(axis cs:2,-429.12964)
--(axis cs:3,-398.608195)
--(axis cs:4,-333.054436666667)
--(axis cs:5,-277.630451666667)
--(axis cs:6,-245.486078333333)
--(axis cs:7,-228.970635)
--(axis cs:8,-217.560506666667)
--(axis cs:9,-214.254188333333)
--(axis cs:10,-207.957366666667)
--(axis cs:11,-201.925046666667)
--(axis cs:12,-194.430388333333)
--(axis cs:13,-190.050983333333)
--(axis cs:14,-181.068888333333)
--(axis cs:15,-178.438433333333)
--(axis cs:16,-178.085896666667)
--(axis cs:17,-161.789333333333)
--(axis cs:18,-149.066263333333)
--(axis cs:19,-139.73499)
--(axis cs:20,-130.810086666667)
--(axis cs:21,-120.732885)
--(axis cs:22,-123.204848333333)
--(axis cs:23,-124.27851)
--(axis cs:24,-122.415133333333)
--(axis cs:24,-110.463388333333)
--(axis cs:24,-110.463388333333)
--(axis cs:23,-111.35774)
--(axis cs:22,-111.293121666667)
--(axis cs:21,-111.02539)
--(axis cs:20,-112.950915)
--(axis cs:19,-115.495051666667)
--(axis cs:18,-115.349208333333)
--(axis cs:17,-115.903665)
--(axis cs:16,-119.880993333333)
--(axis cs:15,-120.621035)
--(axis cs:14,-121.801465)
--(axis cs:13,-124.00333)
--(axis cs:12,-124.86971)
--(axis cs:11,-129.184226666667)
--(axis cs:10,-138.903308333333)
--(axis cs:9,-140.544543333333)
--(axis cs:8,-146.412433333333)
--(axis cs:7,-161.41049)
--(axis cs:6,-177.65621)
--(axis cs:5,-218.687891666667)
--(axis cs:4,-275.283978333333)
--(axis cs:3,-335.214585)
--(axis cs:2,-368.892418333333)
--(axis cs:1,-414.913688333333)
--(axis cs:0,-484.624526025132)
--cycle;

\addplot [very thick, dashed, sourcecolor]
table {%
0 -499.634837253439
1 -391.605673
2 -215.240099633333
3 -182.786821633333
4 -155.685290066667
5 -143.6304846
6 -133.359091066667
7 -126.1229948
8 -122.745379766667
9 -116.649896533333
10 -105.736247133333
11 -104.6602855
12 -97.932516
13 -91.1144326
14 -92.4480361666667
15 -98.3195673333333
16 -99.8623343666666
17 -92.7817959333333
18 -88.8646283666667
19 -93.1368585666667
20 -91.5476631
21 -97.8618985
22 -91.8005003
23 -90.7939748666666
24 -94.9410862
};
\addlegendentry{Source (oracle)}
\addplot [ultra thick, ourcolor]
table {%
0 -498.611656598148
1 -484.443115866667
2 -441.1126486
3 -395.758449166667
4 -304.259430166667
5 -216.947274966667
6 -175.2316496
7 -157.459627133333
8 -144.530438566667
9 -143.059119833333
10 -135.741948866667
11 -129.060355833333
12 -126.784125533333
13 -121.469760833333
14 -116.948524966667
15 -116.294021633333
16 -114.722694233333
17 -112.615263133333
18 -113.7487319
19 -112.652899033333
20 -113.238390833333
21 -111.8509637
22 -109.022657733333
23 -110.694499666667
24 -107.9947475
};
\addlegendentry{Transfer (ours)}
\addplot [very thick, singlecolor]
table {%
0 -499.068418629762
1 -497.584427666667
2 -494.148149866667
3 -492.0558445
4 -491.6769096
5 -490.1167035
6 -476.285542166667
7 -465.865740633333
8 -440.410845433333
9 -405.419219966667
10 -367.089910533333
11 -334.293304166667
12 -304.341640033333
13 -282.4243423
14 -257.647752133333
15 -237.386268966667
16 -206.536173566667
17 -189.069853133333
18 -167.287920266667
19 -148.4596025
20 -141.936690233333
21 -133.219878433333
22 -130.574217166667
23 -129.729764266667
24 -128.097884966667
};
\addlegendentry{Single}
\addplot [very thick, auxcolor]
table {%
0 -498.288415239683
1 -497.2010002
2 -494.727176733333
3 -492.103680933333
4 -483.499375533333
5 -472.968346666667
6 -463.370377966667
7 -428.546566366667
8 -391.5573683
9 -352.203905766667
10 -327.070696266667
11 -291.901611133333
12 -264.9444929
13 -219.904808366667
14 -185.216689133333
15 -172.603010166667
16 -161.037981833333
17 -147.3461817
18 -136.379214033333
19 -126.918858566667
20 -119.609506366667
21 -116.307497666667
22 -112.361428666667
23 -112.869426866667
24 -115.411595466667
};
\addlegendentry{Auxiliary}
\addplot [very thick, tunecolor]
table {%
0 -491.437470168254
1 -443.8965385
2 -399.828371366667
3 -368.716115
4 -303.675578066667
5 -246.9707223
6 -209.496202833333
7 -192.744650533333
8 -178.696098466667
9 -175.374805333333
10 -170.843285033333
11 -162.768387966667
12 -157.272665566667
13 -154.0055966
14 -149.161647566667
15 -147.296724866667
16 -146.609421966667
17 -137.336257566667
18 -131.218163733333
19 -126.960408366667
20 -121.138416266667
21 -115.503102233333
22 -116.861168733333
23 -117.238002
24 -115.998674866667
};
\addlegendentry{Fine-tune}
\addplot [very thick, aligncolor]
table {%
0 -498.700157333333
1 -497.6932927
2 -494.355935933333
3 -491.8728655
4 -488.6991072
5 -482.682559766667
6 -467.443957266667
7 -458.964921566667
8 -434.245564666667
9 -417.860765966667
10 -410.876710133333
11 -376.5533821
12 -340.157942566667
13 -311.260164766667
14 -294.019290833333
15 -269.131668033333
16 -233.835197966667
17 -216.4442473
18 -198.356197266667
19 -194.858334666667
20 -181.3257982
21 -179.2726139
22 -173.255821
23 -170.792088466667
24 -164.961562133333
};
\addlegendentry{Time-aligned}
\end{axis}

\end{tikzpicture}

%% file: figs/cheetah_run.tex
\begin{tikzpicture}[scale=0.52]


\begin{axis}[
legend columns=6, 
legend cell align={left},
legend style={at={(0, 1.15)}, anchor=north west, fill opacity=0, draw opacity=0, text opacity=0, draw=white!80!black, font=\Large},
tick align=outside,
tick pos=left,
title={\Large{\textbf{Vec-to-pixel: Cheetah-Run}}},
x grid style={white!69.0196078431373!black},
xlabel={Steps},
xmajorgrids,
xmin=-0.45, xmax=9.45,
xtick style={color=black},
xtick={0,4.5,9},
xticklabels={0k,250k,500k},
y grid style={white!69.0196078431373!black},
ylabel={Returns},
ymajorgrids,
ymin=140.251007758175, ymax=828.649451241197,
ytick style={color=black}
]
\path [draw=sourcecolor, fill=sourcecolor, opacity=0.25]
(axis cs:0,207.902644970233)
--(axis cs:0,191.831502350144)
--(axis cs:1,492.868804727583)
--(axis cs:2,625.441181501587)
--(axis cs:3,669.529953663893)
--(axis cs:4,703.79727314691)
--(axis cs:5,724.881866761351)
--(axis cs:6,737.901576471254)
--(axis cs:7,726.850831995898)
--(axis cs:8,729.986990412421)
--(axis cs:9,771.744283134398)
--(axis cs:9,797.35861290106)
--(axis cs:9,797.35861290106)
--(axis cs:8,784.391168361048)
--(axis cs:7,770.653134139791)
--(axis cs:6,747.651307325264)
--(axis cs:5,755.926210281424)
--(axis cs:4,713.706105824865)
--(axis cs:3,700.168212600787)
--(axis cs:2,673.617267958567)
--(axis cs:1,557.991311233892)
--(axis cs:0,207.902644970233)
--cycle;

\path [draw=singlecolor, fill=singlecolor, opacity=0.25]
(axis cs:0,198.599346162565)
--(axis cs:0,179.90648532961)
--(axis cs:1,352.491600004381)
--(axis cs:2,359.113169894587)
--(axis cs:3,437.584637139243)
--(axis cs:4,386.752610030574)
--(axis cs:5,452.818534331418)
--(axis cs:6,462.105808570268)
--(axis cs:7,471.191969861788)
--(axis cs:8,453.048903454905)
--(axis cs:9,445.308119260931)
--(axis cs:9,487.638327319847)
--(axis cs:9,487.638327319847)
--(axis cs:8,497.236275536688)
--(axis cs:7,550.294100146159)
--(axis cs:6,540.011344156501)
--(axis cs:5,523.863952621867)
--(axis cs:4,482.942709818012)
--(axis cs:3,491.802205749171)
--(axis cs:2,394.895808965121)
--(axis cs:1,410.351959222161)
--(axis cs:0,198.599346162565)
--cycle;

\path [draw=ourcolor, fill=ourcolor, opacity=0.25]
(axis cs:0,194.695615934737)
--(axis cs:0,181.73516942581)
--(axis cs:1,381.591782834057)
--(axis cs:2,484.064497649354)
--(axis cs:3,529.471074610907)
--(axis cs:4,545.96857085794)
--(axis cs:5,552.831833543155)
--(axis cs:6,571.813618607036)
--(axis cs:7,590.290738588118)
--(axis cs:8,607.959134678948)
--(axis cs:9,609.654794634609)
--(axis cs:9,650.555468262409)
--(axis cs:9,650.555468262409)
--(axis cs:8,638.440856102104)
--(axis cs:7,620.590385737189)
--(axis cs:6,613.799976371042)
--(axis cs:5,599.681531704714)
--(axis cs:4,597.077679169069)
--(axis cs:3,574.954810771127)
--(axis cs:2,506.414795562167)
--(axis cs:1,409.801965262942)
--(axis cs:0,194.695615934737)
--cycle;

\path [draw=auxcolor, fill=auxcolor, opacity=0.25]
(axis cs:0,179.76384890887)
--(axis cs:0,171.541846098312)
--(axis cs:1,353.249077030134)
--(axis cs:2,400.089696733579)
--(axis cs:3,464.98596215118)
--(axis cs:4,475.899457545471)
--(axis cs:5,491.952488467856)
--(axis cs:6,520.019062792394)
--(axis cs:7,542.964427997587)
--(axis cs:8,553.56110178082)
--(axis cs:9,559.676898084762)
--(axis cs:9,602.783681967598)
--(axis cs:9,602.783681967598)
--(axis cs:8,596.508936046606)
--(axis cs:7,580.565456961485)
--(axis cs:6,571.983648514414)
--(axis cs:5,544.704383653909)
--(axis cs:4,536.089674579693)
--(axis cs:3,519.814345216449)
--(axis cs:2,418.468699548445)
--(axis cs:1,381.977559071068)
--(axis cs:0,179.76384890887)
--cycle;

\path [draw=tunecolor, fill=tunecolor, opacity=0.25]
(axis cs:0,203.892576796391)
--(axis cs:0,184.187514108616)
--(axis cs:1,384.909447632736)
--(axis cs:2,424.572008051)
--(axis cs:3,458.04122593703)
--(axis cs:4,463.658090110088)
--(axis cs:5,497.373296015004)
--(axis cs:6,510.131259629636)
--(axis cs:7,525.89085344572)
--(axis cs:8,554.496684504139)
--(axis cs:9,536.684024504506)
--(axis cs:9,603.356591580685)
--(axis cs:9,603.356591580685)
--(axis cs:8,611.549577950775)
--(axis cs:7,592.393702060146)
--(axis cs:6,574.212965662729)
--(axis cs:5,562.084659267693)
--(axis cs:4,577.500044361388)
--(axis cs:3,504.106605545737)
--(axis cs:2,480.407446928424)
--(axis cs:1,425.887864804737)
--(axis cs:0,203.892576796391)
--cycle;

\path [draw=aligncolor, fill=aligncolor, opacity=0.25]
(axis cs:0,182.041273462905)
--(axis cs:0,175.877699133955)
--(axis cs:1,318.312366426722)
--(axis cs:2,394.704002562536)
--(axis cs:3,434.674450710602)
--(axis cs:4,447.175654593568)
--(axis cs:5,450.486961511993)
--(axis cs:6,451.378458960343)
--(axis cs:7,459.322899371637)
--(axis cs:8,466.093117990537)
--(axis cs:9,471.559963321842)
--(axis cs:9,559.337017586039)
--(axis cs:9,559.337017586039)
--(axis cs:8,546.814678963202)
--(axis cs:7,539.359806817437)
--(axis cs:6,531.348091142683)
--(axis cs:5,521.783069423896)
--(axis cs:4,519.542444827647)
--(axis cs:3,504.178721909893)
--(axis cs:2,416.065233410349)
--(axis cs:1,363.253835734476)
--(axis cs:0,182.041273462905)
--cycle;

\addplot [very thick, dashed, sourcecolor]
table {%
0 200.655692826384
1 527.323075930595
2 651.178701691926
3 683.715487147097
4 708.245013827208
5 739.2891816659
6 742.452679838358
7 749.726465995097
8 758.672233885867
9 784.497096778985
};
\addlegendentry{Source}

\addplot [very thick, singlecolor]
table {%
0 188.881352921584
1 380.588011516177
2 377.571661941599
3 463.470845755998
4 434.298247272079
5 487.882063115174
6 499.00459206061
7 509.321815349304
8 475.440872186581
9 465.619134436305
};
\addlegendentry{Without Transfer}
\addplot [ultra thick, ourcolor]
table {%
0 188.533786432264
1 396.294870203832
2 495.348353775272
3 552.39623947499
4 572.295017352543
5 577.135093779887
6 593.466315660652
7 605.871462120734
8 623.459809100769
9 631.421007868601
};
\addlegendentry{With Transfer}
\addplot [very thick, auxcolor]
table {%
0 175.601051642895
1 365.275510941878
2 408.97268576091
3 491.615935301933
4 504.187546676194
5 517.282095880204
6 542.271566931003
7 561.404754995274
8 573.196614936404
9 581.147007512805
};
\addlegendentry{Auxiliary Tasks}
\addplot [very thick, tunecolor]
table {%
0 194.328733547028
1 404.690597848387
2 454.402345358587
3 481.630046464068
4 525.311715481507
5 528.78776446986
6 543.11126825726
7 561.528569107191
8 584.112119844749
9 568.531585324439
};
\addlegendentry{Fine-tune}
\addplot [very thick, aligncolor]
table {%
0 178.628417197807
1 340.036721283328
2 404.526601663728
3 465.04690228254
4 480.408384178008
5 484.224520196241
6 489.740666266114
7 495.70885221505
8 505.556841140709
9 514.006344102661
};
\addlegendentry{Time-aligned} 
\end{axis}

\end{tikzpicture}

%% file: figs/new_ball_withsource.tex
\begin{tikzpicture}[scale=0.42]


\begin{axis}[
legend cell align={left},
legend style={
  fill opacity=0.8,
  draw opacity=1,
  text opacity=1,
  at={(0.03,0.97)},
  anchor=north west,
  draw=white!80!black
},
tick align=outside,
tick pos=left,
title={\LARGE{\textbf{Broken-sensor: 3DBall}}},
x grid style={white!69.0196078431373!black},
xlabel={Steps},
xmajorgrids,
xmin=-14400, xmax=302400,
xtick style={color=black},
xticklabels={0k,50k,100k,150k,200k,250k,300k,350k},
y grid style={white!69.0196078431373!black},
ymajorgrids,
ymin=-4.11763916854858, ymax=103.696587476858,
ytick style={color=black},
ytick={-20,0,20,40,60,80,100,120},
yticklabels={\ensuremath{-}20,0,20,40,60,80,100,120}
]
\path [fill=sourcecolor, fill opacity=0.2]
(axis cs:0,0.903308645566305)
--(axis cs:0,0.871711607158184)
--(axis cs:12000,0.987231400589148)
--(axis cs:24000,0.943894989291827)
--(axis cs:36000,1.03180828779936)
--(axis cs:48000,1.50891705457369)
--(axis cs:60000,2.82048102990786)
--(axis cs:72000,5.45993475961685)
--(axis cs:84000,15.4604966166814)
--(axis cs:96000,33.4598157444)
--(axis cs:108000,43.9824004173279)
--(axis cs:120000,76.2718411865234)
--(axis cs:132000,81.9586740442912)
--(axis cs:144000,85.115667447408)
--(axis cs:156000,88.415674969991)
--(axis cs:168000,90.7334100443522)
--(axis cs:180000,81.3333051605225)
--(axis cs:192000,86.2792280502319)
--(axis cs:204000,80.9889865557353)
--(axis cs:216000,86.8476630401611)
--(axis cs:228000,89.1893850123088)
--(axis cs:240000,88.6008554890951)
--(axis cs:252000,86.3146730677287)
--(axis cs:264000,90.1903163426717)
--(axis cs:276000,91.8371038716634)
--(axis cs:288000,86.4798968785604)
--(axis cs:288000,94.7432353261312)
--(axis cs:288000,94.7432353261312)
--(axis cs:276000,97.9093761698405)
--(axis cs:264000,95.2866243693034)
--(axis cs:252000,94.9171648661296)
--(axis cs:240000,95.6541627197266)
--(axis cs:228000,96.7420734608968)
--(axis cs:216000,93.7457762908935)
--(axis cs:204000,91.8750539016724)
--(axis cs:192000,95.8270459899902)
--(axis cs:180000,90.8444880727132)
--(axis cs:168000,94.5312144037882)
--(axis cs:156000,95.9792774175008)
--(axis cs:144000,94.6821290639242)
--(axis cs:132000,90.9320371246338)
--(axis cs:120000,85.6414913457235)
--(axis cs:108000,62.5394353014628)
--(axis cs:96000,55.9955976537069)
--(axis cs:84000,32.7325391054153)
--(axis cs:72000,7.65651915383339)
--(axis cs:60000,3.26246020491918)
--(axis cs:48000,1.7248030222257)
--(axis cs:36000,1.10481607244412)
--(axis cs:24000,1.00478098056714)
--(axis cs:12000,1.01816495557626)
--(axis cs:0,0.903308645566305)
--cycle;

\path [fill=ourcolor, fill opacity=0.2]
(axis cs:0,0.871625437657038)
--(axis cs:0,0.860644424100717)
--(axis cs:12000,0.915082195182641)
--(axis cs:24000,0.958736839791139)
--(axis cs:36000,0.921717624743779)
--(axis cs:48000,0.889180102248987)
--(axis cs:60000,1.11202181665103)
--(axis cs:72000,2.02664709242185)
--(axis cs:84000,2.76258453591665)
--(axis cs:96000,5.70031883327166)
--(axis cs:108000,26.8373209489187)
--(axis cs:120000,48.6633837893804)
--(axis cs:132000,62.0952841323217)
--(axis cs:144000,84.250654103597)
--(axis cs:156000,88.0555812327067)
--(axis cs:168000,89.408560702006)
--(axis cs:180000,87.3934886627197)
--(axis cs:192000,90.2171062469482)
--(axis cs:204000,88.8721976242065)
--(axis cs:216000,90.328980392456)
--(axis cs:228000,92.4208336588542)
--(axis cs:240000,93.1353723398844)
--(axis cs:252000,93.9672748260498)
--(axis cs:264000,92.4426330235799)
--(axis cs:276000,91.5739183527629)
--(axis cs:288000,91.4239791692098)
--(axis cs:288000,96.272563448588)
--(axis cs:288000,96.272563448588)
--(axis cs:276000,94.7503578618367)
--(axis cs:264000,97.2289012476603)
--(axis cs:252000,98.7959408111572)
--(axis cs:240000,97.7723961970011)
--(axis cs:228000,97.2679335784912)
--(axis cs:216000,96.3426114069621)
--(axis cs:204000,95.3494445241292)
--(axis cs:192000,95.2471070175171)
--(axis cs:180000,93.9537504653931)
--(axis cs:168000,96.0904719899495)
--(axis cs:156000,93.0525354283651)
--(axis cs:144000,92.1936918258667)
--(axis cs:132000,82.6867164255778)
--(axis cs:120000,72.6221162447929)
--(axis cs:108000,46.2817374076843)
--(axis cs:96000,8.14689071027438)
--(axis cs:84000,3.72094519702593)
--(axis cs:72000,2.4979239616394)
--(axis cs:60000,1.29161346805096)
--(axis cs:48000,0.921084444781144)
--(axis cs:36000,0.942949629088243)
--(axis cs:24000,0.974967906494935)
--(axis cs:12000,0.93160862604777)
--(axis cs:0,0.871625437657038)
--cycle;

\path [fill=auxcolor, fill opacity=0.2]
(axis cs:0,0.905805739899476)
--(axis cs:0,0.885368306120237)
--(axis cs:12000,0.965733104149501)
--(axis cs:24000,1.00414194889863)
--(axis cs:36000,0.931317600886027)
--(axis cs:48000,0.934612187544505)
--(axis cs:60000,1.27319534242153)
--(axis cs:72000,2.18016920089722)
--(axis cs:84000,2.61367443358898)
--(axis cs:96000,8.55393337074915)
--(axis cs:108000,27.557117363294)
--(axis cs:120000,51.9643488550186)
--(axis cs:132000,59.820416115125)
--(axis cs:144000,80.2810039978027)
--(axis cs:156000,84.35094348526)
--(axis cs:168000,88.2710009155273)
--(axis cs:180000,85.1886484476725)
--(axis cs:192000,79.9485482076009)
--(axis cs:204000,81.8740868860881)
--(axis cs:216000,82.1456856892904)
--(axis cs:228000,86.1959813359578)
--(axis cs:240000,85.728098177592)
--(axis cs:252000,80.8249141235352)
--(axis cs:264000,85.1471888122558)
--(axis cs:276000,88.051809038798)
--(axis cs:288000,81.7235780817667)
--(axis cs:288000,87.652462214152)
--(axis cs:288000,87.652462214152)
--(axis cs:276000,95.7788760096232)
--(axis cs:264000,92.364907816569)
--(axis cs:252000,87.7491408640544)
--(axis cs:240000,92.4905410283407)
--(axis cs:228000,92.8073863296509)
--(axis cs:216000,87.8504647394816)
--(axis cs:204000,90.353605913798)
--(axis cs:192000,87.4051272150675)
--(axis cs:180000,90.7414795811971)
--(axis cs:168000,94.1404901479085)
--(axis cs:156000,93.8578888244629)
--(axis cs:144000,89.5411552073161)
--(axis cs:132000,84.5527831751506)
--(axis cs:120000,77.2534664389292)
--(axis cs:108000,52.0791354227066)
--(axis cs:96000,18.2121791407267)
--(axis cs:84000,5.06252516269684)
--(axis cs:72000,3.22756845629215)
--(axis cs:60000,1.49864455588659)
--(axis cs:48000,0.984663173476855)
--(axis cs:36000,0.963363782028357)
--(axis cs:24000,1.03275727277994)
--(axis cs:12000,0.994125642975171)
--(axis cs:0,0.905805739899476)
--cycle;

\path [fill=tunecolor, fill opacity=0.2]
(axis cs:0,0.874995684464772)
--(axis cs:0,0.84168517712752)
--(axis cs:12000,0.946955286939939)
--(axis cs:24000,0.97050917806228)
--(axis cs:36000,0.891174908717473)
--(axis cs:48000,0.913378522217274)
--(axis cs:60000,1.23605504763126)
--(axis cs:72000,1.8791417350769)
--(axis cs:84000,2.2165029305617)
--(axis cs:96000,4.52862208835284)
--(axis cs:108000,9.92273860025406)
--(axis cs:120000,29.2337754179637)
--(axis cs:132000,57.6751984608968)
--(axis cs:144000,70.8978962465922)
--(axis cs:156000,71.6475994593303)
--(axis cs:168000,79.2739204406738)
--(axis cs:180000,85.5513000361125)
--(axis cs:192000,81.6443401896159)
--(axis cs:204000,86.3139835968017)
--(axis cs:216000,87.6577226765951)
--(axis cs:228000,90.2648558705648)
--(axis cs:240000,81.2512622782389)
--(axis cs:252000,86.8824660288493)
--(axis cs:264000,80.7541072158813)
--(axis cs:276000,85.4869632593791)
--(axis cs:288000,86.6661014124552)
--(axis cs:288000,94.048179262797)
--(axis cs:288000,94.048179262797)
--(axis cs:276000,92.7303130620321)
--(axis cs:264000,90.2808197352091)
--(axis cs:252000,93.3074644266764)
--(axis cs:240000,92.234884437561)
--(axis cs:228000,95.1732055104574)
--(axis cs:216000,93.7700255381266)
--(axis cs:204000,91.2851739883423)
--(axis cs:192000,89.9365005620321)
--(axis cs:180000,90.2641435496012)
--(axis cs:168000,87.3695178705851)
--(axis cs:156000,86.4611587473551)
--(axis cs:144000,86.1159800987244)
--(axis cs:132000,75.8427429593404)
--(axis cs:120000,49.9005067907969)
--(axis cs:108000,17.8405123836199)
--(axis cs:96000,7.69193862112363)
--(axis cs:84000,3.39428303476175)
--(axis cs:72000,2.57137874114513)
--(axis cs:60000,1.4740597418944)
--(axis cs:48000,0.937584313511848)
--(axis cs:36000,0.921308933953444)
--(axis cs:24000,0.986099369307359)
--(axis cs:12000,0.97199740399917)
--(axis cs:0,0.874995684464772)
--cycle;

\path [fill=singlecolor, fill opacity=0.2]
(axis cs:0,0.911252064784368)
--(axis cs:0,0.8927995839715)
--(axis cs:12000,0.951025849262873)
--(axis cs:24000,0.986446036140124)
--(axis cs:36000,0.937717475136121)
--(axis cs:48000,0.978690633793672)
--(axis cs:60000,1.39164981945356)
--(axis cs:72000,1.6592542321682)
--(axis cs:84000,2.12150564153989)
--(axis cs:96000,5.31743733755748)
--(axis cs:108000,16.9152219683329)
--(axis cs:120000,29.1397227999369)
--(axis cs:132000,60.7728003756205)
--(axis cs:144000,71.3141520996094)
--(axis cs:156000,82.9091638005575)
--(axis cs:168000,78.8625378494263)
--(axis cs:180000,77.3544533716837)
--(axis cs:192000,72.7274548072815)
--(axis cs:204000,77.9918115692139)
--(axis cs:216000,81.2923531595866)
--(axis cs:228000,84.2753611475627)
--(axis cs:240000,77.1025668284098)
--(axis cs:252000,80.1320486272176)
--(axis cs:264000,79.930766263326)
--(axis cs:276000,85.6019685846965)
--(axis cs:288000,88.1199188283285)
--(axis cs:288000,93.7333292821248)
--(axis cs:288000,93.7333292821248)
--(axis cs:276000,92.8516988474528)
--(axis cs:264000,87.6471542282104)
--(axis cs:252000,91.4113543701172)
--(axis cs:240000,84.7398792559306)
--(axis cs:228000,91.8889360911052)
--(axis cs:216000,89.7607459615072)
--(axis cs:204000,84.2694863713582)
--(axis cs:192000,80.6896489651998)
--(axis cs:180000,86.8343160629272)
--(axis cs:168000,86.7147323582967)
--(axis cs:156000,90.423716158549)
--(axis cs:144000,85.6240335426331)
--(axis cs:132000,83.7949559542338)
--(axis cs:120000,54.9419808438619)
--(axis cs:108000,32.4010314741929)
--(axis cs:96000,9.83760951153437)
--(axis cs:84000,2.71120459051927)
--(axis cs:72000,2.22353105040391)
--(axis cs:60000,1.65286607408524)
--(axis cs:48000,1.05020876526833)
--(axis cs:36000,0.963441546797752)
--(axis cs:24000,1.01347631919384)
--(axis cs:12000,0.972499103963375)
--(axis cs:0,0.911252064784368)
--cycle;

\path [fill=aligncolor, fill opacity=0.2]
(axis cs:0,0.92910639133056)
--(axis cs:0,0.894845298151175)
--(axis cs:12000,0.939424009442329)
--(axis cs:24000,0.968710221469402)
--(axis cs:36000,0.940723136723042)
--(axis cs:48000,0.911604249000549)
--(axis cs:60000,0.898879960556825)
--(axis cs:72000,0.896907887121042)
--(axis cs:84000,0.783007497151693)
--(axis cs:96000,0.847254130005836)
--(axis cs:108000,0.875755931695302)
--(axis cs:120000,0.934759270350138)
--(axis cs:132000,0.933553871035576)
--(axis cs:144000,1.03107604895035)
--(axis cs:156000,1.09308454497655)
--(axis cs:168000,1.14202754672368)
--(axis cs:180000,1.18460846360524)
--(axis cs:192000,1.22836834228039)
--(axis cs:204000,1.28239678927263)
--(axis cs:216000,1.32153130364418)
--(axis cs:228000,1.39059544487794)
--(axis cs:240000,1.40644714081287)
--(axis cs:252000,1.44406306783358)
--(axis cs:264000,1.48467004768054)
--(axis cs:276000,1.48513570654392)
--(axis cs:288000,1.50048221401374)
--(axis cs:288000,1.6547232131958)
--(axis cs:288000,1.6547232131958)
--(axis cs:276000,1.62824330147107)
--(axis cs:264000,1.59897840078672)
--(axis cs:252000,1.60275475863616)
--(axis cs:240000,1.55988809756438)
--(axis cs:228000,1.50612354048093)
--(axis cs:216000,1.44587740282218)
--(axis cs:204000,1.41611930370331)
--(axis cs:192000,1.30065051631133)
--(axis cs:180000,1.28446306769053)
--(axis cs:168000,1.20934671159585)
--(axis cs:156000,1.17441085104148)
--(axis cs:144000,1.1374286629955)
--(axis cs:132000,1.05689710454146)
--(axis cs:120000,1.02742967641354)
--(axis cs:108000,0.996329165240129)
--(axis cs:96000,0.980610993345579)
--(axis cs:84000,0.92587605736653)
--(axis cs:72000,0.959590258618196)
--(axis cs:60000,0.933515660464764)
--(axis cs:48000,0.940485562225183)
--(axis cs:36000,0.982129106084506)
--(axis cs:24000,0.992746898710728)
--(axis cs:12000,0.959494814574718)
--(axis cs:0,0.92910639133056)
--cycle;

\addplot [very thick, sourcecolor, dashed]
table {%
0 0.887120589188735
12000 1.00310420569777
24000 0.973647081343333
36000 1.0708851737082
48000 1.62025599593321
60000 3.03774407928785
72000 6.52079530905883
84000 23.6562470047633
96000 44.1789941570918
108000 52.8095589041392
120000 81.3299398208618
132000 86.7258560312907
144000 90.1158864466349
156000 92.5642961858114
168000 92.5804604919433
180000 86.5213597278595
192000 91.2601729326884
204000 87.0155288002014
216000 90.3271565455119
228000 93.1731262084961
240000 92.5051381622314
252000 90.8536412656148
264000 92.8071088722229
276000 95.1567973124186
288000 90.8699078069051
};
\addlegendentry{Source}
\addplot [ultra thick, ourcolor]
table {%
0 0.866401356059313
12000 0.923311001018683
24000 0.96701280965209
36000 0.931685243507226
48000 0.906498474744956
60000 1.19851067503492
72000 2.2548865391175
84000 3.26397615234852
96000 6.92244067751567
108000 35.7219677598953
120000 61.3258530316273
132000 73.3042124208768
144000 88.3566184631348
156000 90.5830771540324
168000 93.0655771840413
180000 90.8069398681641
192000 92.7286668927511
204000 92.2117519381205
216000 93.4995480283101
228000 94.8429745071411
240000 95.5950374776204
252000 96.414679800415
264000 94.8352720072428
276000 93.2111942103068
288000 94.1583673487345
};
\addlegendentry{Transfer (ours)}
\addplot [very thick, auxcolor]
table {%
0 0.895902681656679
12000 0.979833387126525
24000 1.0184763508896
36000 0.948510885149241
48000 0.95974690305988
60000 1.37840375247002
72000 2.69024005479813
84000 3.68476462432146
96000 12.9580551016649
108000 39.5021263949235
120000 65.5367918412527
132000 73.1605666877747
144000 85.1131312466939
156000 89.312142550532
168000 91.3223297795614
180000 87.9618532541911
192000 83.7879649045309
204000 86.2551461034139
216000 85.1264831481934
228000 89.6602869178772
240000 88.9405688802083
252000 84.1869479265849
264000 88.5559531814575
276000 91.9160779424031
288000 84.9323735282898
};
\addlegendentry{Auxiliary Tasks}
\addplot [very thick, tunecolor]
table {%
0 0.858885349581639
12000 0.958904043545326
24000 0.978257468416293
36000 0.905377753184239
48000 0.925310653515657
60000 1.34669038524628
72000 2.19384251983166
84000 2.74066994591951
96000 6.01274616541068
108000 13.8058047488054
120000 39.0565788421631
132000 66.9568392032623
144000 78.093931202062
156000 79.8341773119609
168000 83.4351055745443
180000 87.9904428192139
192000 85.6358896471659
204000 88.6670669497172
216000 90.8542921844483
228000 92.8051968081157
240000 86.7528266199748
252000 90.0913125434876
264000 85.6330345631917
276000 89.2279020754496
288000 90.6049585665385
};
\addlegendentry{Finetune}
\addplot [very thick, singlecolor]
table {%
0 0.901887354258696
12000 0.961757814794779
24000 0.999253526208798
36000 0.951259256201982
48000 1.01464849662781
60000 1.52282360318899
72000 1.94145634954373
84000 2.41541351730029
96000 7.30588257573446
108000 24.3828908842564
120000 41.6985101258596
132000 73.3360431330363
144000 78.8841067995707
156000 86.7659146703084
168000 82.8157478574117
180000 82.3212273050944
192000 76.9279179140727
204000 81.4012720972697
216000 85.7488364351908
228000 88.0219869575501
240000 80.7562561981201
252000 85.36574610672
264000 83.7607510584513
276000 89.5937827163696
288000 90.7673574005127
};
\addlegendentry{Single}
\addplot [very thick, aligncolor]
table {%
0 0.912173792988062
12000 0.94968706390063
24000 0.980696122644345
36000 0.96198507878383
48000 0.926066106148561
60000 0.916447061846654
72000 0.926969946561257
84000 0.851995647833745
96000 0.913744378894567
108000 0.932520803290605
120000 0.979455595662197
132000 0.998756273496151
144000 1.08192697684765
156000 1.13498104462226
168000 1.17657233078082
180000 1.23298903288046
192000 1.26537021776835
204000 1.34979778952599
216000 1.37866251912117
228000 1.45190928396384
240000 1.47999770972729
252000 1.52904209483067
264000 1.54053526924849
276000 1.55249493321975
288000 1.5758259212772
};
\addlegendentry{Time-align}
\legend{};
\end{axis}

\end{tikzpicture}

%% file: figs/halfcheetah_ablation_small.tex
\begin{tikzpicture}[scale=0.55]

\definecolor{color0}{rgb}{1,0.549019607843137,0}
\definecolor{color1}{rgb}{0.55,0.27,0.07}
\definecolor{color2}{rgb}{0.580392156862745,0.403921568627451,0.741176470588235}
\begin{axis}[
legend cell align={left},
legend style={
  fill opacity=0.8,
  draw opacity=1,
  text opacity=1,
  at={(0.97,0.03)},
  anchor=south east,
  draw=white!80!black
},
tick align=outside,
tick pos=left,
title={{HalfCheetah}},
x grid style={white!69.0196078431373!black},
xlabel={{Steps}},
xmajorgrids,
xmin=-4.95, xmax=103.95,
xtick style={color=black},
xtick={0,20,40,60,80,100},
xticklabels={0k,200k,400k,600k,800k,1000k},
y grid style={white!69.0196078431373!black},
ylabel={{Return}},
ymajorgrids,
ymin=-266.200568040981, ymax=5307.13375964089,
ytick style={color=black}
]

\path [draw=blue, fill=blue, opacity=0.25]
(axis cs:0,-4.69312890992943)
--(axis cs:0,-17.6489278332041)
--(axis cs:1,48.4196230838893)
--(axis cs:2,224.085892993268)
--(axis cs:3,433.899416829319)
--(axis cs:4,651.374335075094)
--(axis cs:5,773.396204019732)
--(axis cs:6,705.042808075804)
--(axis cs:7,760.161541903011)
--(axis cs:8,803.536662605523)
--(axis cs:9,875.682041316625)
--(axis cs:10,885.637977920961)
--(axis cs:11,953.945708140717)
--(axis cs:12,933.478008970819)
--(axis cs:13,936.394664162679)
--(axis cs:14,976.784508843318)
--(axis cs:15,1044.01115949428)
--(axis cs:16,1122.89113824202)
--(axis cs:17,1171.63907748341)
--(axis cs:18,1206.67203424047)
--(axis cs:19,1260.59090609802)
--(axis cs:20,1250.54455430699)
--(axis cs:21,1354.73887332777)
--(axis cs:22,1427.61495406515)
--(axis cs:23,1541.55730138082)
--(axis cs:24,1498.62922634593)
--(axis cs:25,1512.56763950996)
--(axis cs:26,1584.96851908096)
--(axis cs:27,1661.32729266277)
--(axis cs:28,1658.7909143109)
--(axis cs:29,1671.01726694997)
--(axis cs:30,1700.57198310876)
--(axis cs:31,1725.43229780922)
--(axis cs:32,1784.67324464652)
--(axis cs:33,1813.37514601896)
--(axis cs:34,1814.57075376337)
--(axis cs:35,1827.30362641694)
--(axis cs:36,1889.85165307179)
--(axis cs:37,1849.35630621537)
--(axis cs:38,1901.63608022612)
--(axis cs:39,1939.34565873711)
--(axis cs:40,1918.4598069219)
--(axis cs:41,1923.91189121227)
--(axis cs:42,1915.66353764437)
--(axis cs:43,1911.7823852757)
--(axis cs:44,1963.88424897745)
--(axis cs:45,1855.73292242444)
--(axis cs:46,1964.81377241806)
--(axis cs:47,1950.87120639778)
--(axis cs:48,1988.57636764583)
--(axis cs:49,1951.72755093991)
--(axis cs:50,1975.58382666033)
--(axis cs:51,2002.22142456002)
--(axis cs:52,1998.59455882182)
--(axis cs:53,2092.82886829939)
--(axis cs:54,2059.58402119079)
--(axis cs:55,2057.71443926768)
--(axis cs:56,2058.7937248732)
--(axis cs:57,2096.41509010125)
--(axis cs:58,2088.23064204906)
--(axis cs:59,2135.95653027454)
--(axis cs:60,2112.0276991081)
--(axis cs:61,2094.14393708331)
--(axis cs:62,2129.68823208115)
--(axis cs:63,2150.22678577463)
--(axis cs:64,2105.6983531378)
--(axis cs:65,2196.73761359791)
--(axis cs:66,2124.17689335286)
--(axis cs:67,2123.24469944621)
--(axis cs:68,2250.66183458359)
--(axis cs:69,2227.34996481021)
--(axis cs:70,2280.03528442713)
--(axis cs:71,2200.90614397843)
--(axis cs:72,2259.90838042161)
--(axis cs:73,2234.31113634627)
--(axis cs:74,2268.24493270003)
--(axis cs:75,2221.58352396378)
--(axis cs:76,2314.9714861276)
--(axis cs:77,2351.01458754251)
--(axis cs:78,2351.23930258279)
--(axis cs:79,2342.63091426649)
--(axis cs:80,2355.17721514452)
--(axis cs:81,2393.87269598959)
--(axis cs:82,2385.48619678364)
--(axis cs:83,2384.0409035733)
--(axis cs:84,2440.83991414794)
--(axis cs:85,2439.51289913974)
--(axis cs:86,2494.02581088025)
--(axis cs:87,2401.87644951462)
--(axis cs:88,2514.24933260454)
--(axis cs:89,2467.42468478773)
--(axis cs:90,2431.06380696522)
--(axis cs:91,2482.89499648818)
--(axis cs:92,2523.12975903549)
--(axis cs:93,2530.47036770005)
--(axis cs:94,2515.39463403946)
--(axis cs:95,2514.90800725056)
--(axis cs:96,2516.38323946424)
--(axis cs:97,2591.70943577726)
--(axis cs:98,2526.40176067548)
--(axis cs:99,2473.75989093863)
--(axis cs:99,3454.32196767895)
--(axis cs:99,3454.32196767895)
--(axis cs:98,3609.73378916989)
--(axis cs:97,3611.54469969064)
--(axis cs:96,3592.85436087659)
--(axis cs:95,3547.86391986894)
--(axis cs:94,3634.7367440994)
--(axis cs:93,3552.24064280243)
--(axis cs:92,3516.38991959465)
--(axis cs:91,3517.31199253149)
--(axis cs:90,3537.09571141303)
--(axis cs:89,3470.03184185023)
--(axis cs:88,3521.08294793979)
--(axis cs:87,3426.38840446597)
--(axis cs:86,3388.77013027803)
--(axis cs:85,3435.71225686408)
--(axis cs:84,3344.13206993206)
--(axis cs:83,3398.45591905594)
--(axis cs:82,3350.99659979894)
--(axis cs:81,3282.67761419759)
--(axis cs:80,3280.49833740014)
--(axis cs:79,3198.24259859315)
--(axis cs:78,3300.44430820178)
--(axis cs:77,3335.93305012482)
--(axis cs:76,3222.92360387623)
--(axis cs:75,3205.36273603947)
--(axis cs:74,3211.87898550048)
--(axis cs:73,3172.25999671678)
--(axis cs:72,3109.84155968791)
--(axis cs:71,3124.96787720871)
--(axis cs:70,3163.06425688898)
--(axis cs:69,3074.99848861961)
--(axis cs:68,3116.20696424908)
--(axis cs:67,3005.83997643866)
--(axis cs:66,3012.22194145395)
--(axis cs:65,3017.23617523856)
--(axis cs:64,2929.36030729206)
--(axis cs:63,3013.96105719838)
--(axis cs:62,3042.2285707763)
--(axis cs:61,2995.49006578194)
--(axis cs:60,3004.89135106062)
--(axis cs:59,2926.32790914883)
--(axis cs:58,2887.60308914719)
--(axis cs:57,2962.38089581548)
--(axis cs:56,2857.88003498309)
--(axis cs:55,2902.30851668398)
--(axis cs:54,2906.06224826494)
--(axis cs:53,2943.71297474148)
--(axis cs:52,2905.90897586828)
--(axis cs:51,2808.82830172019)
--(axis cs:50,2822.91882513191)
--(axis cs:49,2836.74555430608)
--(axis cs:48,2818.795029409)
--(axis cs:47,2829.0425858007)
--(axis cs:46,2794.79586331184)
--(axis cs:45,2666.39236658714)
--(axis cs:44,2745.8878524189)
--(axis cs:43,2679.87048234605)
--(axis cs:42,2660.19393199652)
--(axis cs:41,2670.46585951361)
--(axis cs:40,2628.65629460396)
--(axis cs:39,2618.58010609789)
--(axis cs:38,2604.86259578869)
--(axis cs:37,2533.24865148957)
--(axis cs:36,2592.74091976821)
--(axis cs:35,2495.3427092436)
--(axis cs:34,2500.6602469018)
--(axis cs:33,2461.90140681)
--(axis cs:32,2446.6558919426)
--(axis cs:31,2404.01450603108)
--(axis cs:30,2285.02492142732)
--(axis cs:29,2295.6636825512)
--(axis cs:28,2260.34917269084)
--(axis cs:27,2298.37364203994)
--(axis cs:26,2182.01495371068)
--(axis cs:25,2106.22151709448)
--(axis cs:24,1980.42603690572)
--(axis cs:23,2026.30751491157)
--(axis cs:22,1945.72167887143)
--(axis cs:21,1889.87562858262)
--(axis cs:20,1754.79374489955)
--(axis cs:19,1827.83810516501)
--(axis cs:18,1733.65633603928)
--(axis cs:17,1700.04953086017)
--(axis cs:16,1664.84087711702)
--(axis cs:15,1584.29110470418)
--(axis cs:14,1487.0887834827)
--(axis cs:13,1471.997757771)
--(axis cs:12,1367.02244697723)
--(axis cs:11,1422.12400642221)
--(axis cs:10,1359.79623665067)
--(axis cs:9,1304.89186705402)
--(axis cs:8,1230.02892686611)
--(axis cs:7,1106.97423373146)
--(axis cs:6,1111.60031346159)
--(axis cs:5,1079.7081359667)
--(axis cs:4,934.77096809036)
--(axis cs:3,726.040469646256)
--(axis cs:2,513.026876317791)
--(axis cs:1,157.085270087035)
--(axis cs:0,-4.69312890992943)
--cycle;

\path [draw=color0, fill=color0, opacity=0.25]
(axis cs:0,80.3293013688214)
--(axis cs:0,-6.10413996843572)
--(axis cs:1,1.96385391655971)
--(axis cs:2,87.8233312662103)
--(axis cs:3,177.470593422115)
--(axis cs:4,379.493481424522)
--(axis cs:5,610.940630404371)
--(axis cs:6,876.376382222473)
--(axis cs:7,944.260809085335)
--(axis cs:8,1037.73804800724)
--(axis cs:9,1425.38550181593)
--(axis cs:10,1395.12031593969)
--(axis cs:11,1420.26866359989)
--(axis cs:12,1554.39738985974)
--(axis cs:13,1626.39228487944)
--(axis cs:14,1561.89921511591)
--(axis cs:15,1634.70611769191)
--(axis cs:16,1798.9681001895)
--(axis cs:17,1824.6406208857)
--(axis cs:18,1950.66523641905)
--(axis cs:19,1989.8055509507)
--(axis cs:20,1958.05430432234)
--(axis cs:21,2030.06732672104)
--(axis cs:22,2090.68331689381)
--(axis cs:23,2154.11892517449)
--(axis cs:24,2165.43150621581)
--(axis cs:25,2199.57548263918)
--(axis cs:26,2277.87237973822)
--(axis cs:27,2286.41734257478)
--(axis cs:28,2294.00306000102)
--(axis cs:29,2364.62275348257)
--(axis cs:30,2401.89922881462)
--(axis cs:31,2444.73536473869)
--(axis cs:32,2401.30835854457)
--(axis cs:33,2412.60320226658)
--(axis cs:34,2607.80548064566)
--(axis cs:35,2587.40610266662)
--(axis cs:36,2688.03923423811)
--(axis cs:37,2720.5866462361)
--(axis cs:38,2799.94348210032)
--(axis cs:39,2849.61085321679)
--(axis cs:40,2761.22828345932)
--(axis cs:41,2837.25642415158)
--(axis cs:42,2856.07354049166)
--(axis cs:43,2886.78885989495)
--(axis cs:44,2881.29687084922)
--(axis cs:45,2940.82371234163)
--(axis cs:46,2936.1212674993)
--(axis cs:47,2975.34635378839)
--(axis cs:48,2982.27189186647)
--(axis cs:49,3039.33367621624)
--(axis cs:50,3054.21427075625)
--(axis cs:51,3041.23172999995)
--(axis cs:52,3031.99051905439)
--(axis cs:53,3093.13023819287)
--(axis cs:54,3096.93889394751)
--(axis cs:55,3145.39664477085)
--(axis cs:56,3189.40504447203)
--(axis cs:57,3144.22560431442)
--(axis cs:58,3173.53236797083)
--(axis cs:59,3139.50715347467)
--(axis cs:60,3229.54730031778)
--(axis cs:61,3184.14449092257)
--(axis cs:62,3236.34354283633)
--(axis cs:63,3343.67310978361)
--(axis cs:64,3311.61403909489)
--(axis cs:65,3260.74127647395)
--(axis cs:66,3369.7760055498)
--(axis cs:67,3259.13770052641)
--(axis cs:68,3195.3508435027)
--(axis cs:69,3238.53681688386)
--(axis cs:70,3372.07663477743)
--(axis cs:71,3334.58152508458)
--(axis cs:72,3294.95708997173)
--(axis cs:73,3432.26667303332)
--(axis cs:74,3288.51031289589)
--(axis cs:75,3389.86757570329)
--(axis cs:76,3506.95419687701)
--(axis cs:77,3474.64528598548)
--(axis cs:78,3450.97698208801)
--(axis cs:79,3455.91194600123)
--(axis cs:80,3490.43171913994)
--(axis cs:81,3506.87753387765)
--(axis cs:82,3433.33970472685)
--(axis cs:83,3515.35654452105)
--(axis cs:84,3494.41578281947)
--(axis cs:85,3531.79053528634)
--(axis cs:86,3554.97718954374)
--(axis cs:87,3527.24146749828)
--(axis cs:88,3572.01671791209)
--(axis cs:89,3526.32921748129)
--(axis cs:90,3486.84349492241)
--(axis cs:91,3535.81242951268)
--(axis cs:92,3543.90387632925)
--(axis cs:93,3590.91333412486)
--(axis cs:94,3554.58046135999)
--(axis cs:95,3518.37739556343)
--(axis cs:96,3506.25475322572)
--(axis cs:97,3605.58390313171)
--(axis cs:98,3668.49924910747)
--(axis cs:99,3574.03466078758)
--(axis cs:99,4563.21164169044)
--(axis cs:99,4563.21164169044)
--(axis cs:98,4733.74427137818)
--(axis cs:97,4598.24389745897)
--(axis cs:96,4548.61062632993)
--(axis cs:95,4596.94834305126)
--(axis cs:94,4618.54219847214)
--(axis cs:93,4556.87514447651)
--(axis cs:92,4587.20225016785)
--(axis cs:91,4585.25638225635)
--(axis cs:90,4510.3701587507)
--(axis cs:89,4647.51742548549)
--(axis cs:88,4579.93993390898)
--(axis cs:87,4538.5138904151)
--(axis cs:86,4581.76870341598)
--(axis cs:85,4522.39781656194)
--(axis cs:84,4560.91114867896)
--(axis cs:83,4536.63498445484)
--(axis cs:82,4524.90776115001)
--(axis cs:81,4548.87903284887)
--(axis cs:80,4487.03068023991)
--(axis cs:79,4528.59220450173)
--(axis cs:78,4458.39381192085)
--(axis cs:77,4503.32790838365)
--(axis cs:76,4554.38041480546)
--(axis cs:75,4458.69236112002)
--(axis cs:74,4363.12226921224)
--(axis cs:73,4483.77770365669)
--(axis cs:72,4336.1033260055)
--(axis cs:71,4366.23662732267)
--(axis cs:70,4338.50118037157)
--(axis cs:69,4246.00667224316)
--(axis cs:68,4323.67186333415)
--(axis cs:67,4275.45391017837)
--(axis cs:66,4399.59947211684)
--(axis cs:65,4388.37713641696)
--(axis cs:64,4295.75316302207)
--(axis cs:63,4313.36513093371)
--(axis cs:62,4275.97106438637)
--(axis cs:61,4140.89152926772)
--(axis cs:60,4223.24436639783)
--(axis cs:59,4203.66294215935)
--(axis cs:58,4164.10767947866)
--(axis cs:57,4129.5304445731)
--(axis cs:56,4203.56083608163)
--(axis cs:55,4069.19710819088)
--(axis cs:54,4118.02246204486)
--(axis cs:53,4174.5332031276)
--(axis cs:52,4044.91661129699)
--(axis cs:51,4070.45493347925)
--(axis cs:50,4046.29162787029)
--(axis cs:49,4050.21615745851)
--(axis cs:48,3977.41258630031)
--(axis cs:47,3957.36043580442)
--(axis cs:46,3949.32831055728)
--(axis cs:45,4017.76265050416)
--(axis cs:44,3847.19391728282)
--(axis cs:43,3935.9591971133)
--(axis cs:42,3859.60430375916)
--(axis cs:41,3849.54859462534)
--(axis cs:40,3790.8220226594)
--(axis cs:39,3803.52279287445)
--(axis cs:38,3782.3308539618)
--(axis cs:37,3662.89951522716)
--(axis cs:36,3647.16597777271)
--(axis cs:35,3581.39059609381)
--(axis cs:34,3596.53851468752)
--(axis cs:33,3439.84256447943)
--(axis cs:32,3344.14330862945)
--(axis cs:31,3445.62787082704)
--(axis cs:30,3360.62619189301)
--(axis cs:29,3282.93036786349)
--(axis cs:28,3301.64888504898)
--(axis cs:27,3260.09473730749)
--(axis cs:26,3279.44211659119)
--(axis cs:25,3136.92828161725)
--(axis cs:24,3060.76550809105)
--(axis cs:23,3086.42627100675)
--(axis cs:22,3006.75152399056)
--(axis cs:21,2893.40870284598)
--(axis cs:20,2865.09245408262)
--(axis cs:19,2873.25532716587)
--(axis cs:18,2792.96777830534)
--(axis cs:17,2744.39825597597)
--(axis cs:16,2618.43101264525)
--(axis cs:15,2457.4030940497)
--(axis cs:14,2267.14572631033)
--(axis cs:13,2399.72640528198)
--(axis cs:12,2196.60183628725)
--(axis cs:11,2181.74314315018)
--(axis cs:10,2190.71102119483)
--(axis cs:9,2172.7479495633)
--(axis cs:8,1822.54287262797)
--(axis cs:7,1473.38880974864)
--(axis cs:6,1308.74924517248)
--(axis cs:5,945.724393337461)
--(axis cs:4,715.272795757182)
--(axis cs:3,507.581506681246)
--(axis cs:2,410.377279770895)
--(axis cs:1,79.5503322263406)
--(axis cs:0,80.3293013688214)
--cycle;

\path [draw=color1, fill=color1, opacity=0.25]
(axis cs:0,24.2261254993369)
--(axis cs:0,-4.51020156194292)
--(axis cs:1,-1.80116644201299)
--(axis cs:2,117.612136321464)
--(axis cs:3,572.772428305984)
--(axis cs:4,738.428286102509)
--(axis cs:5,853.121822677509)
--(axis cs:6,905.123300905511)
--(axis cs:7,986.16003403064)
--(axis cs:8,1050.28720709853)
--(axis cs:9,1097.4426825671)
--(axis cs:10,1098.47772394299)
--(axis cs:11,1155.41739351208)
--(axis cs:12,1180.88627573804)
--(axis cs:13,1226.74153046599)
--(axis cs:14,1274.87073049647)
--(axis cs:15,1310.01884896546)
--(axis cs:16,1319.92774816757)
--(axis cs:17,1343.36412193719)
--(axis cs:18,1367.80697718297)
--(axis cs:19,1405.93919460333)
--(axis cs:20,1439.96494504112)
--(axis cs:21,1472.03120421769)
--(axis cs:22,1488.37973083063)
--(axis cs:23,1519.47578351701)
--(axis cs:24,1529.56749312731)
--(axis cs:25,1560.87545969154)
--(axis cs:26,1592.66580211334)
--(axis cs:27,1605.55764173378)
--(axis cs:28,1646.15460664854)
--(axis cs:29,1623.1039922337)
--(axis cs:30,1650.3328575982)
--(axis cs:31,1680.76291656206)
--(axis cs:32,1676.48761332775)
--(axis cs:33,1677.78371931962)
--(axis cs:34,1720.33081859767)
--(axis cs:35,1693.91800059807)
--(axis cs:36,1754.43578144849)
--(axis cs:37,1719.91762082148)
--(axis cs:38,1762.59082852844)
--(axis cs:39,1806.46502406104)
--(axis cs:40,1756.88736366951)
--(axis cs:41,1768.47663911663)
--(axis cs:42,1815.01228508833)
--(axis cs:43,1831.43779460445)
--(axis cs:44,1859.55306414583)
--(axis cs:45,1852.89190678897)
--(axis cs:46,1858.33868886623)
--(axis cs:47,1945.06402269248)
--(axis cs:48,1921.49700371504)
--(axis cs:49,1942.83844392141)
--(axis cs:50,1953.67439193923)
--(axis cs:51,2004.1969047332)
--(axis cs:52,2016.82667341245)
--(axis cs:53,2041.15557244076)
--(axis cs:54,2103.31728939404)
--(axis cs:55,2145.4399435915)
--(axis cs:56,2161.47187840897)
--(axis cs:57,2166.78679653473)
--(axis cs:58,2173.87549480597)
--(axis cs:59,2147.91581584204)
--(axis cs:60,2201.89707320127)
--(axis cs:61,2246.84343781432)
--(axis cs:62,2243.51988642852)
--(axis cs:63,2239.28088923078)
--(axis cs:64,2294.30527797646)
--(axis cs:65,2186.36696370271)
--(axis cs:66,2290.82597629219)
--(axis cs:67,2312.83802275269)
--(axis cs:68,2306.05329724227)
--(axis cs:69,2300.25234668021)
--(axis cs:70,2268.91630148821)
--(axis cs:71,2293.16548723638)
--(axis cs:72,2316.86348509291)
--(axis cs:73,2305.59488785186)
--(axis cs:74,2392.05529049324)
--(axis cs:75,2337.70840080518)
--(axis cs:76,2315.44959911267)
--(axis cs:77,2352.66052182162)
--(axis cs:78,2375.65055446111)
--(axis cs:79,2382.79566435707)
--(axis cs:80,2463.23170319693)
--(axis cs:81,2399.18123917645)
--(axis cs:82,2451.66003549811)
--(axis cs:83,2410.15162490413)
--(axis cs:84,2437.74218384354)
--(axis cs:85,2460.9067982772)
--(axis cs:86,2513.98702263173)
--(axis cs:87,2559.6055537063)
--(axis cs:88,2520.71789405971)
--(axis cs:89,2489.81102292544)
--(axis cs:90,2563.44704165788)
--(axis cs:91,2425.91243212712)
--(axis cs:92,2548.07126658214)
--(axis cs:93,2563.70333961568)
--(axis cs:94,2549.06517522773)
--(axis cs:95,2528.10298311439)
--(axis cs:96,2582.33047247445)
--(axis cs:97,2627.50164594706)
--(axis cs:98,2543.44665621198)
--(axis cs:99,2575.73595869099)
--(axis cs:99,4189.09414129455)
--(axis cs:99,4189.09414129455)
--(axis cs:98,4222.85190779033)
--(axis cs:97,4145.36680955847)
--(axis cs:96,4146.76132151378)
--(axis cs:95,4140.74506760534)
--(axis cs:94,4083.59391200874)
--(axis cs:93,4028.00718375438)
--(axis cs:92,4137.14438296203)
--(axis cs:91,3877.85838879896)
--(axis cs:90,3995.8641679712)
--(axis cs:89,3974.30537648029)
--(axis cs:88,4117.31013956378)
--(axis cs:87,4041.36119583785)
--(axis cs:86,3987.36371253588)
--(axis cs:85,3856.78506152618)
--(axis cs:84,3882.45416590113)
--(axis cs:83,3858.07739502629)
--(axis cs:82,3908.28153547251)
--(axis cs:81,3861.93871166612)
--(axis cs:80,3886.24811356819)
--(axis cs:79,3852.96397870503)
--(axis cs:78,3863.96152429133)
--(axis cs:77,3649.69002023653)
--(axis cs:76,3626.49160310709)
--(axis cs:75,3589.06119150634)
--(axis cs:74,3614.78498207497)
--(axis cs:73,3542.18491469192)
--(axis cs:72,3532.36582112409)
--(axis cs:71,3559.77685751654)
--(axis cs:70,3405.81466694667)
--(axis cs:69,3453.39897303487)
--(axis cs:68,3476.54713601714)
--(axis cs:67,3518.90321915476)
--(axis cs:66,3389.02830739906)
--(axis cs:65,3336.32507234969)
--(axis cs:64,3340.69241044602)
--(axis cs:63,3292.49875288296)
--(axis cs:62,3259.56889690421)
--(axis cs:61,3267.04382561445)
--(axis cs:60,3197.30397040561)
--(axis cs:59,3196.55571982304)
--(axis cs:58,3220.30610998626)
--(axis cs:57,3142.20531589811)
--(axis cs:56,3190.0211410354)
--(axis cs:55,3138.80279563834)
--(axis cs:54,3107.23562664697)
--(axis cs:53,3010.33053194988)
--(axis cs:52,3004.60681038015)
--(axis cs:51,2939.72137678403)
--(axis cs:50,2942.1048269851)
--(axis cs:49,2948.15235761484)
--(axis cs:48,2793.86396801984)
--(axis cs:47,2866.18599004738)
--(axis cs:46,2742.92602442177)
--(axis cs:45,2693.50479293692)
--(axis cs:44,2679.11260592283)
--(axis cs:43,2700.52576831334)
--(axis cs:42,2657.83108263532)
--(axis cs:41,2686.08112598459)
--(axis cs:40,2578.03980539098)
--(axis cs:39,2605.69743132684)
--(axis cs:38,2628.19506968759)
--(axis cs:37,2565.58483703954)
--(axis cs:36,2531.26018425)
--(axis cs:35,2551.72215309384)
--(axis cs:34,2507.78460958715)
--(axis cs:33,2476.31608233635)
--(axis cs:32,2466.31902766929)
--(axis cs:31,2450.89271577406)
--(axis cs:30,2392.67522255248)
--(axis cs:29,2350.83273064836)
--(axis cs:28,2332.28977671714)
--(axis cs:27,2290.0779587064)
--(axis cs:26,2265.81911799326)
--(axis cs:25,2253.50178456486)
--(axis cs:24,2252.17414301039)
--(axis cs:23,2233.49306219714)
--(axis cs:22,2150.8585804693)
--(axis cs:21,2184.2610350763)
--(axis cs:20,2153.9840167293)
--(axis cs:19,2059.88101711186)
--(axis cs:18,2040.46590496512)
--(axis cs:17,2006.85466052475)
--(axis cs:16,2007.10413628188)
--(axis cs:15,1917.92197427493)
--(axis cs:14,1884.78102678594)
--(axis cs:13,1869.270457375)
--(axis cs:12,1794.08153395784)
--(axis cs:11,1730.48464551681)
--(axis cs:10,1625.90043396206)
--(axis cs:9,1552.82078560329)
--(axis cs:8,1399.78489604456)
--(axis cs:7,1153.33656119734)
--(axis cs:6,1033.31009329836)
--(axis cs:5,969.732689431078)
--(axis cs:4,824.55437944107)
--(axis cs:3,729.343101364136)
--(axis cs:2,363.606222875996)
--(axis cs:1,91.8919013479733)
--(axis cs:0,24.2261254993369)
--cycle;

\path [draw=red, fill=red, opacity=0.25]
(axis cs:0,48.7095079349399)
--(axis cs:0,-7.30785449160513)
--(axis cs:1,-12.8671895099869)
--(axis cs:2,10.6498819523004)
--(axis cs:3,77.8329103459216)
--(axis cs:4,226.505065832334)
--(axis cs:5,332.9102716835)
--(axis cs:6,491.965473894132)
--(axis cs:7,694.687646189033)
--(axis cs:8,904.038796941001)
--(axis cs:9,1081.32346456871)
--(axis cs:10,1262.55849715297)
--(axis cs:11,1441.31703435582)
--(axis cs:12,1571.73020799688)
--(axis cs:13,1726.50278855384)
--(axis cs:14,1783.94029175917)
--(axis cs:15,1809.68170192985)
--(axis cs:16,1906.21425479935)
--(axis cs:17,2064.93207550842)
--(axis cs:18,2094.3447485723)
--(axis cs:19,2104.87584622365)
--(axis cs:20,2122.11215120852)
--(axis cs:21,2214.14075691662)
--(axis cs:22,2255.9422141265)
--(axis cs:23,2183.27165135618)
--(axis cs:24,2173.15984479526)
--(axis cs:25,2204.6868887364)
--(axis cs:26,2391.91268420802)
--(axis cs:27,2485.61933074168)
--(axis cs:28,2471.19384086682)
--(axis cs:29,2484.96410872609)
--(axis cs:30,2727.27741833266)
--(axis cs:31,2466.6909197146)
--(axis cs:32,2585.02067245254)
--(axis cs:33,2731.17817087167)
--(axis cs:34,2691.49595737574)
--(axis cs:35,2860.32444196475)
--(axis cs:36,2952.98966261781)
--(axis cs:37,2926.24023357579)
--(axis cs:38,2861.54418243717)
--(axis cs:39,2996.85218154587)
--(axis cs:40,3039.14487585924)
--(axis cs:41,3136.12495157322)
--(axis cs:42,3140.10190717849)
--(axis cs:43,3053.5841490409)
--(axis cs:44,3170.67937112573)
--(axis cs:45,3213.18976005074)
--(axis cs:46,3315.59702707421)
--(axis cs:47,3183.24674127966)
--(axis cs:48,3457.68790378846)
--(axis cs:49,3429.98463481286)
--(axis cs:50,3467.88282574794)
--(axis cs:51,3549.17596360765)
--(axis cs:52,3500.89674905224)
--(axis cs:53,3621.18426775035)
--(axis cs:54,3625.27895069009)
--(axis cs:55,3596.03759616266)
--(axis cs:56,3740.71021460753)
--(axis cs:57,3908.61560223821)
--(axis cs:58,3802.81955273787)
--(axis cs:59,3863.73742167561)
--(axis cs:60,3853.77929891312)
--(axis cs:61,3818.065801354)
--(axis cs:62,4064.17204214605)
--(axis cs:63,3897.14713770914)
--(axis cs:64,4048.69009945898)
--(axis cs:65,4048.51479500494)
--(axis cs:66,4006.20148236414)
--(axis cs:67,4030.6861412147)
--(axis cs:68,4077.93471095143)
--(axis cs:69,4087.99917151147)
--(axis cs:70,4142.11485463693)
--(axis cs:71,4108.43229266956)
--(axis cs:72,4084.0122957673)
--(axis cs:73,4178.35274962541)
--(axis cs:74,4193.2238640741)
--(axis cs:75,4205.44552076535)
--(axis cs:76,4266.54875454272)
--(axis cs:77,4185.01753521327)
--(axis cs:78,4270.66021685746)
--(axis cs:79,4137.59305656165)
--(axis cs:80,4340.33129171447)
--(axis cs:81,4364.66447597977)
--(axis cs:82,4313.53056549152)
--(axis cs:83,4432.68961233444)
--(axis cs:84,4404.67332181747)
--(axis cs:85,4440.87793932001)
--(axis cs:86,4448.90467369875)
--(axis cs:87,4447.40955252939)
--(axis cs:88,4469.75111831538)
--(axis cs:89,4453.96909444737)
--(axis cs:90,4455.01864116496)
--(axis cs:91,4512.61368064045)
--(axis cs:92,4526.00674630882)
--(axis cs:93,4541.04850730409)
--(axis cs:94,4515.81578620823)
--(axis cs:95,4539.44316335004)
--(axis cs:96,4537.4779131639)
--(axis cs:97,4534.35289441518)
--(axis cs:98,4548.70403239847)
--(axis cs:99,4586.82183378184)
--(axis cs:99,5053.8003811099)
--(axis cs:99,5053.8003811099)
--(axis cs:98,5046.62416034271)
--(axis cs:97,5000.30422419439)
--(axis cs:96,4973.74249108297)
--(axis cs:95,4982.81106880917)
--(axis cs:94,4982.12951757894)
--(axis cs:93,4995.47782143402)
--(axis cs:92,4981.6376778329)
--(axis cs:91,4938.60503805324)
--(axis cs:90,4890.53361127605)
--(axis cs:89,4890.16190705245)
--(axis cs:88,4893.99000921434)
--(axis cs:87,4892.11482180455)
--(axis cs:86,4889.0082730929)
--(axis cs:85,4907.45409468035)
--(axis cs:84,4847.59275079302)
--(axis cs:83,4868.24984889792)
--(axis cs:82,4773.55409269018)
--(axis cs:81,4797.24580230255)
--(axis cs:80,4814.03402404453)
--(axis cs:79,4711.42226188192)
--(axis cs:78,4789.39638683669)
--(axis cs:77,4710.75800561544)
--(axis cs:76,4751.47631339311)
--(axis cs:75,4680.21678951907)
--(axis cs:74,4693.85287767365)
--(axis cs:73,4679.00301993693)
--(axis cs:72,4659.05737149386)
--(axis cs:71,4662.3220550268)
--(axis cs:70,4719.43039105384)
--(axis cs:69,4636.87712907093)
--(axis cs:68,4629.54219099191)
--(axis cs:67,4576.38834498138)
--(axis cs:66,4598.93915688325)
--(axis cs:65,4593.20891021444)
--(axis cs:64,4635.17454688076)
--(axis cs:63,4547.65095493543)
--(axis cs:62,4609.1292358815)
--(axis cs:61,4484.46247780099)
--(axis cs:60,4479.80555358369)
--(axis cs:59,4467.43875617875)
--(axis cs:58,4462.51723333678)
--(axis cs:57,4471.03915159175)
--(axis cs:56,4453.53041120177)
--(axis cs:55,4349.33849853635)
--(axis cs:54,4349.16155486204)
--(axis cs:53,4250.65228232437)
--(axis cs:52,4290.70581701184)
--(axis cs:51,4231.64176007242)
--(axis cs:50,4211.30088817515)
--(axis cs:49,4198.34806494185)
--(axis cs:48,4199.51617470304)
--(axis cs:47,4008.15729863971)
--(axis cs:46,4017.41641852425)
--(axis cs:45,4043.75264404519)
--(axis cs:44,4035.4580851211)
--(axis cs:43,3934.07695019044)
--(axis cs:42,3972.00807179465)
--(axis cs:41,3922.3670372375)
--(axis cs:40,3851.10007922692)
--(axis cs:39,3850.74856226755)
--(axis cs:38,3774.56820227011)
--(axis cs:37,3808.79015320458)
--(axis cs:36,3835.18770025905)
--(axis cs:35,3802.09967613191)
--(axis cs:34,3659.57631488221)
--(axis cs:33,3668.86289362134)
--(axis cs:32,3574.75530073646)
--(axis cs:31,3507.90205614356)
--(axis cs:30,3568.61238344969)
--(axis cs:29,3459.61115932133)
--(axis cs:28,3526.77867842184)
--(axis cs:27,3447.61030843759)
--(axis cs:26,3279.10359292781)
--(axis cs:25,3153.2598857932)
--(axis cs:24,3206.44654725348)
--(axis cs:23,3260.56890823777)
--(axis cs:22,3276.95530403881)
--(axis cs:21,3176.1516709148)
--(axis cs:20,3062.72328844904)
--(axis cs:19,2962.14318015461)
--(axis cs:18,2893.32821010861)
--(axis cs:17,2878.6996473074)
--(axis cs:16,2660.86342440608)
--(axis cs:15,2450.42211159532)
--(axis cs:14,2325.85378281345)
--(axis cs:13,2180.11434067542)
--(axis cs:12,1982.7602212415)
--(axis cs:11,1705.52102370049)
--(axis cs:10,1500.79809321423)
--(axis cs:9,1373.70494375825)
--(axis cs:8,1268.44419596225)
--(axis cs:7,1097.34728051216)
--(axis cs:6,964.626732242247)
--(axis cs:5,741.045193230484)
--(axis cs:4,580.856263194181)
--(axis cs:3,331.477325930986)
--(axis cs:2,173.588479701568)
--(axis cs:1,18.7349784680543)
--(axis cs:0,48.7095079349399)
--cycle;
\addplot [ultra thick, red]
table {%
0 21.1362775419912
1 3.17336224103879
2 83.268205330398
3 194.757409276003
4 393.964067050292
5 550.282168103466
6 733.028764717683
7 897.872453584461
8 1101.05540312677
9 1220.27793821503
10 1384.53259392711
11 1574.13199388744
12 1771.60741646491
13 1943.12415611738
14 2045.43543385983
15 2105.16991946714
16 2264.7002899841
17 2427.03643027418
18 2467.33101373369
19 2515.77445594357
20 2571.61546364005
21 2650.96425939603
22 2741.08465127957
23 2624.00188823102
24 2650.58209852973
25 2677.65042523032
26 2814.76181584111
27 2953.23305608932
28 2969.05829477055
29 2964.70779620861
30 3140.91047385336
31 2959.88163027965
32 3058.3701785386
33 3186.37336282369
34 3166.34272497959
35 3312.06228624368
36 3370.66030173289
37 3353.29880510616
38 3337.78783990556
39 3413.98556735254
40 3450.86294798424
41 3496.50148885725
42 3531.50509361021
43 3507.92294897865
44 3584.95488159384
45 3609.35191721403
46 3674.93872695687
47 3574.23672626692
48 3810.93328153851
49 3806.48736549654
50 3819.76885921119
51 3884.2879829191
52 3891.24369594138
53 3932.0716092578
54 3979.91196675239
55 3971.84784793044
56 4079.98433635019
57 4192.59353518088
58 4112.75168802223
59 4172.54006590849
60 4163.68876136098
61 4126.56350056355
62 4324.72303814195
63 4214.13609503513
64 4338.48717219042
65 4321.07870308427
66 4275.2727626628
67 4297.55820060662
68 4347.70248356101
69 4350.80628552002
70 4406.07597467195
71 4377.78273669928
72 4369.92808798936
73 4420.27709963625
74 4442.59056335099
75 4435.47924987347
76 4503.80268275828
77 4440.82514999354
78 4519.67154656011
79 4407.71040007063
80 4566.27313009734
81 4559.45078790367
82 4527.64872882384
83 4639.48483070322
84 4625.59170791827
85 4672.0339357011
86 4658.62775037913
87 4662.2271111355
88 4674.38221986221
89 4670.09550933003
90 4673.43620904179
91 4722.02452343923
92 4736.29910531417
93 4737.31788282022
94 4748.93478492324
95 4751.01246706157
96 4761.09737319514
97 4763.2839379763
98 4791.9619841706
99 4810.51000076615
};
\addlegendentry{Our Method}

\addplot [very thick, color0]
table {%
0 37.3094361967087
1 38.027655689851
2 225.413344247887
3 330.949289150662
4 552.469713705319
5 790.301230837579
6 1098.6936704127
7 1228.60340424714
8 1396.94881513452
9 1758.11399342238
10 1746.69549329209
11 1778.88999305995
12 1832.85407471355
13 1967.52106512268
14 1878.38904909163
15 2002.45149160443
16 2151.49957043548
17 2248.18024856077
18 2334.86237357702
19 2389.59334235251
20 2362.54200894263
21 2456.27126413206
22 2534.88686804889
23 2596.22775966941
24 2609.03838278441
25 2635.58936785133
26 2714.39440346595
27 2739.97127246859
28 2749.36715896382
29 2787.71631704187
30 2851.62278203101
31 2957.94880574962
32 2882.0314630918
33 2914.31010675944
34 3065.41164566259
35 3097.41420047104
36 3151.41042653258
37 3180.90599142946
38 3248.24926876717
39 3307.79511592169
40 3261.80196560642
41 3333.95429166662
42 3331.68219296042
43 3404.06172131972
44 3385.22949323795
45 3458.24309792008
46 3486.40233239145
47 3509.43795709587
48 3470.03657018247
49 3528.45271341551
50 3540.12126406845
51 3567.18856200085
52 3549.19954578798
53 3647.78705512063
54 3604.08735875078
55 3603.89452611701
56 3695.04474584646
57 3634.81887484795
58 3700.51001412433
59 3675.401793967
60 3754.74568283386
61 3671.10344137625
62 3786.52774765705
63 3798.55195016088
64 3826.52615978341
65 3844.4455536602
66 3883.86035268755
67 3791.58388573817
68 3822.41109948936
69 3740.50713313524
70 3849.3641066529
71 3853.19068451966
72 3830.75954781293
73 3972.60840710495
74 3875.74059614119
75 3911.97019478188
76 4003.05256972446
77 4015.19786632733
78 3959.10013128615
79 4014.38813948407
80 3998.21234819561
81 4072.07558475494
82 4008.45895785542
83 4032.3031351032
84 4060.61383578006
85 4043.62358871189
86 4081.01618598949
87 4077.82674494101
88 4094.42063632382
89 4096.93935339484
90 4050.93897993865
91 4081.74230314738
92 4091.94686591034
93 4097.43754028792
94 4139.69908340458
95 4122.06822543156
96 4075.84304381213
97 4113.75960462698
98 4201.03886630064
99 4104.62314516946
};
\addlegendentry{Transfer $\hat{P}$ only}
\addplot [very thick, color1]
table {%
0 9.8418866310933
1 41.4368104614329
2 239.963946549034
3 653.466651396104
4 783.135201325639
5 909.509845328097
6 968.433466059677
7 1059.60418909313
8 1212.56366027514
9 1315.8055781537
10 1332.88291225554
11 1412.81560045277
12 1453.41364732155
13 1507.91579449019
14 1542.47411257717
15 1576.77944449373
16 1616.66628723079
17 1643.38878739787
18 1665.93007715678
19 1703.67321833492
20 1745.38672584164
21 1785.17846505177
22 1790.87766745657
23 1841.98766487427
24 1856.34239485405
25 1870.88042430354
26 1894.17416666584
27 1927.57505111566
28 1954.21694016843
29 1972.18186903036
30 1969.15359286742
31 2030.20598799758
32 2031.21397789708
33 2044.38615662373
34 2074.76304913471
35 2088.87677826334
36 2120.44338636051
37 2130.98896312018
38 2157.64499306432
39 2173.87081031594
40 2134.42252058999
41 2199.42485116163
42 2214.34859160191
43 2246.17653814415
44 2247.34303986301
45 2220.91524376628
46 2291.44416955129
47 2355.46662645005
48 2335.8222755594
49 2389.71122404354
50 2401.89341898041
51 2442.57802634682
52 2494.67743576228
53 2485.04514180669
54 2600.03574168654
55 2627.01939524707
56 2645.82068150205
57 2635.25257852853
58 2656.07554653328
59 2690.93455432894
60 2674.62661752591
61 2713.74321885722
62 2719.89076820396
63 2761.53996142784
64 2798.27449490454
65 2742.98077666177
66 2814.16434114423
67 2844.19023047353
68 2869.8730490803
69 2847.13506412668
70 2815.23969788522
71 2922.48146730584
72 2923.85171483153
73 2929.74835658998
74 2968.22446713598
75 2936.40171500703
76 2952.4342572882
77 3010.76094959267
78 3089.24083592455
79 3112.67066443793
80 3102.45653347678
81 3095.7474814425
82 3147.14721578525
83 3141.57851477508
84 3138.9976041204
85 3176.63110725011
86 3223.80405305329
87 3260.28030289124
88 3270.41960995327
89 3204.38329824873
90 3265.47733255923
91 3138.56424698424
92 3293.61899890344
93 3248.36066104595
94 3274.14712602086
95 3316.25214661311
96 3329.95699247456
97 3346.63458299668
98 3379.75711400428
99 3343.69961288018
};
\addlegendentry{Transfer $\hat{R}$ only}
\addplot [very thick, blue]
table {%
0 -11.084663737933
1 99.8035792984406
2 369.900565979479
3 584.491087768476
4 802.898785917725
5 933.856890457755
6 904.084758783503
7 931.573936866358
8 996.272641607167
9 1082.42602324171
10 1111.27649033751
11 1176.66357321053
12 1145.95858919874
13 1181.6331366736
14 1218.25336727096
15 1304.48026881758
16 1364.40074847913
17 1417.35190403534
18 1450.59813822457
19 1511.67369013555
20 1490.86598336642
21 1601.70149076006
22 1645.50966594009
23 1760.60678544715
24 1731.20170199876
25 1777.2110833715
26 1852.4236165584
27 1940.59609356621
28 1908.87075198114
29 1956.21941726234
30 1966.08798451785
31 2014.66053648493
32 2077.86305765496
33 2110.61393629852
34 2139.18284279186
35 2135.63615161644
36 2196.00203690086
37 2174.61057232744
38 2238.91776854732
39 2249.92945282929
40 2253.31614702679
41 2254.1585224219
42 2275.8540739049
43 2274.0966216333
44 2329.24586028362
45 2233.59602682774
46 2349.25067696172
47 2337.02365386167
48 2353.55958223237
49 2372.39457456569
50 2353.60573199366
51 2397.41680530923
52 2400.91861893596
53 2477.31980285485
54 2445.4002537686
55 2468.92304876792
56 2444.07235533217
57 2509.26514813107
58 2475.08070046857
59 2515.71377399278
60 2521.05501264526
61 2530.97282563386
62 2558.09002081235
63 2557.31663447986
64 2486.5556732633
65 2606.46933549486
66 2542.37099610749
67 2519.07712079405
68 2655.96859416512
69 2632.6132881532
70 2704.0633823554
71 2652.98687993913
72 2663.99173415709
73 2698.21201071726
74 2703.47250301136
75 2707.20847982912
76 2740.02159061625
77 2802.89117356973
78 2819.73546847909
79 2747.61299481459
80 2818.58383828957
81 2790.76639860887
82 2864.91797586068
83 2859.4494813188
84 2887.23845454987
85 2909.31680606687
86 2917.45273638468
87 2904.71604120769
88 3002.83946380317
89 2964.19000782691
90 2984.15340418278
91 3003.54484807172
92 2994.5579650952
93 3002.80014220374
94 3038.36645525333
95 3023.88081388067
96 3067.36878338539
97 3077.44423856982
98 3068.37905131114
99 2942.79358012208
};
\addlegendentry{Single}
\end{axis}

\end{tikzpicture}

%% file: s7-conclusion.tex
\vspace{-0.5em}
\section{Conclusion}
\label{sec:conclusion}
\vspace{-0.5em}
In this paper, we identify and propose a solution to an important but rarely studied problem: transferring knowledge between tasks with drastically different observation spaces where inter-task mappings are not available. 
We propose to learn a latent dynamics model in the source task and transfer the model to the target task to facilitate representation learning. Theoretical analysis and empirical study justify the effectiveness of the proposed algorithm.

%% file: s8-ethics.tex
\section*{Ethics Statement}

Transfer learning aims to apply previously learned experience to new tasks to improve learning efficiency, which is becoming more and more important nowadays for training intelligent agents in complex systems.
This paper focuses on a practical but rarely studied transfer learning scenario, where the observation feature space of an RL environment is subject to drastic changes. Driven by theoretical analysis on representation learning and its relation to latent dynamics learning, we propose a novel algorithm that transfers knowledge between tasks with totally different observation spaces, without any prior knowledge of an inter-task mapping.

This work can benefit many applications as suggested by the examples below. \\
(1) In many real-life environments where deep RL is used (e.g. navigating in a building), the underlying dynamics (e.g. the structure of the building) are usually fixed, but what features the observation space has is designed by human developers (e.g. what sensors are installed) and thus may change frequently during the development. When the agent gets equipped with better sensors, our algorithm makes it possible to reuse previously learned models when learning with the new sensors. \\
(2) An agent usually learns better with a compact observation space (e.g. a low-dimensional vector space containing its location and the goal's location) than a rich/noisy observation space (e.g. an image containing the goal). However, a compact observation is usually more difficult to construct in practice as it may require expert knowledge and human work. In this case, one can extract compact observations in a few samples and pre-train a policy with our Algorithm~\ref{alg:source}, then train the agent in the real environment with rich observations with our Algorithm~\ref{alg:target} using the learned dynamics models. Our experiment in Figure~\ref{fig:all} shows that the learning efficiency in the rich-observation environment can be significantly improved with our proposed transfer method.


\section*{Reproducibility Statement}

For theoretical results, we provide concrete proofs in Appendix~\ref{app:proofs} and Appendix~\ref{app:avi}. For empirical results, we illustrate implementation details in Appendix~\ref{app:exp}. The source code and running instructions are provided in the supplementary materials.

%% file: a0-prelim.tex
\section{Additional Preliminary Knowledge}
\label{app:prelim}

For any policy $\pi$, its Q value $Q^\pi$ is the unique fixed point of the Bellman operator

\begin{equation}
\label{eq:bellman_pi}
    (\bellman^\pi Q) (o,a) = \mathbb{E}_{o^\prime \sim P(o,a), a^\prime \sim \pi(o^\prime)}[R(o,a) + \gamma  Q(o^\prime, a^\prime)] 
\end{equation}  

The optimal policy can be found by \textit{policy iteration}~\citep{howard1960dynamic}, where one starts from an initial policy $\pi_0$ and repeats policy evaluation and policy improvement. More specifically, at iteration $k$, the algorithm evaluates $Q_{\pi_k}$ via Equation~(\ref{eq:bellman_pi}), then improves the policy by $\pi_{k+1}(o) := \mathrm{argmax}_{a\in\actions} Q_{\pi_k}(o,a), \forall o \in\states.$
It is well-known that the policy iteration algorithm converges to the optimal policy under mild conditions~\citep{puterman2014markov}.
When the dynamics $P$ and $R$ are unknown, reinforcement learning algorithms use interaction samples from the environment to approximately solve $\hat{Q}_{\pi_k}$.
Prior works~\citep{BertsekasTsitsiklis96,munos2005error} have shown that if the approximation error is bounded by a small constant, the performance of $\pi_k$ as $k\to\infty$ is guaranteed to be close to the optimal policy value $Q_{\pi^*}$, which we also denote as $Q^*$.

%% file: a3-representation.tex
\section{Additional Discussion for Representation Sufficiency}
\label{app:repre}

\paragraph{Good Representation for A Fixed Policy} We slightly abuse notation and use $\apx$ to denote the approximation operator for state value function $V:\states\to\mathbb{R}$, similar to the approximation of $Q$ as introduced in Section~\ref{sec:representation}. The following definition characterizes the sufficiency of a representation mapping in terms of evaluating a fixed policy. 
\begin{definition}[Sufficient Representation for A Fixed Policy]
\label{def:suf_policy}
A representation mapping $\phi$ is \textbf{sufficient} for a policy $\pi$ w.r.t. approximation operator $\apx$ iff $\apx V^\pi = V^\pi$. More generally, for a constant $\epsilon \geq 0$, $\rep$ is $\boldsymbol{\epsilon}$\textbf{-sufficient} for $\pi$ iff $\|\apx V^\pi - V^\pi\| \leq \epsilon$.
\end{definition}

\textbf{Remarks.} 
(1) If $o_1, o_2 \in\states$ have different values under $\pi$, a good representation should be able to distinguish them, i.e., $\phi^*(o_1)\neq \phi^*(o_2)$. \\
(2) The approximation operator $\apx$ is linear if $\apx V = \mathrm{Proj}_{\rep} (V)$ where $\mathrm{Proj}_\rep(\cdot)$ denotes the orthogonal projection to the subspace spanned by the basis functions of $\langle \phi_1, \phi_2, \cdots, \phi_d \rangle$.

\paragraph{Model Sufficiency over Policies Induces Linear Learning Sufficiency}
It can be found from Definition~\ref{def:suf_learn} that $\phi$ is sufficient as long as it represents $\qpi$ for all $\pi\in\reppolicy$. Fitting various value functions to improve representation quality is proposed by some prior works~\cite{bellemare2019geometric,dabney2020the} and shown to be effective. However, learning and fitting many policy values could be computationally expensive, and is not easy to be applied to transfer learning between tasks with different observation spaces. Can we regularize the representation with the latent dynamics instead of policy values?
Proposition~\ref{prop:lin_suf} below shows that if $\phi$ is linearly sufficient for all dynamics pairs $(\ppi, \rpi)$ induced by policies in $\reppolicy$ and the dynamics pairs associated with all actions, then $\phi$ is linearly sufficient for learning. 

For notation simplicity, assume the state space is finite.
Then, let $(\ppi, \rpi)$ be the transition matrix and reward vector induced by policy $\pi$, i.e., $\ppi[i,j] = \mathbb{E}_{a \sim \pi(o_i)} [P(o_j|o_i,a)]$, $\rpi[i] = \mathbb{E}_{a \sim \pi(o_i)} [R(o_i,a)]$. 
Similarly, let $\pact[i,j] = P(o_j|o_i,a)$, $\ract[i] = R(o_i,a)$.
We let $\rep$ denote the representation matrix, where the $i$-th row of $\rep$ refers to the feature of the $i$-the observation.

\begin{proposition}[Linear Sufficiency Induced by Policy-based Model Sufficiency]
\label{prop:lin_suf}
For representation $\rep$, if for all $\pi\in\reppolicy, a\in\actions$,  there exist $\appi,\arpi,\apa,\ara$ such that $\rep\appi=\ppi\rep$, $\rep\arpi=\rpi$, $\rep\apa=\pact\rep, \rep\ara=\ract$, i.e. $\rep$ is linearly sufficient both policy-based dynamics and policy-independent dynamics models, i.e., 
then $\rep$ is linearly sufficient for a task $\mdp$ w.r.t. approximation operator $\apx$.
\end{proposition}

Proposition~\ref{prop:lin_suf} proven in Appendix~\ref{proof:lin_suf} suggests that we can let representation fit $(\ppi,\rpi)$ for many different $\pi$'s. However, it could be computationally intractable since the policy space is large. More importantly, it is not memory-friendly to store and transfer a large number of dynamics models for all $(\appi,\arpi)$. 
In our Proposition~\ref{prop:nonlin_suf}, we show that learning sufficiency can be induced by policy-independent model sufficiency, which is much simpler as there is no need to learn and store $(\appi,\arpi)$ for many policies. As a trade-off, the policy-independent model induces non-linear sufficiency instead of linear sufficiency, requiring a more expressive approximation operator.


\textbf{Latent MDP induced by Representation}\quad
We follow the analysis by~\citet{gelada2019deepmdp} and define a new MDP under the representation mapping $\phi$: $\tilde{\mdp} = \langle \tilde{\states}, \actions, \tilde{P}, \tilde{R}, \gamma \rangle$, where for all $o\in\states$, $\phi(o)\in\tilde{\states}$,
$\tilde{P}(\phi(o),a) = \apa\phi(o) $, $\tilde{R}(\phi(o),a) = \ara \phi(o)$. 
Let $\tilde{V}$ denote the value function in $\tilde{\mdp}$, and let $\tilde{\pi}$ denote a policy in $\tilde{\mdp}$.
We make the following mild assumption
\begin{assumption}
\label{assump:lips}
There exists a constant $\lipsvalue$ such that $$ |\tilde{V}_{\tilde{\pi}}(\phi(o_1)) - \tilde{V}_{\tilde{\pi}}(\phi(o_2)) | \leq \lipsvalue \|\phi(o_1)-\phi(o_2)\|, \forall \tilde{\pi}:\tilde{O}\to\actions, o_1,o_2\in\states.$$
\end{assumption}
This assumption is mild as any MDP with bounded reward has bounded value functions.

%% file: a1-proofs.tex

\section{Technical Proofs}
\label{app:proofs}

\subsection{Proof of Lemma~\ref{lem:bound_pi}}
\label{proof:bound_pi}

\begin{proof}[Proof of Lemma~\ref{lem:bound_pi}]

We first show that policy iteration with approximation operator $\apx$ starting from a policy $\pi_0\in\reppolicy$ generates a sequence of policies that are in $\reppolicy$. That is, for all $\pi_k$, and any $o_1,o_2\in\states$, $\pi_k(o_1)=\pi_k(o_2)$ if $\phi(o_1)=\phi(o_2)$.
We prove this claim by induction.

\textit{Base case:} when $k=0$, $\pi_0\in\reppolicy$.

\textit{Inductive step:} assume $\pi_k \in \reppolicy$ for $k\geq 0$, then for iteration $k+1$, we know that
\begin{equation}
    \pi_{k+1} (o) := \mathrm{argmax}_a Q_k(o,a)
\end{equation}
where $Q_k = \apx Q_{\pi_k}$. 

Based on the definition of $\apx$, $Q_k(o,a)=f(\phi(o);\theta_a)$ for some $\theta_a$. Hence, if $o_1$ and $o_2$ have the same representation, $Q_k(o_1,\cdot)$ and $Q_k(o_2,\cdot)$ are equal. As a result, $\pi_{k+1}(o_1)$ and $\pi_{k+1}(o_2)$ are equal, so $\pi_{k+1}\in\reppolicy$.

Next we prove the error bound in Lemma~\ref{lem:bound_pi}.
We start by restating the error bounds for approximate policy iteration by~\citet{BertsekasTsitsiklis96}: 
\begin{equation}
    \mathrm{limsup}_{k\to\infty} \|V^*-V^{\pi_k} \|_\infty \leq 
    \frac{2\gamma}{(1-\gamma)^2} \mathrm{sup}_k \| V_k - V^{\pi_k} \|_\infty
\end{equation}
where $\pi_k$ is the policy in the $k$-th iteration.
Then we extend the above result to the action value $Q$.

For any $\pi_k$ during the policy iteration (as proven above, $\pi_k\in\reppolicy$), if $\phi$ is $\epsilon$-sufficient for $\mdp$ as defined in Definition~\ref{def:suf_learn} with $\ell_\infty$ norm, then we have $\| Q_k - Q_{\pi_k} \|_\infty \leq \epsilon$. That is, $\forall o\in\states, a\in\actions, | Q_k(o,a) - Q_{\pi_k}(o,a) | \leq \epsilon $. Therefore, $\forall o\in\states$,
\begin{equation}
    |V_k(o) - V_{\pi_k}(o)| = | \sum_{a\in\actions} \pi(a|o) (Q_k(o,a) - Q_{\pi_k}(o,a)) | \leq \epsilon \label{eq:vk_bound}
\end{equation}

On the other hand, we can derive
\begin{align}
    \| Q^* - Q_{\pi_k} \|_\infty &= \max_{o,a} |Q^*(o,a) - Q_{\pi_k}(o,a) | \\
    &= \max_{o,a} |R(o,a) + \gamma \sum_{o^\prime\in\states}P(o^\prime|o,a)V^*(o^\prime) - R(o,a) - \gamma \sum_{o^\prime\in\states}P(o^\prime|o,a)V_{\pi_k}(o^\prime) | \\
    &= \gamma \|V^* - V_{\pi_k} \|_\infty \label{eq:opt_bound}
\end{align}

Combining Equation~(\ref{eq:vk_bound}) and~(\ref{eq:opt_bound}), we obtain
\begin{equation}
    \mathrm{limsup}_{k\to\infty} \|Q^* - Q_{\pi_k} \|_\infty \leq \frac{2\gamma^2}{(1-\gamma)^2} \epsilon.
\end{equation}

\end{proof}

\subsection{Proof of Proposition~\ref{prop:nonlin_suf}}
\label{proof:nonlin_suf}
\begin{proof}[Proof of Proposition~\ref{prop:nonlin_suf}]

Given that $P$ is deterministic for $o\in\states, a\in\actions$, we slightly abuse notation and let $o^\prime = P(o,a) = \pact(o)$ if $P(o^\prime|o,a)=1$.
If $\phi$ is sufficient for the dynamics models, i.e. $\forall o\in\states, a\in\actions$, $\apa \phi(o) = \phi(\pact(o))$, $\ara \phi(o) = \ract(o)$.
Then, we can define a new MDP $\tilde{\mdp} = \langle \tilde{\states}, \actions, \tilde{P}, \tilde{R}, \gamma \rangle$, where for all $o\in\states$, $\phi(o)\in\tilde{\states}$, and
$\tilde{P}(\phi(o),a) = \apa\phi(o) = \phi(P(o,a))$, $\tilde{R}(\phi(o),a) = \ara \phi(o) = R(o,a)$. 

Any policy $\pi\in\reppolicy$, based on the definition of $\reppolicy$, can be written as $\tilde{\pi}\circ\phi$, where $\tilde{\pi}$ is a deterministic policy in $\tilde{\mdp}$.
Next, we show that for all $o\in\states, a\in\actions$, $\qpi(o)=\tilde{Q}_{\tilde{\pi}}(\phi(o))$.

By definition of the Q value, we know 
\begin{align}
    \qpi(o,a) &= \mathbb{E}_{\pi,P}[\sum_{t=0}^\infty \gamma^t R(o_t, a_t)|o_0=o, a_0=a] \label{eq:q_original}\\
    \tilde{Q}_{\tilde{\pi}}(\phi(o)) 
    &= \mathbb{E}_{\tilde{\pi},\tilde{P}}[\sum_{t=0}^\infty \gamma^t \tilde{R}(\tilde{o}_t, \tilde{a}_t)| \tilde{o}_0=\phi(o), \tilde{a}_0=a] \label{eq:q_induced}
\end{align}

We claim that in the above equations, $\tilde{o_t} = \phi(o_t)$, $\tilde{a_t}=a_t$, for all $t\geq 0$. We prove the claim by induction.

When $t=0$, the claim holds as $\tilde{o_0}=\phi(o_0)=\phi(o)$, $\tilde{a_0}=a_0=a$.

Then, with inductive hypothesis that $\tilde{o_t} = \phi(o_t)$, $\tilde{a_t}=a_t$, we show the claim holds for $t+1$:

Action:
$a_{t+1} = \pi(o_{t+1}) = \tilde{\pi}(\phi(o_{t+1})) = \tilde{a}_{t+1}$.

State: $\tilde{o}_{t+1} = \tilde{P}(\tilde{o_t},a_t) = \tilde{P}(\phi(o_t),a_t) = \phi(P(o_t,a_t)) = \phi(o_{t+1})$. 

Hence, we have shown $\tilde{o}_t = \phi(o_t)$, $\tilde{a}_t=a_t$, for all $t\geq 0$, then the reward in the $t$-th step of Equation~(\ref{eq:q_original}) and (\ref{eq:q_induced}) are the same, as $\tilde{R}(\tilde{o_t},\tilde{a_t}) = \tilde{R}(\phi(o_t),a_t) = R(o_t,a_t)$. Therefore, $\qpi(o)=\tilde{Q}_{\tilde{\pi}}(\phi(o))$.

Therefore, for any $\pi\in\reppolicy$, its action value can be represented by $\tilde{Q}_{\tilde{\pi}}\circ\phi$. Therefore, $\phi$ is sufficient for learning in $\mdp$.

Next, we show that \textbf{$\phi$ is not necessarily linearly sufficient} for learning the task. 

Consider an arbitrary policy $\pi\in\reppolicy$. Without loss of generality, suppose $\rpi$ is linearly represented by $\phi$, i.e. $\ract(o,a) = \phi(o)^\top \ara$, then we have
\begin{align*}
    \rpi(o) &= \sum_{a\in\actions} \pi(a|\phi(o)) \ract(o,a) \\
    &= \sum_{a\in\actions} \pi(a|\phi(o)) \phi(o)^\top \ara \\
    &= \langle \pi(\phi(o)), \hat{R}^\top \phi(o) \rangle \\
    &= \langle \hat{R} \pi(\phi(o)), \phi(o) \rangle
\end{align*}

where $\hat{R}:=[\hat{R}_{a_1};\hat{R}_{a_1};\cdots;\hat{R}_{a_{|\actions|}} ]$. We can find that unless $\pi$ always takes the same action for all input states, $\phi$ is not guaranteed to linearly encode $\rpi$.


Similarly, for $\ppi$, suppose $\pact(\cdot|o,a)=\phi(o)^\top \apa$ we have
\begin{align*}
    \phi(\ppi(o)) &= \sum_{a\in\actions} \pi(a|\phi(o)) \pact(\cdot|o,a) \\
    &= \sum_{a\in\actions} \pi(a|\phi(o)) \phi(o)^\top \apa \\
    &= \hat{P}(\pi(\phi(o)), \phi(o), I) 
\end{align*}
where $\hat{P}:=[\hat{P}_{a_1};\hat{P}_{a_1};\cdots;\hat{P}_{a_{|\actions|}} ]$ is an $|\actions|\times d \times d$ tensor, and $\hat{P}(\cdot,\cdot,\cdot)$ denotes the multi-linear operation. Hence, if $\pi$ takes different actions in different states, $\phi$ may not linearly encode the transition, either.

Therefore, $\phi$ is not guaranteed to linearly encode $\rpi$ and $\ppi$, and thus is not guaranteed to linearly encode $\vpi$ and $\qpi$. 

\end{proof}

\subsection{Proof of Proposition~\ref{prop:lin_suf}}
\label{proof:lin_suf}
\begin{proof}[Proof of Proposition~\ref{prop:lin_suf}]

We first show that for any $\pi\in\Pi$, if $\phi$ is linearly sufficient for $(\ppi,\rpi)$, then there exists a vector $\omega \in \mathbb{R}^k$ such that $\vpi = \avpi = \rep\omega$.

Since $\rep$ is linearly sufficient for $\ppi$ and $\rpi$, we have $\rep \appi = \ppi \rep$ and $\rep \arpi = \rpi$ for some $\appi$ and $\arpi$. Let $\omega = (I-\gamma \appi)^{\dagger} \arpi$, then the Bellman error of $\avpi=\rep\omega$ can be computed as

\begin{align*}
    \rpi + \gamma\ppi\avpi - \avpi &= \rpi + \gamma\ppi\rep\omega - \rep\omega \\
    &= \rep \arpi + \gamma \rep \appi \omega - \rep\omega \\
    &= \rep (\arpi - (I-\gamma\appi) (I-\gamma \appi)^{\dagger} \arpi ) \\
    &= \rep (\arpi - \arpi) \\
    &= 0
\end{align*}

Therefore, $\avpi$ is a fixed point of the Bellman operator $\bellman^\pi$, and thus equal to $\vpi$.


Next, as we know that $\qpi(\cdot,a) = \ract + \gamma \langle \pact, \vpi \rangle$, and $\rep\apa=\pact\rep, \rep\ara=\ract$, we can obtain
\begin{align*}
    \qpi(\cdot,a) &= \ract + \gamma \langle \pact, \vpi \rangle \\
    &= \rep\ara + \gamma \pact \rep \omega \\
    &= \rep\ara + \gamma \rep\apa \omega \\
    &= \rep (\ara + \gamma \apa \omega)
\end{align*}

Therefore, for any $\pi\in\reppolicy$, $\qpi$ can be linearly represented by $\rep$, and thus $\rep$ is linearly sufficient for learning by definition.

\end{proof}

\subsection{Proof of Theorem~\ref{thm:model_learn_error}}
\label{proof:model_learn_error}





\begin{proof}[Proof of Theorem~\ref{thm:model_learn_error}]

Lemma 2 in~\citet{gelada2019deepmdp} is based on one policy in the induced MDP and bounded model errors. We can replace the Wasserstein distance $\mathcal{W}(\phi P(\cdot|o,a), \tilde{P}(\cdot|\phi(o),a))$ by the Euclidean distance $\| \phi P(o,a), \tilde{P}(\phi(o),a) \|$ as we focus on deterministic transitions.

For any policy $\pi\in\reppolicy$ that can be written as $\tilde{\pi}\circ\phi$, we have
\begin{equation}
    | Q_{\pi}(o,a) - \tilde{Q}_{\tilde{\pi}}(\phi(o),a) | \leq \frac{\epsilon_R + \gamma \lipsvalue \epsilon_P}{1-\gamma}
\end{equation}

Therefore, $\phi$ is $(1-\gamma)^{-1}(\epsilon_R+\gamma\lipsvalue\epsilon_P)$-sufficient for learning $\mdp$. 

Combined with Lemma~\ref{lem:bound_pi}, we can obtain the bound in Theorem~\ref{thm:model_learn_error}.

\end{proof}

\subsection{Proof of Theorem~\ref{thm:transfer_suf}}
\label{proof:transfer_suf}

\begin{proof}[Proof of Theorem~\ref{thm:transfer_suf}]
First of all, if there exists $\phi\target$ satisfying
$\apa(\phi(o_i))=\pact\target[i]\rep\target$, $\ara(\phi(o_i)) = \ract\target[i]$, $\forall o_i\in\states\target$, then it is sufficient for the dynamics of the target task, and thus sufficient for learning the target task as stated in Proposition~\ref{prop:nonlin_suf}. Therefore, our focus is to show the existence of such a representation.

As $\phi\source$ is sufficient for $\pact\source$ and $\ract\source$ for all $a\in\actions$, we have 
\begin{align}
    &\apa(\phi\source(o\source)) = P\source(o\source,a)\rep\source = \phi\source(P\source(o\source,a))\\
    &\ara(\phi\source(o\source)) = R\source(o\source,a)
\end{align}
where we let $P\source(o\source,a)$ denote the next state of $(o\source,a)$, given that $P$ is deterministic.

Based on Assumption~\ref{assum:similarity}, we know that there exists a function $f$ such that $\forall o\target\in\states\target$, $f(o\target)\in\states\source$, and $f(P\target(o\target,a))=P\source(f(o\target),a)$, $R\target(o\target,a)=R\source(f(o\target),a)$.
Hence, we can obtain 
\begin{align}
    &\apa(\phi\source(f(o\target))) = \phi\source(P\source(f(o\target),a)) = \phi\source(f(P\target(o\target,a)))\\
    &\ara(\phi\source(f(o\target))) = R\source(f(o\target),a) = R\target(o\target,a)
\end{align}
Let $\hat{\phi}\target := \phi\source\circ f$, then we get
\begin{align}
    &\apa(\hat{\phi}\target(o\target)) = \hat{\phi}\target (P\target(o\target,a))  \\
    &\ara(\hat{\phi}\target(o\target)) = R\target(o\target,a)
\end{align}

Therefore, $\hat{\phi}\target$ is a feasible solution satisfying model sufficiency in the target task, and thus is sufficient for learning.

Theorem~\ref{thm:transfer_suf} holds since we have shown (1) all feasible solutions to $\rep\target \apa = \pact\target \rep$ and $\rep\target \ara = \ract\target$ are sufficient for learning in $\mdp\target$, and (2) there exists at least one feasible solution to $\rep\target \apa = \pact\target \rep$ and $\rep\target \ara = \ract\target$.

\end{proof}

%% file: a4-avi.tex
\section{Representation Learning for Approximate Value Iteration}
\label{app:avi}

Now we illustrate how our transfer algorithm and the proposed model-based regularization work for approximate value iteration.
We focus on the case where the reward function $R(o,a)\geq 0$ for all $o\in\states$ and $a\in\actions$.

\textbf{Preliminaries}\quad
The bases of value iteration is the Bellman optimality operator $\bellman^*$. For the value function, we have
\begin{equation}
    \bellman^* V(o) = \max_{a\in\actions} [ R(o,a) + \gamma \sum_{o\in\states} P(o^\prime|o,a) V(o^\prime) ]
\end{equation}
For the Q function, we have
\begin{equation}
    \bellman^* Q(o,a) =  R(o,a) + \gamma \sum_{o\in\states} P(o^\prime|o,a) \max_{a^\prime\in\actions} Q(o^\prime,a^\prime),
\end{equation}
where we slightly abuse notation and use $\bellman^*$ for both the value function and the Q function when there is no ambiguity. 

Starting from some initial $Q_0$ (or $V_0$) and iteratively applying $\bellman^*$, i.e., $Q_{k+1}=\bellman^* Q_k$, $Q_k$ (or $V_k$) can finally converge to $Q^*$ or $V^*$ when $k\to\infty$, which is known as value iteration.
When an approximation operator $\mathcal{H}$ is used, the process is called approximate value iteration (AVI): $Q_{k+1}=\apx \bellman^* Q_k$. A prior work has shown the following asymptotic result:
\begin{lemma}[Approximate Value Iteration for $Q$]
\label{lem:avi}
For a reinforcement learning algorithm based on value iteration with approximation operator $\mathcal{H}$, if $\|\mathcal{H}\bellman^* Q_k - \bellman^* Q_k\|_\infty\leq \epsilon$, for all $Q_k$ along the value iteration path, then we have
\begin{equation}
    \limsup_{k\to\infty} \|V^*-V_{\pi_k}\| \leq \frac{2\epsilon}{(1-\gamma)^2}.
\end{equation}
\end{lemma}

\textbf{Value Iteration Learning with Given Representation}\quad
Given a representation mapping $\phi$, we aim to learn an approximation function $h:\phi(\states)\to\mathbb{R}^{|A|}$ such that $\hat{Q}(o,\cdot)=\apf(\phi(o))\approx Q(o,\cdot)$. For notation simplicity, we further use $\apf(\phi(o),a)$ to denote the approximated Q value $\hat{Q}(o,a)$ for $a\in\actions$.
We start from an initial $\apf_0$ that has a uniform value $c$ for all inputs, where $c>0$ can be randomly selected. The initial approximation for Q value is $\hat{Q}_0=\apf_0\circ\phi$.
Then, at iteration $k>0$, we solve $\hat{Q}_k=\apx\bellman^*\hat{Q}_{k-1}=\apf_{k}\circ\phi$, where $\apf_k:=\mathrm{argmin}_\apf \| \apf\circ\phi - \bellman^*\hat{Q}_{k-1} \|_\infty$. We use a neural network (universal function approximator) to parameterize $h$, so the approximation error $\| \apf\circ\phi - \bellman^*\hat{Q}_{k-1} \|_\infty$ depends on the representation quality of $\phi$.

Therefore, a representation mapping $\phi$ is $\epsilon$-sufficient for learning with value iteration if $\|\apx\bellman^* Q_k - \bellman^* Q_k\|_\infty\leq \epsilon$. 
Next, we identify the relationship between policy-independent model sufficiency and the learning sufficiency with value iteration methods.

\textbf{Guaranteed Learning with Model-regularized Representation}\quad
We first make the following assumption for the learned approximation function $h$.
\begin{assumption}[Lipschitz Value Approximation]
\label{assump:lips_appro}
There exists a constant $\lipsapx$, such that $\forall k\geq 0, o_1,o_2\in\states, a\in\actions$, 
\begin{equation}
    | \apf_k(\phi(o_1),a) - \apf_k(\phi(o_2),a) | \leq \lipsapx \| \phi(o_1) - \phi(o_2) \|,
\end{equation}
where $\apf_k$ is the approximation function in the $k$-th iteration.
\end{assumption}

Then, the following Theorem holds, which justifies that learning with model-regularized representation helps with value iteration learning.
\begin{theorem}
\label{thm:error_avi}
For an MDP $\mdp$, if encoder $\phi$ satisfies
$\max_{o\in\states,a\in\actions}|R(o,a) - \hat{R}_a(\phi(o)) | \leq \epsilon_R$ and
$\max_{o\in\states,a\in\actions} \| \mathbb{E}_{o^\prime\sim P(\cdot|o,a)} \phi(o^\prime) - \hat{P}_a (\phi(o)) \|_2 \leq \epsilon_P$ for dynamics models $(\apa,\ara)_{a\in\actions}$, 
then the approximated value iteration with approximation operator $\apx$ under Assumption~\ref{assump:lips_appro} satisfies 
\begin{equation}
    \limsup_{k\to\infty} \|V^* - V^{\pi_k} \|_{\infty} \leq \frac{2}{(1-\gamma)^2}(\epsilon_R+\gamma\epsilon_P\lipsapx).
\end{equation}
\end{theorem}

\begin{proof}[Proof of Theorem~\ref{thm:error_avi}]
Let $\apf_k$ be the approximation function in the $k$-th iteration. That is, the approximated Q function in the $k$-th iteration is $\hat{Q}_k=\apf_k\circ\phi$. As the rewards of all state-action pairs are non-negative, we have $\apf_k(\phi(o),a)\geq 0$.

Define a function $\hat{\apf}_{k+1}$ as
\begin{equation}
    \hat{\apf}_{k+1}(\phi(o),a) = \ara(\phi(o)) + \gamma \max_{a^\prime\in\actions} \apf_k(\apa(\phi(o)), a^\prime)
\end{equation}

Given that 
\begin{align}
    \bellman^*\hat{Q}_k(o,a) &= R(o,a) + \gamma \mathbb{E}_{o^\prime\sim P(\cdot|o,a)}[\max_{a^\prime\in\actions} \hat{Q}_k(o^\prime,a^\prime)] \\
    &= R(o,a) + \gamma \mathbb{E}_{o^\prime\sim P(\cdot|o,a)}[\max_{a^\prime\in\actions} \apf_k(\phi(o^\prime), a^\prime)],
\end{align}
we have that for any $o\in\states$, $a\in\actions$,
\setlength\abovedisplayskip{3pt}
\setlength\belowdisplayskip{3pt}
\begin{align}
    &\left|\hat{\apf}_{k+1}(\phi(o),a) - \bellman^*\hat{Q}_k(o,a)\right| \\
    =& \left| \ara(\phi(o)) + \gamma \max_{a^\prime\in\actions} \apf_k(\apa(\phi(o)), a^\prime) - (R(o,a) + \gamma \mathbb{E}_{o^\prime\sim P(\cdot|o,a)}[\max_{a^\prime\in\actions} \apf_k(\phi(o^\prime), a^\prime)]) \right| \\
    \leq& \left| \ara(\phi(o)) - R(o,a)\right| + \gamma \left| \max_{a^\prime\in\actions} \apf_k(\apa(\phi(o)), a^\prime) - \mathbb{E}_{o^\prime\sim P(\cdot|o,a)}[\max_{a^\prime\in\actions} \apf_k(\phi(o^\prime), a^\prime)]  \right|\\
    \leq& \epsilon_R + \gamma \mathbb{E}_{o^\prime\sim P(\cdot|o,a)}[ \left| \max_{a^\prime\in\actions} \apf_k(\apa(\phi(o)), a^\prime) - \max_{a^\prime\in\actions} \apf_k(\phi(o^\prime), a^\prime) \right| ]\\
    \leq& \epsilon_R + \gamma \mathbb{E}_{o^\prime\sim P(\cdot|o,a)} \max_{a^\prime\in\actions} \left| \apf_k(\apa(\phi(o)), a^\prime) - \apf_k(\phi(o^\prime), a^\prime) \right| \label{eq:maxa}\\
    \leq& \epsilon_R + \gamma \mathbb{E}_{o^\prime\sim P(\cdot|o,a)} \max_{a^\prime\in\actions} \lipsapx \left\| \apa(\phi(o) - \phi(o^\prime) \right\| \label{eq:lips}\\
    \leq& \epsilon_R + \gamma \epsilon_P \lipsapx,
\end{align}
where (\ref{eq:maxa}) is due to the non-negativity of $h_k$, and (\ref{eq:lips}) is due to Assumption~\ref{assump:lips_appro}.

Now we have shown that the constructed $\hat{h}_{k+1}$ satisfies 
\begin{equation}
    \|\hat{\apf}_{k+1}\circ \phi - \bellman^* \hat{Q}_k \|_\infty \leq \epsilon_R + \gamma \epsilon_P \lipsapx.
\end{equation}

According to the definition of $\apx$, we obtain
\begin{equation}
\label{eq:apx_error}
    \|\apf_{k+1}\circ \phi - \bellman^* \hat{Q}_k \|_\infty \leq \|\hat{\apf}_{k+1}\circ \phi - \bellman^* \hat{Q}_k \|_\infty \leq \epsilon_R + \gamma \epsilon_P \lipsapx
\end{equation}
since $\apx$ finds a $\apf_{k+1}$ that minimizes the approximation error.

Therefore, Theorem~\ref{thm:error_avi} follows by combining Inequality~(\ref{eq:apx_error}) and Lemma~\ref{lem:avi}.

\end{proof}

%% file: a2-exp.tex
\section{Experiment Details and Additional Results}
\label{app:exp}

\subsection{Experiment Setting Details}
\label{app:exp_setting}

\subsubsection{Baselines}
\label{app:exp_baseline}
\begin{itemize}
    \item \textbf{Single}: A DQN or SAC learner on the target domain without any auxiliary tasks.
    \item \textbf{Auxiliary}: On the target domain, the encoder $\phi\target$ is optimized based on the loss $L_{\text{base}}(\phi\target, \pi\target) + \lambda \Big[L_P(\phi\target;\hat{P}\target)+ L_R(\phi\target;\hat{R}\target)\Big]$. Compared with our transfer algorithms which transfer the learned dynamics from source domain to the target domain, it learns the dynamics model $(\hat{P}\target, \hat{R}\target)$ on the target domain from scratch. Here we set $\lambda$ to be the same as our transferred algorithm (values of $\lambda$ are provided in Appendix~\ref{app:exp_drl}). The purpose of this baseline is to test whether the efficiency of our proposed transfer algorithms come from the transferred latent dynamics or from the auxiliary loss (or potentially both). 
    \item \textbf{Fine-tune}: To test whether our transfer algorithms benefit from loading the learned policy head $\pi\source$, on the target domain, we load the weights of $\pi\target$ from the trained source policy head $\pi\source$ and train the DQN or SAC agent without any auxiliary loss.
    \item \textbf{Time-aligned}:
    \citet{gupta2017learning} propose to learn aligned representations for two tasks, under the assumption that the source-task agent and the target-task agent reach similar latent sates at the same time step, i.e. $\phi\target(s\target_t)=\phi\source(s\source_t)$. Note that this assumption is valid when the initial state is fixed and the transitions are all deterministic. 
    Although in our setting, the agent can not learn both tasks simultaneously, we can adapt the idea of time-based alignment and encourage the target encoder to map target observations to the source representations happening at the same time step.\\
    In our experiments, we store $N$ source trajectories $\Big\{s^{i}_0,a^{i}_0,s^{i}_{1}, a^{i}_{1},...,\Big\}_{i=1}^{N}$ collected during source task training. Then on the target domain, we first collect $N$ trajectories following the same action as the one collected from the source domain. In other words, at time step $t$ of the $i$-th trajectory, we take action $a_{t}^i$. 
    After the target trajectories are collected, we minimize the alignment loss $L_{\text{align}}(\phi\target)=\mathbb{E}\Big[\big(\phi\target(s\target_t)-\phi\source(s\source_t)\big)^2\Big]$ to enforce that observations from source and target domain at the same time-step have the same representations.\newline
    In our experiments, we set $N$ to be 10\% of the training trajectories. (We also experimented with larger $N$', for example using all the training trajectories, but the differences are minor.) In terms the alignment loss, we optimize the loss for 1000 epochs with batch size equal to 256, where at each epoch we sample a batch of paired source and target observations and compute the alignment loss. After pre-training the target encoder, we load the weight into $\phi\target$ and resumes the normal DQN or SAC training. \\
    Our experimental results show that, although more training steps are given to the time-aligned learner, it does not outperform the single-task learner, and sometimes fails to learn (e.g. in 3DBall). The main reason is that the time-based assumption does not hold in practice as initial states are usually randomly generated. 
    Therefore, even though the agent exactly imitates the source-task policy at every step, 
    the observations from source and target task do not necessarily match at every time-step. 
    In environments with non-deterministic transitions, the state mismatch will be a more severe issue and may lead to an unreasonable encoder.
\end{itemize}

\subsubsection{Environments}
\label{app:exp_env}

\textbf{Environment Settings in Vec-to-pixel Tasks}\quad
\begin{itemize}
    \item CartPole: The source task is the same as the ordinary CartPole environment on Gym. For the pixel-input target task, we extract the screen of the environment which is of size (400,600), and crop the pixel input to let the image be centered at the cart. The resulting observation has size (40,90) after cropping. We take the difference between two consecutive frames as the agent's observation.
    \item Acrobot: The source task is the same as the ordinary Acrobot environment on Gym. For the pixel-input target task, we first extract the screen of the environment which is of size (150,150), and then down-sample the image to (40,40). We also take the difference between two consecutive frames as the agent's observation.\newline
    \item Cheetah-Run: The source task is the Cheetah Run Task provided by DeepMind Control Suite (DMC)~\citep{tassa2018deepmind}. For the target task, we use the image of size (84,84) rendered from the environment as the agent's observation.
\end{itemize}
\textbf{Environment Settings in More-sensor Tasks}\quad
For the target task of MuJoCo environments, we add the center of the mass based inertia and velocity into the observations of the agent, concatenating them with the original observation on the source task. Consequently, in the target environments, the dimensionality of the observation space on target task become much larger than that of the source task. On Hopper, the dimensionality of the target observation is 91, whereas the the source observation space only has 11 dimensions. The dimensionalities of target tasks on HalfCheetah, Hopper and Walker are 145, 91, 145 respectively.

\textbf{Environment Settings in Broken-sensor Tasks}\quad
3DBall is an example environment provided by the ML-Agents Toolkit~\citep{juliani2018unity}. In this task, the agent (a cube) is supposed to balance a ball on its top. At every step, the agent will be rewarded if the ball is still on its top. If the ball falls off, the episode will immediately end.
The highest episodic return in this task is 100. 
There are two versions of this game, which only differ by their observation spaces.
The simpler version (named 3DBall in the toolkit) has 8 observation features corresponding to the rotation of the agent cube, and the position and velocity of the ball.
The harder version (named 3DBallHard in the toolkit) does not have access to the ball velocity, but observes a stack of 9 past frames, each of which corresponds to the rotation of the agent cube, and the position of the ball, resulting in 45 observation dimensions at every step.
We regard 3DBall as the source task and 3DBallHard as the target task in our experiments.

\subsubsection{Implementation of Base DRL Algorithms and Hyper-parameter Settings}
\label{app:exp_drl}

\textbf{Implementation of DQN}\quad
To ensure that the base learning algorithm learns the pixel-input target tasks well, we follow the existing online codebases for pre-processing, architectures and hyperparameter settings in pixel CartPole\footnote{\href{https://pytorch.org/tutorials/intermediate/reinforcement_q_learning.html}{https://pytorch.org/tutorials/intermediate/reinforcement\_q\_learning.html}} and pixel Acrobot
\footnote{\href{https://github.com/eyalbd2/Deep_RL_Course }{https://github.com\/eyalbd2\/Deep\_RL\_Course}}. On source domain, the DQN network has a 2-layer encoder and a 2-layer Q head of hidden size 64, and the representation dimension is set as 16. For pixel-input, the encoder has three convolution layers followed by a linear layer. The number of channels of the convolutional layers are equal to 16, 32, 32, respectively (kernel size=5 for all three layers). We use the Adam optimizer with learning rate $0.001$ and $\beta_1, \beta_2=0.9,0.999$. The target Q network is updated every 10 iterations.
In CartPole, we use a replay buffer with size 10000. In the more challenging Acrobot, we use a prioritized replay buffer with size 100000.

\textbf{Implementation of SAC}\quad
For MuJoCo environments, we follow an elegant open-sourced SAC implementation\footnote{\href{https://github.com/pranz24/pytorch-soft-actor-critic}{https://github.com/pranz24/pytorch-soft-actor-critic}}.
The number of hidden units for all neural networks is 256.
The actor has a two-layer encoder and a two-layer policy head.
The two Q networks both have three linear layers.
The activation function is ReLU and the learning rate is $3\cdot 10^{-4}$.
We train the dynamics model and the reward model every 50k interactive steps in the source task.
For the DMC environment Cheetah-Run, we follow the open-sourced SAC implementation with an autoencoder \footnote{\href{https://github.com/denisyarats/pytorch_sac_ae}{https://github.com/denisyarats/pytorch\_sac\_ae}}.
The pixel encoder has three convolution layers and one linear layer. The number of channels for all convolutional layers is 32 and the kernel size is 3.
For the 3DBall environment, as it can only be learned within the ML-Agents toolkit, we directly use the SAC implementation provided by the toolkit with the default hyperparameter settings.

\textbf{Implementation of Latent Dynamics Model}\quad
Note that our goal is to learn a good representation by enforcing it predicting the latent dynamics, different from model-based RL~\citep{hafner2019dream} that aims to learn accurate models for planning. Therefore, we let the dynamics models $\hat{P}$ and $\hat{R}$ be simple linear networks, so that the representation can be more informative in terms of representing dynamics and learning values/policies. For environments with discrete action spaces, we learn $|\actions|$ linear transition networks and $|\actions|$ linear reward models. For environments with continuous action spaces, we first learn an action encoder $\psi: \actions \to \mathbb{R}^d$ with the same encoding size $d$ as the state representation. Then, we learn a linear transition network and a linear reward network with $\hat{P}(\phi(o)\circ\psi(a))$ being the predicted next representation, and $\hat{R}(\phi(o)\circ\psi(a))$ being the predicted reward, where $\circ$ denotes element-wise product. In practice, we find this implementation achieves good performance across many environments. \\
In addition, note that due to the significant difference between source observation and target observation, the initial encoding scale could be very different in source and target tasks, making it hard for them to be regularized by the same dynamics model. Therefore, we normalize the output of both encoders to be a unit vector (l2 norm is 1), which remedies the potential mismatch in their scales.

\textbf{Hyperparameter Settings for Transfer Learning}\quad
In experiments, we find that it is better to set $\lambda$ relatively large when the environment dynamics are simple and the dynamics model is of high quality. When the environment dynamics is complex, we choose to be more conservative and set $\lambda$ to be smaller. Concretely, in CartPole, $\lambda$ is set as 18; in 3DBall, $\lambda$ is set as 10; in Acrobot, $\lambda$ is set as 5; in the remaining MuJoCo environments where dynamics are more complicated, $\lambda$ is set as 1. Although we use different $\lambda$'s in different environments based on domain knowledge, we find that different values of $\lambda$'s do not have much influence on the learning performance. Figure~\ref{fig:hyper} provided in Appendix~\ref{app:exp_results} shows a test on the hyper-parameter $\lambda$, where we can see that our algorithm effectively transfers knowledge under various values of $\lambda$. \\
Regarding the representation dimension, we set it to be smaller for simpler tasks, and larger for more complex tasks. In 3DBall, we set the encoding size to be 8; in CartPole, we set the encoding size as 16; in Acrobot, we set the encoding size as 32; in Cheetah-Run, we set the encoding size as 50; in MuJoCo tasks, we set the encoding size as 256. Again, we find that the feature size does not influence the performance too much. But based on the theoretical insights of learning minimal sufficient representation~\cite{achille2018emergence}, we believe that it is generally better to have a lower-dimensional representation while making sure it is sufficient for learning.

\newpage 
\subsection{Additional Experimental Results}
\label{app:exp_results}

\textbf{Ablation Study: Transferring Different Components}\\
Figure~\ref{fig:mujoco_ablation} shows the ablation study of our method in continuous control tasks. We compare our method with the following variants:\\
(1) learning auxiliary tasks without transfer, \\
(2) only transferring transition models $\hat{P}$ and \\
(3) only transferring reward models $\hat{R}$.

Compared with the single-task learning baseline (the blue curves), we find that all the variants of our method can make some improvements, which suggests that learning dynamics models as auxiliary tasks, transferring $\hat{P}$ and $\hat{R}$ are all effective designs for accelerating the target task learning. Finally, our method (the red curves) that combines the above components achieves the best performance, justifying the effectiveness of our transfer algorithm.

\begin{figure}[!htbp]
\centering
 \begin{subfigure}[t]{0.48\columnwidth}
  \centering
  \input{figs/halfcheetah_ablation}
 \end{subfigure}
 \hfill
 \hspace{-1em}
 \begin{subfigure}[t]{0.48\columnwidth}
  \centering
  \input{figs/hopper_ablation} 
 \end{subfigure}

 \begin{subfigure}[t]{0.48\columnwidth}
  \centering
  \input{figs/walker2d_ablation} 
 \end{subfigure}
 \hfill
 \begin{subfigure}[t]{0.48\columnwidth}
  \centering
  \input{figs/ball_ablation_all} 
 \end{subfigure}
 \vspace{-2em}
 \caption{\small{Ablation study of our method on different transferred components.}}
\label{fig:mujoco_ablation}
\end{figure}
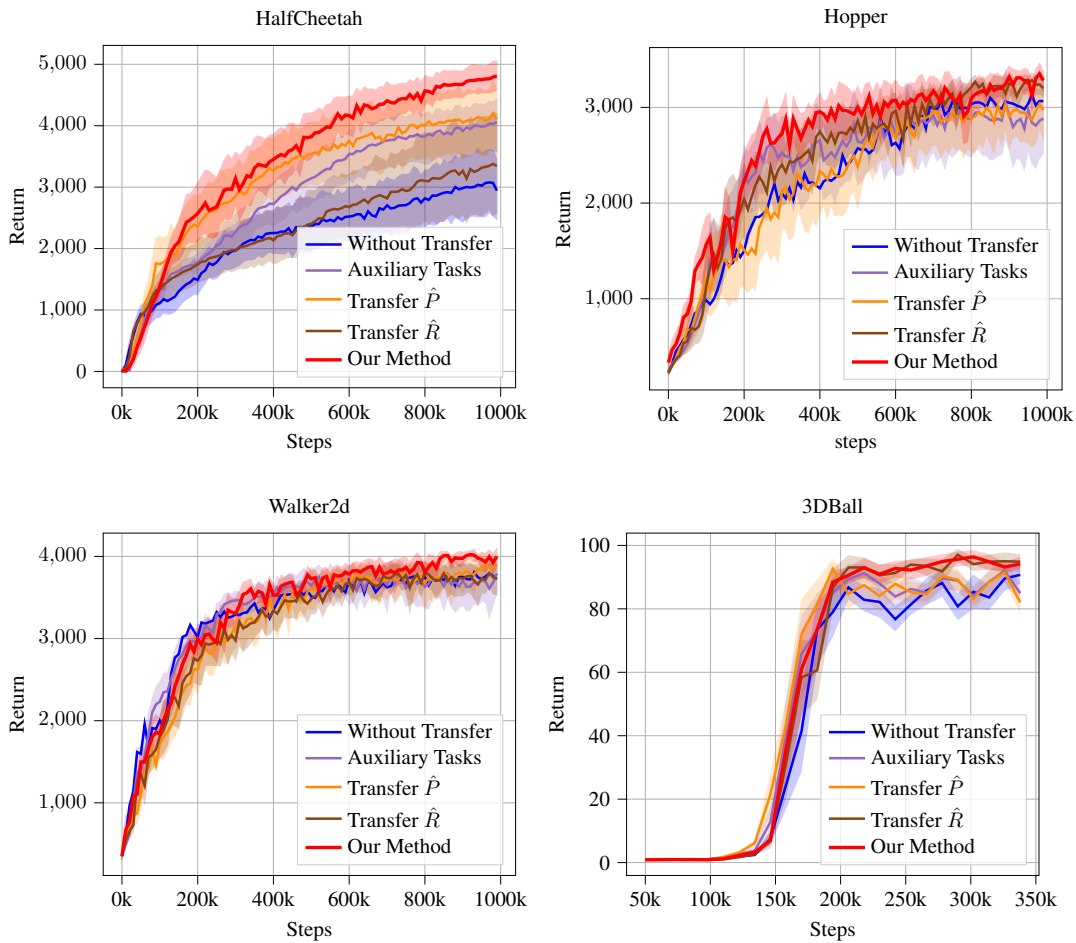

\newpage 
\begin{wrapfigure}{r}{0.5\textwidth}
    \centering
    \input{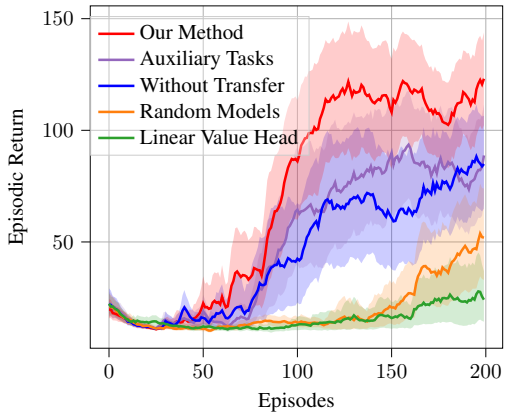}
    \vspace{-1em}
    \caption{In the Vec-to-pixel CartPole environment, sanity check verifies the effectiveness of our algorithm design. Results are averaged over 20 random seeds.}
    \label{fig:cart_ablation}
    \vspace{-1em}
\end{wrapfigure}
\textbf{Sanity Check: Effectiveness of the Proposed Transfer Method}\\
We conduct another ablation study to evaluate each component of our algorithm in the CartPole environment as shown in Figure~\ref{fig:cart_ablation}. We find that when transferring the dynamics models with only a linear value head (the \textcolor{linearcolor}{\textbf{green}} curve), the agent fails to learn a good policy as we analyzed in Section~\ref{sec:theory}. If the dynamics models $(\hat{P},\hat{R})$ are randomly generated instead of being transferred from the source task (the \textcolor{randomcolor}{\textbf{orange}} curve), the agent does not learn, either. More importantly, if we learn dynamics models as auxiliary tasks in the target task without transferring them from the source (the \textcolor{auxcolor}{\textbf{purple}} curve), the agent learns a little better than a vanilla agent, but is worse than our proposed transfer algorithm. 
These empirical results have verified our theoretical insights and shown the effectiveness of our algorithm design.

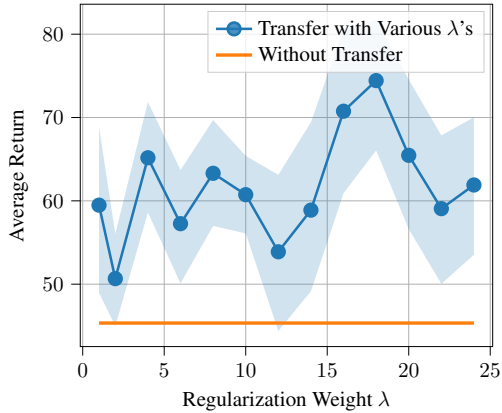
\begin{wrapfigure}{r}{0.5\textwidth}
    \centering
    \input{figs/coeff_compare}
    \vspace{-1em}
    \caption{In the Vec-to-pixel CartPole environment, under different selections of hyperparameter $\lambda$, the algorithm works better than learning from scratch (when $\lambda=0$). Results are averaged over 20 random seeds.}
    \label{fig:hyper}
    \vspace{-1em}
\end{wrapfigure}
\textbf{Hyper-parameter Test}\\
Figure~\ref{fig:hyper} further visualizes how the hyperparameter $\lambda$ (regularization weight) influences the transfer performance in the Vec-to-pixel CartPole environment. It can be found that the agent generally benefits from a larger $\lambda$, which suggests that the model-based regularization has a positive impact on the learning performance. For a wide range of $\lambda$'s, the agent always outperforms the learner without transfer (the learner with $\lambda=0$). Therefore, our algorithm is not sensitive to the hyperparameter $\lambda$, and a larger $\lambda$ is preferred to get better performance.
In Appendix~\ref{app:exp_drl}, we have provided the $\lambda$ selections for all experiments.

%% file: figs/halfcheetah_ablation.tex
\begin{tikzpicture}[scale=0.8]

\definecolor{color0}{rgb}{1,0.549019607843137,0}
\definecolor{color1}{rgb}{0.55,0.27,0.07}
\definecolor{color2}{rgb}{0.580392156862745,0.403921568627451,0.741176470588235}
\begin{axis}[
legend cell align={left},
legend style={
  fill opacity=0.8,
  draw opacity=1,
  text opacity=1,
  at={(0.97,0.03)},
  anchor=south east,
  draw=white!80!black
},
tick align=outside,
tick pos=left,
title={{HalfCheetah}},
x grid style={white!69.0196078431373!black},
xlabel={{Steps}},
xmajorgrids,
xmin=-4.95, xmax=103.95,
xtick style={color=black},
xtick={0,20,40,60,80,100},
xticklabels={0k,200k,400k,600k,800k,1000k},
y grid style={white!69.0196078431373!black},
ylabel={{Return}},
ymajorgrids,
ymin=-266.200568040981, ymax=5307.13375964089,
ytick style={color=black}
]

\path [draw=blue, fill=blue, opacity=0.25]
(axis cs:0,-4.69312890992943)
--(axis cs:0,-17.6489278332041)
--(axis cs:1,48.4196230838893)
--(axis cs:2,224.085892993268)
--(axis cs:3,433.899416829319)
--(axis cs:4,651.374335075094)
--(axis cs:5,773.396204019732)
--(axis cs:6,705.042808075804)
--(axis cs:7,760.161541903011)
--(axis cs:8,803.536662605523)
--(axis cs:9,875.682041316625)
--(axis cs:10,885.637977920961)
--(axis cs:11,953.945708140717)
--(axis cs:12,933.478008970819)
--(axis cs:13,936.394664162679)
--(axis cs:14,976.784508843318)
--(axis cs:15,1044.01115949428)
--(axis cs:16,1122.89113824202)
--(axis cs:17,1171.63907748341)
--(axis cs:18,1206.67203424047)
--(axis cs:19,1260.59090609802)
--(axis cs:20,1250.54455430699)
--(axis cs:21,1354.73887332777)
--(axis cs:22,1427.61495406515)
--(axis cs:23,1541.55730138082)
--(axis cs:24,1498.62922634593)
--(axis cs:25,1512.56763950996)
--(axis cs:26,1584.96851908096)
--(axis cs:27,1661.32729266277)
--(axis cs:28,1658.7909143109)
--(axis cs:29,1671.01726694997)
--(axis cs:30,1700.57198310876)
--(axis cs:31,1725.43229780922)
--(axis cs:32,1784.67324464652)
--(axis cs:33,1813.37514601896)
--(axis cs:34,1814.57075376337)
--(axis cs:35,1827.30362641694)
--(axis cs:36,1889.85165307179)
--(axis cs:37,1849.35630621537)
--(axis cs:38,1901.63608022612)
--(axis cs:39,1939.34565873711)
--(axis cs:40,1918.4598069219)
--(axis cs:41,1923.91189121227)
--(axis cs:42,1915.66353764437)
--(axis cs:43,1911.7823852757)
--(axis cs:44,1963.88424897745)
--(axis cs:45,1855.73292242444)
--(axis cs:46,1964.81377241806)
--(axis cs:47,1950.87120639778)
--(axis cs:48,1988.57636764583)
--(axis cs:49,1951.72755093991)
--(axis cs:50,1975.58382666033)
--(axis cs:51,2002.22142456002)
--(axis cs:52,1998.59455882182)
--(axis cs:53,2092.82886829939)
--(axis cs:54,2059.58402119079)
--(axis cs:55,2057.71443926768)
--(axis cs:56,2058.7937248732)
--(axis cs:57,2096.41509010125)
--(axis cs:58,2088.23064204906)
--(axis cs:59,2135.95653027454)
--(axis cs:60,2112.0276991081)
--(axis cs:61,2094.14393708331)
--(axis cs:62,2129.68823208115)
--(axis cs:63,2150.22678577463)
--(axis cs:64,2105.6983531378)
--(axis cs:65,2196.73761359791)
--(axis cs:66,2124.17689335286)
--(axis cs:67,2123.24469944621)
--(axis cs:68,2250.66183458359)
--(axis cs:69,2227.34996481021)
--(axis cs:70,2280.03528442713)
--(axis cs:71,2200.90614397843)
--(axis cs:72,2259.90838042161)
--(axis cs:73,2234.31113634627)
--(axis cs:74,2268.24493270003)
--(axis cs:75,2221.58352396378)
--(axis cs:76,2314.9714861276)
--(axis cs:77,2351.01458754251)
--(axis cs:78,2351.23930258279)
--(axis cs:79,2342.63091426649)
--(axis cs:80,2355.17721514452)
--(axis cs:81,2393.87269598959)
--(axis cs:82,2385.48619678364)
--(axis cs:83,2384.0409035733)
--(axis cs:84,2440.83991414794)
--(axis cs:85,2439.51289913974)
--(axis cs:86,2494.02581088025)
--(axis cs:87,2401.87644951462)
--(axis cs:88,2514.24933260454)
--(axis cs:89,2467.42468478773)
--(axis cs:90,2431.06380696522)
--(axis cs:91,2482.89499648818)
--(axis cs:92,2523.12975903549)
--(axis cs:93,2530.47036770005)
--(axis cs:94,2515.39463403946)
--(axis cs:95,2514.90800725056)
--(axis cs:96,2516.38323946424)
--(axis cs:97,2591.70943577726)
--(axis cs:98,2526.40176067548)
--(axis cs:99,2473.75989093863)
--(axis cs:99,3454.32196767895)
--(axis cs:99,3454.32196767895)
--(axis cs:98,3609.73378916989)
--(axis cs:97,3611.54469969064)
--(axis cs:96,3592.85436087659)
--(axis cs:95,3547.86391986894)
--(axis cs:94,3634.7367440994)
--(axis cs:93,3552.24064280243)
--(axis cs:92,3516.38991959465)
--(axis cs:91,3517.31199253149)
--(axis cs:90,3537.09571141303)
--(axis cs:89,3470.03184185023)
--(axis cs:88,3521.08294793979)
--(axis cs:87,3426.38840446597)
--(axis cs:86,3388.77013027803)
--(axis cs:85,3435.71225686408)
--(axis cs:84,3344.13206993206)
--(axis cs:83,3398.45591905594)
--(axis cs:82,3350.99659979894)
--(axis cs:81,3282.67761419759)
--(axis cs:80,3280.49833740014)
--(axis cs:79,3198.24259859315)
--(axis cs:78,3300.44430820178)
--(axis cs:77,3335.93305012482)
--(axis cs:76,3222.92360387623)
--(axis cs:75,3205.36273603947)
--(axis cs:74,3211.87898550048)
--(axis cs:73,3172.25999671678)
--(axis cs:72,3109.84155968791)
--(axis cs:71,3124.96787720871)
--(axis cs:70,3163.06425688898)
--(axis cs:69,3074.99848861961)
--(axis cs:68,3116.20696424908)
--(axis cs:67,3005.83997643866)
--(axis cs:66,3012.22194145395)
--(axis cs:65,3017.23617523856)
--(axis cs:64,2929.36030729206)
--(axis cs:63,3013.96105719838)
--(axis cs:62,3042.2285707763)
--(axis cs:61,2995.49006578194)
--(axis cs:60,3004.89135106062)
--(axis cs:59,2926.32790914883)
--(axis cs:58,2887.60308914719)
--(axis cs:57,2962.38089581548)
--(axis cs:56,2857.88003498309)
--(axis cs:55,2902.30851668398)
--(axis cs:54,2906.06224826494)
--(axis cs:53,2943.71297474148)
--(axis cs:52,2905.90897586828)
--(axis cs:51,2808.82830172019)
--(axis cs:50,2822.91882513191)
--(axis cs:49,2836.74555430608)
--(axis cs:48,2818.795029409)
--(axis cs:47,2829.0425858007)
--(axis cs:46,2794.79586331184)
--(axis cs:45,2666.39236658714)
--(axis cs:44,2745.8878524189)
--(axis cs:43,2679.87048234605)
--(axis cs:42,2660.19393199652)
--(axis cs:41,2670.46585951361)
--(axis cs:40,2628.65629460396)
--(axis cs:39,2618.58010609789)
--(axis cs:38,2604.86259578869)
--(axis cs:37,2533.24865148957)
--(axis cs:36,2592.74091976821)
--(axis cs:35,2495.3427092436)
--(axis cs:34,2500.6602469018)
--(axis cs:33,2461.90140681)
--(axis cs:32,2446.6558919426)
--(axis cs:31,2404.01450603108)
--(axis cs:30,2285.02492142732)
--(axis cs:29,2295.6636825512)
--(axis cs:28,2260.34917269084)
--(axis cs:27,2298.37364203994)
--(axis cs:26,2182.01495371068)
--(axis cs:25,2106.22151709448)
--(axis cs:24,1980.42603690572)
--(axis cs:23,2026.30751491157)
--(axis cs:22,1945.72167887143)
--(axis cs:21,1889.87562858262)
--(axis cs:20,1754.79374489955)
--(axis cs:19,1827.83810516501)
--(axis cs:18,1733.65633603928)
--(axis cs:17,1700.04953086017)
--(axis cs:16,1664.84087711702)
--(axis cs:15,1584.29110470418)
--(axis cs:14,1487.0887834827)
--(axis cs:13,1471.997757771)
--(axis cs:12,1367.02244697723)
--(axis cs:11,1422.12400642221)
--(axis cs:10,1359.79623665067)
--(axis cs:9,1304.89186705402)
--(axis cs:8,1230.02892686611)
--(axis cs:7,1106.97423373146)
--(axis cs:6,1111.60031346159)
--(axis cs:5,1079.7081359667)
--(axis cs:4,934.77096809036)
--(axis cs:3,726.040469646256)
--(axis cs:2,513.026876317791)
--(axis cs:1,157.085270087035)
--(axis cs:0,-4.69312890992943)
--cycle;

\path [draw=color2, fill=color2, opacity=0.25]
(axis cs:0,46.4049684178804)
--(axis cs:0,-3.1944013522495)
--(axis cs:1,1.77491809581539)
--(axis cs:2,2.14177742628211)
--(axis cs:3,160.092161622056)
--(axis cs:4,312.065551367168)
--(axis cs:5,408.54806645707)
--(axis cs:6,633.553170433161)
--(axis cs:7,861.02354761137)
--(axis cs:8,923.699593260105)
--(axis cs:9,1034.13536122978)
--(axis cs:10,1191.1327102446)
--(axis cs:11,1196.36175667335)
--(axis cs:12,1382.09544823016)
--(axis cs:13,1483.98920797277)
--(axis cs:14,1511.84065827494)
--(axis cs:15,1543.724137511)
--(axis cs:16,1559.35243182548)
--(axis cs:17,1587.6014687426)
--(axis cs:18,1612.74301909282)
--(axis cs:19,1583.90533633748)
--(axis cs:20,1611.79684426801)
--(axis cs:21,1667.94583876547)
--(axis cs:22,1701.5644320497)
--(axis cs:23,1738.27827707309)
--(axis cs:24,1771.79454841964)
--(axis cs:25,1864.73599227034)
--(axis cs:26,1845.36476766936)
--(axis cs:27,1981.67471077717)
--(axis cs:28,1979.59417240678)
--(axis cs:29,1930.22247958016)
--(axis cs:30,1971.59707322419)
--(axis cs:31,2060.00317438772)
--(axis cs:32,2111.17739268972)
--(axis cs:33,2085.6394407831)
--(axis cs:34,2160.50845421661)
--(axis cs:35,2284.75094174242)
--(axis cs:36,2338.78669174957)
--(axis cs:37,2360.0590991717)
--(axis cs:38,2388.7772040506)
--(axis cs:39,2401.14478938955)
--(axis cs:40,2364.59515631575)
--(axis cs:41,2420.26935512812)
--(axis cs:42,2436.24525067552)
--(axis cs:43,2487.38201044031)
--(axis cs:44,2559.53656075675)
--(axis cs:45,2642.99679289012)
--(axis cs:46,2587.38490678681)
--(axis cs:47,2613.31715083706)
--(axis cs:48,2692.79726738622)
--(axis cs:49,2666.24127850063)
--(axis cs:50,2747.89307776528)
--(axis cs:51,2774.87398074118)
--(axis cs:52,2825.99290116884)
--(axis cs:53,2814.97724856398)
--(axis cs:54,2878.42801620634)
--(axis cs:55,2878.99348443219)
--(axis cs:56,3000.17237057851)
--(axis cs:57,3024.19891129187)
--(axis cs:58,3039.53039150985)
--(axis cs:59,3049.20263313369)
--(axis cs:60,3130.95444203899)
--(axis cs:61,3073.46687222197)
--(axis cs:62,3170.60807284338)
--(axis cs:63,3163.90863319882)
--(axis cs:64,3126.03607721863)
--(axis cs:65,3278.04460922075)
--(axis cs:66,3267.04849238562)
--(axis cs:67,3271.6101331913)
--(axis cs:68,3282.39943859596)
--(axis cs:69,3299.92999277311)
--(axis cs:70,3306.60459739647)
--(axis cs:71,3338.82651338291)
--(axis cs:72,3316.4559898878)
--(axis cs:73,3376.612016229)
--(axis cs:74,3352.23174724611)
--(axis cs:75,3362.26057973579)
--(axis cs:76,3422.05454467443)
--(axis cs:77,3420.30704003751)
--(axis cs:78,3410.40099632884)
--(axis cs:79,3395.25135223454)
--(axis cs:80,3378.44887317217)
--(axis cs:81,3372.68369309771)
--(axis cs:82,3432.03236710533)
--(axis cs:83,3463.93067800023)
--(axis cs:84,3492.17884030902)
--(axis cs:85,3441.10026580878)
--(axis cs:86,3409.16166932222)
--(axis cs:87,3387.71911264992)
--(axis cs:88,3470.20077383434)
--(axis cs:89,3566.01915091405)
--(axis cs:90,3448.53124930152)
--(axis cs:91,3492.45320382607)
--(axis cs:92,3514.52964365984)
--(axis cs:93,3569.66374176191)
--(axis cs:94,3547.90278601316)
--(axis cs:95,3543.56798679413)
--(axis cs:96,3527.56091994612)
--(axis cs:97,3597.73262747857)
--(axis cs:98,3583.45357730936)
--(axis cs:99,3577.9657926669)
--(axis cs:99,4429.37897260584)
--(axis cs:99,4429.37897260584)
--(axis cs:98,4386.44194609022)
--(axis cs:97,4397.43580972091)
--(axis cs:96,4350.99967330653)
--(axis cs:95,4351.93985487165)
--(axis cs:94,4352.27506164069)
--(axis cs:93,4362.04067665877)
--(axis cs:92,4288.3511293256)
--(axis cs:91,4311.44822736269)
--(axis cs:90,4271.830382911)
--(axis cs:89,4335.50772955162)
--(axis cs:88,4325.0289615321)
--(axis cs:87,4219.89770581493)
--(axis cs:86,4271.13064245505)
--(axis cs:85,4222.52127567816)
--(axis cs:84,4271.64285857476)
--(axis cs:83,4280.88130818263)
--(axis cs:82,4175.91984766266)
--(axis cs:81,4232.69119217255)
--(axis cs:80,4203.11390046369)
--(axis cs:79,4174.78806345755)
--(axis cs:78,4227.66354568115)
--(axis cs:77,4178.74958093564)
--(axis cs:76,4179.81005986612)
--(axis cs:75,4137.86139497789)
--(axis cs:74,4163.98657240457)
--(axis cs:73,4193.43118009519)
--(axis cs:72,4135.50960351967)
--(axis cs:71,4129.40488792699)
--(axis cs:70,4118.33913907081)
--(axis cs:69,4110.83020822662)
--(axis cs:68,4086.30998442988)
--(axis cs:67,4045.0782266466)
--(axis cs:66,4053.93944044067)
--(axis cs:65,4058.60751883368)
--(axis cs:64,3988.84666665166)
--(axis cs:63,3941.6856077587)
--(axis cs:62,3980.72753718981)
--(axis cs:61,3934.45770585279)
--(axis cs:60,3937.0594887812)
--(axis cs:59,3826.68693174467)
--(axis cs:58,3825.88470299377)
--(axis cs:57,3812.25792183462)
--(axis cs:56,3773.62035578867)
--(axis cs:55,3631.66727465638)
--(axis cs:54,3696.5138259668)
--(axis cs:53,3662.34709983434)
--(axis cs:52,3591.16694700229)
--(axis cs:51,3544.70742539881)
--(axis cs:50,3535.24417294059)
--(axis cs:49,3476.54816605777)
--(axis cs:48,3450.36590985585)
--(axis cs:47,3442.97800919469)
--(axis cs:46,3268.90206996188)
--(axis cs:45,3347.5640435908)
--(axis cs:44,3277.82950307037)
--(axis cs:43,3244.79147815704)
--(axis cs:42,3212.60678438025)
--(axis cs:41,3132.99322313696)
--(axis cs:40,3076.17336806987)
--(axis cs:39,3047.85543718029)
--(axis cs:38,3001.82162722788)
--(axis cs:37,2991.10561815664)
--(axis cs:36,2984.67152549884)
--(axis cs:35,2833.25008552212)
--(axis cs:34,2737.41767019988)
--(axis cs:33,2695.14758341957)
--(axis cs:32,2664.76417535951)
--(axis cs:31,2619.82334478286)
--(axis cs:30,2566.7587878217)
--(axis cs:29,2437.53196054235)
--(axis cs:28,2458.82281298098)
--(axis cs:27,2439.89167763469)
--(axis cs:26,2312.78664378705)
--(axis cs:25,2237.92642771515)
--(axis cs:24,2172.56395317124)
--(axis cs:23,2112.58123778248)
--(axis cs:22,2058.60829017592)
--(axis cs:21,1972.05602599422)
--(axis cs:20,1937.29382815097)
--(axis cs:19,1860.13064565994)
--(axis cs:18,1851.96780047886)
--(axis cs:17,1837.53749314227)
--(axis cs:16,1803.62399620397)
--(axis cs:15,1739.13250587978)
--(axis cs:14,1698.9007480097)
--(axis cs:13,1681.65452921625)
--(axis cs:12,1608.21262602646)
--(axis cs:11,1485.20634252639)
--(axis cs:10,1474.13646540115)
--(axis cs:9,1411.52610980622)
--(axis cs:8,1345.05594936498)
--(axis cs:7,1260.75910684166)
--(axis cs:6,987.586866919023)
--(axis cs:5,793.149952227815)
--(axis cs:4,647.954873664815)
--(axis cs:3,370.288975095488)
--(axis cs:2,47.91600726901)
--(axis cs:1,96.1597613712627)
--(axis cs:0,46.4049684178804)
--cycle;

\path [draw=color0, fill=color0, opacity=0.25]
(axis cs:0,80.3293013688214)
--(axis cs:0,-6.10413996843572)
--(axis cs:1,1.96385391655971)
--(axis cs:2,87.8233312662103)
--(axis cs:3,177.470593422115)
--(axis cs:4,379.493481424522)
--(axis cs:5,610.940630404371)
--(axis cs:6,876.376382222473)
--(axis cs:7,944.260809085335)
--(axis cs:8,1037.73804800724)
--(axis cs:9,1425.38550181593)
--(axis cs:10,1395.12031593969)
--(axis cs:11,1420.26866359989)
--(axis cs:12,1554.39738985974)
--(axis cs:13,1626.39228487944)
--(axis cs:14,1561.89921511591)
--(axis cs:15,1634.70611769191)
--(axis cs:16,1798.9681001895)
--(axis cs:17,1824.6406208857)
--(axis cs:18,1950.66523641905)
--(axis cs:19,1989.8055509507)
--(axis cs:20,1958.05430432234)
--(axis cs:21,2030.06732672104)
--(axis cs:22,2090.68331689381)
--(axis cs:23,2154.11892517449)
--(axis cs:24,2165.43150621581)
--(axis cs:25,2199.57548263918)
--(axis cs:26,2277.87237973822)
--(axis cs:27,2286.41734257478)
--(axis cs:28,2294.00306000102)
--(axis cs:29,2364.62275348257)
--(axis cs:30,2401.89922881462)
--(axis cs:31,2444.73536473869)
--(axis cs:32,2401.30835854457)
--(axis cs:33,2412.60320226658)
--(axis cs:34,2607.80548064566)
--(axis cs:35,2587.40610266662)
--(axis cs:36,2688.03923423811)
--(axis cs:37,2720.5866462361)
--(axis cs:38,2799.94348210032)
--(axis cs:39,2849.61085321679)
--(axis cs:40,2761.22828345932)
--(axis cs:41,2837.25642415158)
--(axis cs:42,2856.07354049166)
--(axis cs:43,2886.78885989495)
--(axis cs:44,2881.29687084922)
--(axis cs:45,2940.82371234163)
--(axis cs:46,2936.1212674993)
--(axis cs:47,2975.34635378839)
--(axis cs:48,2982.27189186647)
--(axis cs:49,3039.33367621624)
--(axis cs:50,3054.21427075625)
--(axis cs:51,3041.23172999995)
--(axis cs:52,3031.99051905439)
--(axis cs:53,3093.13023819287)
--(axis cs:54,3096.93889394751)
--(axis cs:55,3145.39664477085)
--(axis cs:56,3189.40504447203)
--(axis cs:57,3144.22560431442)
--(axis cs:58,3173.53236797083)
--(axis cs:59,3139.50715347467)
--(axis cs:60,3229.54730031778)
--(axis cs:61,3184.14449092257)
--(axis cs:62,3236.34354283633)
--(axis cs:63,3343.67310978361)
--(axis cs:64,3311.61403909489)
--(axis cs:65,3260.74127647395)
--(axis cs:66,3369.7760055498)
--(axis cs:67,3259.13770052641)
--(axis cs:68,3195.3508435027)
--(axis cs:69,3238.53681688386)
--(axis cs:70,3372.07663477743)
--(axis cs:71,3334.58152508458)
--(axis cs:72,3294.95708997173)
--(axis cs:73,3432.26667303332)
--(axis cs:74,3288.51031289589)
--(axis cs:75,3389.86757570329)
--(axis cs:76,3506.95419687701)
--(axis cs:77,3474.64528598548)
--(axis cs:78,3450.97698208801)
--(axis cs:79,3455.91194600123)
--(axis cs:80,3490.43171913994)
--(axis cs:81,3506.87753387765)
--(axis cs:82,3433.33970472685)
--(axis cs:83,3515.35654452105)
--(axis cs:84,3494.41578281947)
--(axis cs:85,3531.79053528634)
--(axis cs:86,3554.97718954374)
--(axis cs:87,3527.24146749828)
--(axis cs:88,3572.01671791209)
--(axis cs:89,3526.32921748129)
--(axis cs:90,3486.84349492241)
--(axis cs:91,3535.81242951268)
--(axis cs:92,3543.90387632925)
--(axis cs:93,3590.91333412486)
--(axis cs:94,3554.58046135999)
--(axis cs:95,3518.37739556343)
--(axis cs:96,3506.25475322572)
--(axis cs:97,3605.58390313171)
--(axis cs:98,3668.49924910747)
--(axis cs:99,3574.03466078758)
--(axis cs:99,4563.21164169044)
--(axis cs:99,4563.21164169044)
--(axis cs:98,4733.74427137818)
--(axis cs:97,4598.24389745897)
--(axis cs:96,4548.61062632993)
--(axis cs:95,4596.94834305126)
--(axis cs:94,4618.54219847214)
--(axis cs:93,4556.87514447651)
--(axis cs:92,4587.20225016785)
--(axis cs:91,4585.25638225635)
--(axis cs:90,4510.3701587507)
--(axis cs:89,4647.51742548549)
--(axis cs:88,4579.93993390898)
--(axis cs:87,4538.5138904151)
--(axis cs:86,4581.76870341598)
--(axis cs:85,4522.39781656194)
--(axis cs:84,4560.91114867896)
--(axis cs:83,4536.63498445484)
--(axis cs:82,4524.90776115001)
--(axis cs:81,4548.87903284887)
--(axis cs:80,4487.03068023991)
--(axis cs:79,4528.59220450173)
--(axis cs:78,4458.39381192085)
--(axis cs:77,4503.32790838365)
--(axis cs:76,4554.38041480546)
--(axis cs:75,4458.69236112002)
--(axis cs:74,4363.12226921224)
--(axis cs:73,4483.77770365669)
--(axis cs:72,4336.1033260055)
--(axis cs:71,4366.23662732267)
--(axis cs:70,4338.50118037157)
--(axis cs:69,4246.00667224316)
--(axis cs:68,4323.67186333415)
--(axis cs:67,4275.45391017837)
--(axis cs:66,4399.59947211684)
--(axis cs:65,4388.37713641696)
--(axis cs:64,4295.75316302207)
--(axis cs:63,4313.36513093371)
--(axis cs:62,4275.97106438637)
--(axis cs:61,4140.89152926772)
--(axis cs:60,4223.24436639783)
--(axis cs:59,4203.66294215935)
--(axis cs:58,4164.10767947866)
--(axis cs:57,4129.5304445731)
--(axis cs:56,4203.56083608163)
--(axis cs:55,4069.19710819088)
--(axis cs:54,4118.02246204486)
--(axis cs:53,4174.5332031276)
--(axis cs:52,4044.91661129699)
--(axis cs:51,4070.45493347925)
--(axis cs:50,4046.29162787029)
--(axis cs:49,4050.21615745851)
--(axis cs:48,3977.41258630031)
--(axis cs:47,3957.36043580442)
--(axis cs:46,3949.32831055728)
--(axis cs:45,4017.76265050416)
--(axis cs:44,3847.19391728282)
--(axis cs:43,3935.9591971133)
--(axis cs:42,3859.60430375916)
--(axis cs:41,3849.54859462534)
--(axis cs:40,3790.8220226594)
--(axis cs:39,3803.52279287445)
--(axis cs:38,3782.3308539618)
--(axis cs:37,3662.89951522716)
--(axis cs:36,3647.16597777271)
--(axis cs:35,3581.39059609381)
--(axis cs:34,3596.53851468752)
--(axis cs:33,3439.84256447943)
--(axis cs:32,3344.14330862945)
--(axis cs:31,3445.62787082704)
--(axis cs:30,3360.62619189301)
--(axis cs:29,3282.93036786349)
--(axis cs:28,3301.64888504898)
--(axis cs:27,3260.09473730749)
--(axis cs:26,3279.44211659119)
--(axis cs:25,3136.92828161725)
--(axis cs:24,3060.76550809105)
--(axis cs:23,3086.42627100675)
--(axis cs:22,3006.75152399056)
--(axis cs:21,2893.40870284598)
--(axis cs:20,2865.09245408262)
--(axis cs:19,2873.25532716587)
--(axis cs:18,2792.96777830534)
--(axis cs:17,2744.39825597597)
--(axis cs:16,2618.43101264525)
--(axis cs:15,2457.4030940497)
--(axis cs:14,2267.14572631033)
--(axis cs:13,2399.72640528198)
--(axis cs:12,2196.60183628725)
--(axis cs:11,2181.74314315018)
--(axis cs:10,2190.71102119483)
--(axis cs:9,2172.7479495633)
--(axis cs:8,1822.54287262797)
--(axis cs:7,1473.38880974864)
--(axis cs:6,1308.74924517248)
--(axis cs:5,945.724393337461)
--(axis cs:4,715.272795757182)
--(axis cs:3,507.581506681246)
--(axis cs:2,410.377279770895)
--(axis cs:1,79.5503322263406)
--(axis cs:0,80.3293013688214)
--cycle;

\path [draw=color1, fill=color1, opacity=0.25]
(axis cs:0,24.2261254993369)
--(axis cs:0,-4.51020156194292)
--(axis cs:1,-1.80116644201299)
--(axis cs:2,117.612136321464)
--(axis cs:3,572.772428305984)
--(axis cs:4,738.428286102509)
--(axis cs:5,853.121822677509)
--(axis cs:6,905.123300905511)
--(axis cs:7,986.16003403064)
--(axis cs:8,1050.28720709853)
--(axis cs:9,1097.4426825671)
--(axis cs:10,1098.47772394299)
--(axis cs:11,1155.41739351208)
--(axis cs:12,1180.88627573804)
--(axis cs:13,1226.74153046599)
--(axis cs:14,1274.87073049647)
--(axis cs:15,1310.01884896546)
--(axis cs:16,1319.92774816757)
--(axis cs:17,1343.36412193719)
--(axis cs:18,1367.80697718297)
--(axis cs:19,1405.93919460333)
--(axis cs:20,1439.96494504112)
--(axis cs:21,1472.03120421769)
--(axis cs:22,1488.37973083063)
--(axis cs:23,1519.47578351701)
--(axis cs:24,1529.56749312731)
--(axis cs:25,1560.87545969154)
--(axis cs:26,1592.66580211334)
--(axis cs:27,1605.55764173378)
--(axis cs:28,1646.15460664854)
--(axis cs:29,1623.1039922337)
--(axis cs:30,1650.3328575982)
--(axis cs:31,1680.76291656206)
--(axis cs:32,1676.48761332775)
--(axis cs:33,1677.78371931962)
--(axis cs:34,1720.33081859767)
--(axis cs:35,1693.91800059807)
--(axis cs:36,1754.43578144849)
--(axis cs:37,1719.91762082148)
--(axis cs:38,1762.59082852844)
--(axis cs:39,1806.46502406104)
--(axis cs:40,1756.88736366951)
--(axis cs:41,1768.47663911663)
--(axis cs:42,1815.01228508833)
--(axis cs:43,1831.43779460445)
--(axis cs:44,1859.55306414583)
--(axis cs:45,1852.89190678897)
--(axis cs:46,1858.33868886623)
--(axis cs:47,1945.06402269248)
--(axis cs:48,1921.49700371504)
--(axis cs:49,1942.83844392141)
--(axis cs:50,1953.67439193923)
--(axis cs:51,2004.1969047332)
--(axis cs:52,2016.82667341245)
--(axis cs:53,2041.15557244076)
--(axis cs:54,2103.31728939404)
--(axis cs:55,2145.4399435915)
--(axis cs:56,2161.47187840897)
--(axis cs:57,2166.78679653473)
--(axis cs:58,2173.87549480597)
--(axis cs:59,2147.91581584204)
--(axis cs:60,2201.89707320127)
--(axis cs:61,2246.84343781432)
--(axis cs:62,2243.51988642852)
--(axis cs:63,2239.28088923078)
--(axis cs:64,2294.30527797646)
--(axis cs:65,2186.36696370271)
--(axis cs:66,2290.82597629219)
--(axis cs:67,2312.83802275269)
--(axis cs:68,2306.05329724227)
--(axis cs:69,2300.25234668021)
--(axis cs:70,2268.91630148821)
--(axis cs:71,2293.16548723638)
--(axis cs:72,2316.86348509291)
--(axis cs:73,2305.59488785186)
--(axis cs:74,2392.05529049324)
--(axis cs:75,2337.70840080518)
--(axis cs:76,2315.44959911267)
--(axis cs:77,2352.66052182162)
--(axis cs:78,2375.65055446111)
--(axis cs:79,2382.79566435707)
--(axis cs:80,2463.23170319693)
--(axis cs:81,2399.18123917645)
--(axis cs:82,2451.66003549811)
--(axis cs:83,2410.15162490413)
--(axis cs:84,2437.74218384354)
--(axis cs:85,2460.9067982772)
--(axis cs:86,2513.98702263173)
--(axis cs:87,2559.6055537063)
--(axis cs:88,2520.71789405971)
--(axis cs:89,2489.81102292544)
--(axis cs:90,2563.44704165788)
--(axis cs:91,2425.91243212712)
--(axis cs:92,2548.07126658214)
--(axis cs:93,2563.70333961568)
--(axis cs:94,2549.06517522773)
--(axis cs:95,2528.10298311439)
--(axis cs:96,2582.33047247445)
--(axis cs:97,2627.50164594706)
--(axis cs:98,2543.44665621198)
--(axis cs:99,2575.73595869099)
--(axis cs:99,4189.09414129455)
--(axis cs:99,4189.09414129455)
--(axis cs:98,4222.85190779033)
--(axis cs:97,4145.36680955847)
--(axis cs:96,4146.76132151378)
--(axis cs:95,4140.74506760534)
--(axis cs:94,4083.59391200874)
--(axis cs:93,4028.00718375438)
--(axis cs:92,4137.14438296203)
--(axis cs:91,3877.85838879896)
--(axis cs:90,3995.8641679712)
--(axis cs:89,3974.30537648029)
--(axis cs:88,4117.31013956378)
--(axis cs:87,4041.36119583785)
--(axis cs:86,3987.36371253588)
--(axis cs:85,3856.78506152618)
--(axis cs:84,3882.45416590113)
--(axis cs:83,3858.07739502629)
--(axis cs:82,3908.28153547251)
--(axis cs:81,3861.93871166612)
--(axis cs:80,3886.24811356819)
--(axis cs:79,3852.96397870503)
--(axis cs:78,3863.96152429133)
--(axis cs:77,3649.69002023653)
--(axis cs:76,3626.49160310709)
--(axis cs:75,3589.06119150634)
--(axis cs:74,3614.78498207497)
--(axis cs:73,3542.18491469192)
--(axis cs:72,3532.36582112409)
--(axis cs:71,3559.77685751654)
--(axis cs:70,3405.81466694667)
--(axis cs:69,3453.39897303487)
--(axis cs:68,3476.54713601714)
--(axis cs:67,3518.90321915476)
--(axis cs:66,3389.02830739906)
--(axis cs:65,3336.32507234969)
--(axis cs:64,3340.69241044602)
--(axis cs:63,3292.49875288296)
--(axis cs:62,3259.56889690421)
--(axis cs:61,3267.04382561445)
--(axis cs:60,3197.30397040561)
--(axis cs:59,3196.55571982304)
--(axis cs:58,3220.30610998626)
--(axis cs:57,3142.20531589811)
--(axis cs:56,3190.0211410354)
--(axis cs:55,3138.80279563834)
--(axis cs:54,3107.23562664697)
--(axis cs:53,3010.33053194988)
--(axis cs:52,3004.60681038015)
--(axis cs:51,2939.72137678403)
--(axis cs:50,2942.1048269851)
--(axis cs:49,2948.15235761484)
--(axis cs:48,2793.86396801984)
--(axis cs:47,2866.18599004738)
--(axis cs:46,2742.92602442177)
--(axis cs:45,2693.50479293692)
--(axis cs:44,2679.11260592283)
--(axis cs:43,2700.52576831334)
--(axis cs:42,2657.83108263532)
--(axis cs:41,2686.08112598459)
--(axis cs:40,2578.03980539098)
--(axis cs:39,2605.69743132684)
--(axis cs:38,2628.19506968759)
--(axis cs:37,2565.58483703954)
--(axis cs:36,2531.26018425)
--(axis cs:35,2551.72215309384)
--(axis cs:34,2507.78460958715)
--(axis cs:33,2476.31608233635)
--(axis cs:32,2466.31902766929)
--(axis cs:31,2450.89271577406)
--(axis cs:30,2392.67522255248)
--(axis cs:29,2350.83273064836)
--(axis cs:28,2332.28977671714)
--(axis cs:27,2290.0779587064)
--(axis cs:26,2265.81911799326)
--(axis cs:25,2253.50178456486)
--(axis cs:24,2252.17414301039)
--(axis cs:23,2233.49306219714)
--(axis cs:22,2150.8585804693)
--(axis cs:21,2184.2610350763)
--(axis cs:20,2153.9840167293)
--(axis cs:19,2059.88101711186)
--(axis cs:18,2040.46590496512)
--(axis cs:17,2006.85466052475)
--(axis cs:16,2007.10413628188)
--(axis cs:15,1917.92197427493)
--(axis cs:14,1884.78102678594)
--(axis cs:13,1869.270457375)
--(axis cs:12,1794.08153395784)
--(axis cs:11,1730.48464551681)
--(axis cs:10,1625.90043396206)
--(axis cs:9,1552.82078560329)
--(axis cs:8,1399.78489604456)
--(axis cs:7,1153.33656119734)
--(axis cs:6,1033.31009329836)
--(axis cs:5,969.732689431078)
--(axis cs:4,824.55437944107)
--(axis cs:3,729.343101364136)
--(axis cs:2,363.606222875996)
--(axis cs:1,91.8919013479733)
--(axis cs:0,24.2261254993369)
--cycle;

\path [draw=red, fill=red, opacity=0.25]
(axis cs:0,48.7095079349399)
--(axis cs:0,-7.30785449160513)
--(axis cs:1,-12.8671895099869)
--(axis cs:2,10.6498819523004)
--(axis cs:3,77.8329103459216)
--(axis cs:4,226.505065832334)
--(axis cs:5,332.9102716835)
--(axis cs:6,491.965473894132)
--(axis cs:7,694.687646189033)
--(axis cs:8,904.038796941001)
--(axis cs:9,1081.32346456871)
--(axis cs:10,1262.55849715297)
--(axis cs:11,1441.31703435582)
--(axis cs:12,1571.73020799688)
--(axis cs:13,1726.50278855384)
--(axis cs:14,1783.94029175917)
--(axis cs:15,1809.68170192985)
--(axis cs:16,1906.21425479935)
--(axis cs:17,2064.93207550842)
--(axis cs:18,2094.3447485723)
--(axis cs:19,2104.87584622365)
--(axis cs:20,2122.11215120852)
--(axis cs:21,2214.14075691662)
--(axis cs:22,2255.9422141265)
--(axis cs:23,2183.27165135618)
--(axis cs:24,2173.15984479526)
--(axis cs:25,2204.6868887364)
--(axis cs:26,2391.91268420802)
--(axis cs:27,2485.61933074168)
--(axis cs:28,2471.19384086682)
--(axis cs:29,2484.96410872609)
--(axis cs:30,2727.27741833266)
--(axis cs:31,2466.6909197146)
--(axis cs:32,2585.02067245254)
--(axis cs:33,2731.17817087167)
--(axis cs:34,2691.49595737574)
--(axis cs:35,2860.32444196475)
--(axis cs:36,2952.98966261781)
--(axis cs:37,2926.24023357579)
--(axis cs:38,2861.54418243717)
--(axis cs:39,2996.85218154587)
--(axis cs:40,3039.14487585924)
--(axis cs:41,3136.12495157322)
--(axis cs:42,3140.10190717849)
--(axis cs:43,3053.5841490409)
--(axis cs:44,3170.67937112573)
--(axis cs:45,3213.18976005074)
--(axis cs:46,3315.59702707421)
--(axis cs:47,3183.24674127966)
--(axis cs:48,3457.68790378846)
--(axis cs:49,3429.98463481286)
--(axis cs:50,3467.88282574794)
--(axis cs:51,3549.17596360765)
--(axis cs:52,3500.89674905224)
--(axis cs:53,3621.18426775035)
--(axis cs:54,3625.27895069009)
--(axis cs:55,3596.03759616266)
--(axis cs:56,3740.71021460753)
--(axis cs:57,3908.61560223821)
--(axis cs:58,3802.81955273787)
--(axis cs:59,3863.73742167561)
--(axis cs:60,3853.77929891312)
--(axis cs:61,3818.065801354)
--(axis cs:62,4064.17204214605)
--(axis cs:63,3897.14713770914)
--(axis cs:64,4048.69009945898)
--(axis cs:65,4048.51479500494)
--(axis cs:66,4006.20148236414)
--(axis cs:67,4030.6861412147)
--(axis cs:68,4077.93471095143)
--(axis cs:69,4087.99917151147)
--(axis cs:70,4142.11485463693)
--(axis cs:71,4108.43229266956)
--(axis cs:72,4084.0122957673)
--(axis cs:73,4178.35274962541)
--(axis cs:74,4193.2238640741)
--(axis cs:75,4205.44552076535)
--(axis cs:76,4266.54875454272)
--(axis cs:77,4185.01753521327)
--(axis cs:78,4270.66021685746)
--(axis cs:79,4137.59305656165)
--(axis cs:80,4340.33129171447)
--(axis cs:81,4364.66447597977)
--(axis cs:82,4313.53056549152)
--(axis cs:83,4432.68961233444)
--(axis cs:84,4404.67332181747)
--(axis cs:85,4440.87793932001)
--(axis cs:86,4448.90467369875)
--(axis cs:87,4447.40955252939)
--(axis cs:88,4469.75111831538)
--(axis cs:89,4453.96909444737)
--(axis cs:90,4455.01864116496)
--(axis cs:91,4512.61368064045)
--(axis cs:92,4526.00674630882)
--(axis cs:93,4541.04850730409)
--(axis cs:94,4515.81578620823)
--(axis cs:95,4539.44316335004)
--(axis cs:96,4537.4779131639)
--(axis cs:97,4534.35289441518)
--(axis cs:98,4548.70403239847)
--(axis cs:99,4586.82183378184)
--(axis cs:99,5053.8003811099)
--(axis cs:99,5053.8003811099)
--(axis cs:98,5046.62416034271)
--(axis cs:97,5000.30422419439)
--(axis cs:96,4973.74249108297)
--(axis cs:95,4982.81106880917)
--(axis cs:94,4982.12951757894)
--(axis cs:93,4995.47782143402)
--(axis cs:92,4981.6376778329)
--(axis cs:91,4938.60503805324)
--(axis cs:90,4890.53361127605)
--(axis cs:89,4890.16190705245)
--(axis cs:88,4893.99000921434)
--(axis cs:87,4892.11482180455)
--(axis cs:86,4889.0082730929)
--(axis cs:85,4907.45409468035)
--(axis cs:84,4847.59275079302)
--(axis cs:83,4868.24984889792)
--(axis cs:82,4773.55409269018)
--(axis cs:81,4797.24580230255)
--(axis cs:80,4814.03402404453)
--(axis cs:79,4711.42226188192)
--(axis cs:78,4789.39638683669)
--(axis cs:77,4710.75800561544)
--(axis cs:76,4751.47631339311)
--(axis cs:75,4680.21678951907)
--(axis cs:74,4693.85287767365)
--(axis cs:73,4679.00301993693)
--(axis cs:72,4659.05737149386)
--(axis cs:71,4662.3220550268)
--(axis cs:70,4719.43039105384)
--(axis cs:69,4636.87712907093)
--(axis cs:68,4629.54219099191)
--(axis cs:67,4576.38834498138)
--(axis cs:66,4598.93915688325)
--(axis cs:65,4593.20891021444)
--(axis cs:64,4635.17454688076)
--(axis cs:63,4547.65095493543)
--(axis cs:62,4609.1292358815)
--(axis cs:61,4484.46247780099)
--(axis cs:60,4479.80555358369)
--(axis cs:59,4467.43875617875)
--(axis cs:58,4462.51723333678)
--(axis cs:57,4471.03915159175)
--(axis cs:56,4453.53041120177)
--(axis cs:55,4349.33849853635)
--(axis cs:54,4349.16155486204)
--(axis cs:53,4250.65228232437)
--(axis cs:52,4290.70581701184)
--(axis cs:51,4231.64176007242)
--(axis cs:50,4211.30088817515)
--(axis cs:49,4198.34806494185)
--(axis cs:48,4199.51617470304)
--(axis cs:47,4008.15729863971)
--(axis cs:46,4017.41641852425)
--(axis cs:45,4043.75264404519)
--(axis cs:44,4035.4580851211)
--(axis cs:43,3934.07695019044)
--(axis cs:42,3972.00807179465)
--(axis cs:41,3922.3670372375)
--(axis cs:40,3851.10007922692)
--(axis cs:39,3850.74856226755)
--(axis cs:38,3774.56820227011)
--(axis cs:37,3808.79015320458)
--(axis cs:36,3835.18770025905)
--(axis cs:35,3802.09967613191)
--(axis cs:34,3659.57631488221)
--(axis cs:33,3668.86289362134)
--(axis cs:32,3574.75530073646)
--(axis cs:31,3507.90205614356)
--(axis cs:30,3568.61238344969)
--(axis cs:29,3459.61115932133)
--(axis cs:28,3526.77867842184)
--(axis cs:27,3447.61030843759)
--(axis cs:26,3279.10359292781)
--(axis cs:25,3153.2598857932)
--(axis cs:24,3206.44654725348)
--(axis cs:23,3260.56890823777)
--(axis cs:22,3276.95530403881)
--(axis cs:21,3176.1516709148)
--(axis cs:20,3062.72328844904)
--(axis cs:19,2962.14318015461)
--(axis cs:18,2893.32821010861)
--(axis cs:17,2878.6996473074)
--(axis cs:16,2660.86342440608)
--(axis cs:15,2450.42211159532)
--(axis cs:14,2325.85378281345)
--(axis cs:13,2180.11434067542)
--(axis cs:12,1982.7602212415)
--(axis cs:11,1705.52102370049)
--(axis cs:10,1500.79809321423)
--(axis cs:9,1373.70494375825)
--(axis cs:8,1268.44419596225)
--(axis cs:7,1097.34728051216)
--(axis cs:6,964.626732242247)
--(axis cs:5,741.045193230484)
--(axis cs:4,580.856263194181)
--(axis cs:3,331.477325930986)
--(axis cs:2,173.588479701568)
--(axis cs:1,18.7349784680543)
--(axis cs:0,48.7095079349399)
--cycle;

\addplot [very thick, blue]
table {%
0 -11.084663737933
1 99.8035792984406
2 369.900565979479
3 584.491087768476
4 802.898785917725
5 933.856890457755
6 904.084758783503
7 931.573936866358
8 996.272641607167
9 1082.42602324171
10 1111.27649033751
11 1176.66357321053
12 1145.95858919874
13 1181.6331366736
14 1218.25336727096
15 1304.48026881758
16 1364.40074847913
17 1417.35190403534
18 1450.59813822457
19 1511.67369013555
20 1490.86598336642
21 1601.70149076006
22 1645.50966594009
23 1760.60678544715
24 1731.20170199876
25 1777.2110833715
26 1852.4236165584
27 1940.59609356621
28 1908.87075198114
29 1956.21941726234
30 1966.08798451785
31 2014.66053648493
32 2077.86305765496
33 2110.61393629852
34 2139.18284279186
35 2135.63615161644
36 2196.00203690086
37 2174.61057232744
38 2238.91776854732
39 2249.92945282929
40 2253.31614702679
41 2254.1585224219
42 2275.8540739049
43 2274.0966216333
44 2329.24586028362
45 2233.59602682774
46 2349.25067696172
47 2337.02365386167
48 2353.55958223237
49 2372.39457456569
50 2353.60573199366
51 2397.41680530923
52 2400.91861893596
53 2477.31980285485
54 2445.4002537686
55 2468.92304876792
56 2444.07235533217
57 2509.26514813107
58 2475.08070046857
59 2515.71377399278
60 2521.05501264526
61 2530.97282563386
62 2558.09002081235
63 2557.31663447986
64 2486.5556732633
65 2606.46933549486
66 2542.37099610749
67 2519.07712079405
68 2655.96859416512
69 2632.6132881532
70 2704.0633823554
71 2652.98687993913
72 2663.99173415709
73 2698.21201071726
74 2703.47250301136
75 2707.20847982912
76 2740.02159061625
77 2802.89117356973
78 2819.73546847909
79 2747.61299481459
80 2818.58383828957
81 2790.76639860887
82 2864.91797586068
83 2859.4494813188
84 2887.23845454987
85 2909.31680606687
86 2917.45273638468
87 2904.71604120769
88 3002.83946380317
89 2964.19000782691
90 2984.15340418278
91 3003.54484807172
92 2994.5579650952
93 3002.80014220374
94 3038.36645525333
95 3023.88081388067
96 3067.36878338539
97 3077.44423856982
98 3068.37905131114
99 2942.79358012208
};
\addlegendentry{Without Transfer}
\addplot [very thick, color2]
table {%
0 20.5307341728056
1 46.1879005239307
2 24.1006784452388
3 264.979045112803
4 479.01882278355
5 610.226753209642
6 818.470246923374
7 1073.35650177926
8 1140.13032569925
9 1233.26432432097
10 1338.92983459612
11 1344.32948832973
12 1495.56628699163
13 1583.81863065588
14 1606.95801448276
15 1635.82007187599
16 1687.76206620986
17 1709.23133830847
18 1731.84484098599
19 1728.74080870223
20 1777.00505189865
21 1821.36431328408
22 1880.46403014412
23 1920.78705162559
24 1967.9865747105
25 2053.71156075281
26 2077.07699987766
27 2205.91669220776
28 2217.08339723805
29 2192.96800985991
30 2277.90837657608
31 2341.22488502833
32 2388.40339659383
33 2395.27584770034
34 2464.39559757861
35 2575.66411395406
36 2651.52573534412
37 2671.84453126957
38 2686.71878576602
39 2732.07720874785
40 2724.09730624935
41 2783.60401336346
42 2848.82601564497
43 2884.96095435302
44 2943.80154771072
45 3002.00393854365
46 2949.80290598418
47 3037.28344249644
48 3079.78862628055
49 3113.45079168921
50 3143.94912749899
51 3167.42551396237
52 3212.9623513968
53 3256.03765749084
54 3295.50806042259
55 3294.92064635032
56 3397.76358603925
57 3431.80772543023
58 3434.19708605875
59 3459.22069104264
60 3541.79333825432
61 3537.8686069817
62 3589.50458135528
63 3581.20449378407
64 3617.33595589225
65 3672.01683756633
66 3691.33873934469
67 3681.757003725
68 3705.70232281665
69 3747.58921666423
70 3755.73161303641
71 3768.5188369616
72 3773.41254042095
73 3822.58547821874
74 3799.00024818538
75 3784.03147695444
76 3828.52094199943
77 3827.55769820201
78 3862.25795708833
79 3820.66236838165
80 3827.89627263351
81 3867.78524108468
82 3839.27494689246
83 3900.84584324402
84 3920.70492191252
85 3861.32650086462
86 3887.70118982031
87 3859.41619580224
88 3958.34599172974
89 3970.14254080924
90 3914.22687825776
91 3974.39052760519
92 3946.5033161487
93 4000.35549904009
94 3998.23110535964
95 3983.35101061304
96 3992.31001865269
97 4033.31252019764
98 4030.95408486175
99 4059.84183053874
};
\addlegendentry{Auxiliary Tasks}
\addplot [very thick, color0]
table {%
0 37.3094361967087
1 38.027655689851
2 225.413344247887
3 330.949289150662
4 552.469713705319
5 790.301230837579
6 1098.6936704127
7 1228.60340424714
8 1396.94881513452
9 1758.11399342238
10 1746.69549329209
11 1778.88999305995
12 1832.85407471355
13 1967.52106512268
14 1878.38904909163
15 2002.45149160443
16 2151.49957043548
17 2248.18024856077
18 2334.86237357702
19 2389.59334235251
20 2362.54200894263
21 2456.27126413206
22 2534.88686804889
23 2596.22775966941
24 2609.03838278441
25 2635.58936785133
26 2714.39440346595
27 2739.97127246859
28 2749.36715896382
29 2787.71631704187
30 2851.62278203101
31 2957.94880574962
32 2882.0314630918
33 2914.31010675944
34 3065.41164566259
35 3097.41420047104
36 3151.41042653258
37 3180.90599142946
38 3248.24926876717
39 3307.79511592169
40 3261.80196560642
41 3333.95429166662
42 3331.68219296042
43 3404.06172131972
44 3385.22949323795
45 3458.24309792008
46 3486.40233239145
47 3509.43795709587
48 3470.03657018247
49 3528.45271341551
50 3540.12126406845
51 3567.18856200085
52 3549.19954578798
53 3647.78705512063
54 3604.08735875078
55 3603.89452611701
56 3695.04474584646
57 3634.81887484795
58 3700.51001412433
59 3675.401793967
60 3754.74568283386
61 3671.10344137625
62 3786.52774765705
63 3798.55195016088
64 3826.52615978341
65 3844.4455536602
66 3883.86035268755
67 3791.58388573817
68 3822.41109948936
69 3740.50713313524
70 3849.3641066529
71 3853.19068451966
72 3830.75954781293
73 3972.60840710495
74 3875.74059614119
75 3911.97019478188
76 4003.05256972446
77 4015.19786632733
78 3959.10013128615
79 4014.38813948407
80 3998.21234819561
81 4072.07558475494
82 4008.45895785542
83 4032.3031351032
84 4060.61383578006
85 4043.62358871189
86 4081.01618598949
87 4077.82674494101
88 4094.42063632382
89 4096.93935339484
90 4050.93897993865
91 4081.74230314738
92 4091.94686591034
93 4097.43754028792
94 4139.69908340458
95 4122.06822543156
96 4075.84304381213
97 4113.75960462698
98 4201.03886630064
99 4104.62314516946
};
\addlegendentry{Transfer $\hat{P}$}
\addplot [very thick, color1]
table {%
0 9.8418866310933
1 41.4368104614329
2 239.963946549034
3 653.466651396104
4 783.135201325639
5 909.509845328097
6 968.433466059677
7 1059.60418909313
8 1212.56366027514
9 1315.8055781537
10 1332.88291225554
11 1412.81560045277
12 1453.41364732155
13 1507.91579449019
14 1542.47411257717
15 1576.77944449373
16 1616.66628723079
17 1643.38878739787
18 1665.93007715678
19 1703.67321833492
20 1745.38672584164
21 1785.17846505177
22 1790.87766745657
23 1841.98766487427
24 1856.34239485405
25 1870.88042430354
26 1894.17416666584
27 1927.57505111566
28 1954.21694016843
29 1972.18186903036
30 1969.15359286742
31 2030.20598799758
32 2031.21397789708
33 2044.38615662373
34 2074.76304913471
35 2088.87677826334
36 2120.44338636051
37 2130.98896312018
38 2157.64499306432
39 2173.87081031594
40 2134.42252058999
41 2199.42485116163
42 2214.34859160191
43 2246.17653814415
44 2247.34303986301
45 2220.91524376628
46 2291.44416955129
47 2355.46662645005
48 2335.8222755594
49 2389.71122404354
50 2401.89341898041
51 2442.57802634682
52 2494.67743576228
53 2485.04514180669
54 2600.03574168654
55 2627.01939524707
56 2645.82068150205
57 2635.25257852853
58 2656.07554653328
59 2690.93455432894
60 2674.62661752591
61 2713.74321885722
62 2719.89076820396
63 2761.53996142784
64 2798.27449490454
65 2742.98077666177
66 2814.16434114423
67 2844.19023047353
68 2869.8730490803
69 2847.13506412668
70 2815.23969788522
71 2922.48146730584
72 2923.85171483153
73 2929.74835658998
74 2968.22446713598
75 2936.40171500703
76 2952.4342572882
77 3010.76094959267
78 3089.24083592455
79 3112.67066443793
80 3102.45653347678
81 3095.7474814425
82 3147.14721578525
83 3141.57851477508
84 3138.9976041204
85 3176.63110725011
86 3223.80405305329
87 3260.28030289124
88 3270.41960995327
89 3204.38329824873
90 3265.47733255923
91 3138.56424698424
92 3293.61899890344
93 3248.36066104595
94 3274.14712602086
95 3316.25214661311
96 3329.95699247456
97 3346.63458299668
98 3379.75711400428
99 3343.69961288018
};
\addlegendentry{Transfer $\hat{R}$}

\addplot [ultra thick, red]
table {%
0 21.1362775419912
1 3.17336224103879
2 83.268205330398
3 194.757409276003
4 393.964067050292
5 550.282168103466
6 733.028764717683
7 897.872453584461
8 1101.05540312677
9 1220.27793821503
10 1384.53259392711
11 1574.13199388744
12 1771.60741646491
13 1943.12415611738
14 2045.43543385983
15 2105.16991946714
16 2264.7002899841
17 2427.03643027418
18 2467.33101373369
19 2515.77445594357
20 2571.61546364005
21 2650.96425939603
22 2741.08465127957
23 2624.00188823102
24 2650.58209852973
25 2677.65042523032
26 2814.76181584111
27 2953.23305608932
28 2969.05829477055
29 2964.70779620861
30 3140.91047385336
31 2959.88163027965
32 3058.3701785386
33 3186.37336282369
34 3166.34272497959
35 3312.06228624368
36 3370.66030173289
37 3353.29880510616
38 3337.78783990556
39 3413.98556735254
40 3450.86294798424
41 3496.50148885725
42 3531.50509361021
43 3507.92294897865
44 3584.95488159384
45 3609.35191721403
46 3674.93872695687
47 3574.23672626692
48 3810.93328153851
49 3806.48736549654
50 3819.76885921119
51 3884.2879829191
52 3891.24369594138
53 3932.0716092578
54 3979.91196675239
55 3971.84784793044
56 4079.98433635019
57 4192.59353518088
58 4112.75168802223
59 4172.54006590849
60 4163.68876136098
61 4126.56350056355
62 4324.72303814195
63 4214.13609503513
64 4338.48717219042
65 4321.07870308427
66 4275.2727626628
67 4297.55820060662
68 4347.70248356101
69 4350.80628552002
70 4406.07597467195
71 4377.78273669928
72 4369.92808798936
73 4420.27709963625
74 4442.59056335099
75 4435.47924987347
76 4503.80268275828
77 4440.82514999354
78 4519.67154656011
79 4407.71040007063
80 4566.27313009734
81 4559.45078790367
82 4527.64872882384
83 4639.48483070322
84 4625.59170791827
85 4672.0339357011
86 4658.62775037913
87 4662.2271111355
88 4674.38221986221
89 4670.09550933003
90 4673.43620904179
91 4722.02452343923
92 4736.29910531417
93 4737.31788282022
94 4748.93478492324
95 4751.01246706157
96 4761.09737319514
97 4763.2839379763
98 4791.9619841706
99 4810.51000076615
};
\addlegendentry{Our Method}
\end{axis}

\end{tikzpicture}

%% file: figs/hopper_ablation.tex
\begin{tikzpicture}[scale=0.8]

\definecolor{color0}{rgb}{1,0.549019607843137,0}
\definecolor{color1}{rgb}{0.55,0.27,0.07}
\definecolor{color2}{rgb}{0.580392156862745,0.403921568627451,0.741176470588235}
\begin{axis}[
legend cell align={left},
legend style={
  fill opacity=0.8,
  draw opacity=1,
  text opacity=1,
  at={(0.97,0.03)},
  anchor=south east,
  draw=white!80!black
},
tick align=outside,
tick pos=left,
title={{Hopper}},
x grid style={white!69.0196078431373!black},
xlabel={{steps}},
xmajorgrids,
xmin=-4.95, xmax=103.95,
xtick style={color=black},
xtick={0,20,40,60,80,100},
xticklabels={0k,200k,400k,600k,800k,1000k},
y grid style={white!69.0196078431373!black},
ymajorgrids,
ylabel={{Return}},
ymin=47.2517038881593, ymax=3624.97592162401,
ytick style={color=black}
]

\path [draw=color2, fill=color2, opacity=0.25]
(axis cs:0,247.292806366722)
--(axis cs:0,221.85711389826)
--(axis cs:1,309.783743040074)
--(axis cs:2,344.728043361613)
--(axis cs:3,375.024082998293)
--(axis cs:4,416.396594859147)
--(axis cs:5,476.350161093595)
--(axis cs:6,548.190779644293)
--(axis cs:7,639.852004069495)
--(axis cs:8,705.470059994717)
--(axis cs:9,830.689557495225)
--(axis cs:10,925.030610410932)
--(axis cs:11,1067.63688543197)
--(axis cs:12,1085.52278400633)
--(axis cs:13,1133.09061679767)
--(axis cs:14,1143.7796502912)
--(axis cs:15,1229.51500162515)
--(axis cs:16,1400.26966379336)
--(axis cs:17,1371.05936692635)
--(axis cs:18,1607.01540508912)
--(axis cs:19,1706.93285070991)
--(axis cs:20,1986.5199169859)
--(axis cs:21,1899.57222798881)
--(axis cs:22,1908.12267972762)
--(axis cs:23,1932.24488922602)
--(axis cs:24,2117.08583199424)
--(axis cs:25,2093.02468017649)
--(axis cs:26,2216.69322470528)
--(axis cs:27,2255.9492834924)
--(axis cs:28,2143.50203009065)
--(axis cs:29,2094.98168285379)
--(axis cs:30,2123.79485136005)
--(axis cs:31,2051.47663367561)
--(axis cs:32,2090.39942794534)
--(axis cs:33,1916.07755622759)
--(axis cs:34,2114.62697693972)
--(axis cs:35,1936.78324889936)
--(axis cs:36,1991.33529915474)
--(axis cs:37,2037.11745190284)
--(axis cs:38,2234.48594650153)
--(axis cs:39,2194.65806945007)
--(axis cs:40,2253.53260710441)
--(axis cs:41,2035.23334889978)
--(axis cs:42,2071.49846847513)
--(axis cs:43,2114.55817666617)
--(axis cs:44,2006.08629906251)
--(axis cs:45,2152.45555530958)
--(axis cs:46,2189.03297283975)
--(axis cs:47,2232.4872949305)
--(axis cs:48,2230.63955101515)
--(axis cs:49,2388.64946213562)
--(axis cs:50,2192.24747896927)
--(axis cs:51,2221.49765687323)
--(axis cs:52,2300.10721649776)
--(axis cs:53,2341.27918173073)
--(axis cs:54,2386.20026295494)
--(axis cs:55,2301.95393605178)
--(axis cs:56,2447.90496139067)
--(axis cs:57,2246.22990420658)
--(axis cs:58,2329.38797717438)
--(axis cs:59,2378.93916745534)
--(axis cs:60,2291.21622953324)
--(axis cs:61,2288.8515543347)
--(axis cs:62,2312.33428378439)
--(axis cs:63,2484.82778342949)
--(axis cs:64,2492.87674797247)
--(axis cs:65,2376.76478960012)
--(axis cs:66,2354.77818568314)
--(axis cs:67,2471.43198310546)
--(axis cs:68,2371.19023052219)
--(axis cs:69,2485.49989605906)
--(axis cs:70,2535.74813095094)
--(axis cs:71,2436.87887940527)
--(axis cs:72,2490.0651220222)
--(axis cs:73,2360.90722307716)
--(axis cs:74,2513.21546987917)
--(axis cs:75,2481.17416626682)
--(axis cs:76,2524.56442386035)
--(axis cs:77,2580.72822291951)
--(axis cs:78,2525.88989491583)
--(axis cs:79,2497.77395792081)
--(axis cs:80,2480.69732614854)
--(axis cs:81,2553.76770368339)
--(axis cs:82,2559.97047205417)
--(axis cs:83,2498.03191674125)
--(axis cs:84,2497.1843025849)
--(axis cs:85,2449.82930095186)
--(axis cs:86,2462.31579122229)
--(axis cs:87,2492.19062404976)
--(axis cs:88,2504.18022132319)
--(axis cs:89,2670.55391772668)
--(axis cs:90,2553.49853628345)
--(axis cs:91,2378.86899090793)
--(axis cs:92,2476.74109793404)
--(axis cs:93,2448.42766772197)
--(axis cs:94,2346.65349723563)
--(axis cs:95,2566.34011835299)
--(axis cs:96,2411.74770234783)
--(axis cs:97,2363.31073119805)
--(axis cs:98,2473.31414271539)
--(axis cs:99,2513.6653782727)
--(axis cs:99,3221.53694981699)
--(axis cs:99,3221.53694981699)
--(axis cs:98,3184.71804342879)
--(axis cs:97,3189.66597951631)
--(axis cs:96,3209.83499129837)
--(axis cs:95,3317.07711874712)
--(axis cs:94,3176.97448434784)
--(axis cs:93,3216.80622189977)
--(axis cs:92,3168.81625202038)
--(axis cs:91,3184.91907560677)
--(axis cs:90,3253.6123434849)
--(axis cs:89,3375.63339199235)
--(axis cs:88,3181.21637497473)
--(axis cs:87,3207.47647877773)
--(axis cs:86,3224.4766169445)
--(axis cs:85,3204.7550550864)
--(axis cs:84,3194.34661589731)
--(axis cs:83,3232.65047107524)
--(axis cs:82,3271.89220746852)
--(axis cs:81,3244.21294520523)
--(axis cs:80,3208.39646552367)
--(axis cs:79,3196.43355826459)
--(axis cs:78,3248.35416668677)
--(axis cs:77,3280.6198985141)
--(axis cs:76,3203.73005701075)
--(axis cs:75,3285.90213891098)
--(axis cs:74,3239.29399070977)
--(axis cs:73,3131.02641678002)
--(axis cs:72,3227.37707950039)
--(axis cs:71,3150.60076704849)
--(axis cs:70,3250.50321191133)
--(axis cs:69,3180.42451904836)
--(axis cs:68,3188.09362539488)
--(axis cs:67,3217.93469216416)
--(axis cs:66,3189.73353074069)
--(axis cs:65,3129.51641931564)
--(axis cs:64,3229.84401660248)
--(axis cs:63,3231.43169309052)
--(axis cs:62,3145.04080185275)
--(axis cs:61,3121.62269228865)
--(axis cs:60,3070.7809706456)
--(axis cs:59,3205.59648265553)
--(axis cs:58,3064.82032982143)
--(axis cs:57,3051.79881920029)
--(axis cs:56,3206.55267116328)
--(axis cs:55,3029.06346073863)
--(axis cs:54,3096.98581057134)
--(axis cs:53,3128.84480884083)
--(axis cs:52,3022.74009136249)
--(axis cs:51,2981.07058640489)
--(axis cs:50,3021.31907101621)
--(axis cs:49,3176.66697484766)
--(axis cs:48,3003.90489610123)
--(axis cs:47,2970.79287797205)
--(axis cs:46,3051.98036980727)
--(axis cs:45,2916.19617023855)
--(axis cs:44,2873.93849297767)
--(axis cs:43,2912.9038613463)
--(axis cs:42,2850.03514180071)
--(axis cs:41,2809.19371740796)
--(axis cs:40,2992.45391983507)
--(axis cs:39,3029.96133697401)
--(axis cs:38,2923.58696986791)
--(axis cs:37,2834.90511928096)
--(axis cs:36,2782.42813094045)
--(axis cs:35,2771.7731311323)
--(axis cs:34,2806.12133914927)
--(axis cs:33,2771.44205149594)
--(axis cs:32,2891.13302464769)
--(axis cs:31,2891.59002432329)
--(axis cs:30,2906.44592846546)
--(axis cs:29,2828.20369951782)
--(axis cs:28,2940.11871047011)
--(axis cs:27,2905.27225244642)
--(axis cs:26,2886.47797235389)
--(axis cs:25,2721.81263030643)
--(axis cs:24,2849.4049384426)
--(axis cs:23,2665.23728939558)
--(axis cs:22,2603.89676730021)
--(axis cs:21,2528.89859323194)
--(axis cs:20,2595.98761901145)
--(axis cs:19,2314.25018381773)
--(axis cs:18,2219.11206618943)
--(axis cs:17,2132.44219423515)
--(axis cs:16,2076.71563217095)
--(axis cs:15,1903.63297198649)
--(axis cs:14,1731.26921818665)
--(axis cs:13,1775.04737898674)
--(axis cs:12,1821.59540971259)
--(axis cs:11,1793.97653758392)
--(axis cs:10,1419.26149042061)
--(axis cs:9,1199.68101796773)
--(axis cs:8,993.821478694019)
--(axis cs:7,851.761684302326)
--(axis cs:6,700.859773918295)
--(axis cs:5,580.927466221848)
--(axis cs:4,507.901468283821)
--(axis cs:3,432.861903587485)
--(axis cs:2,384.614636554726)
--(axis cs:1,331.973209326178)
--(axis cs:0,247.292806366722)
--cycle;

\path [draw=color0, fill=color0, opacity=0.25]
(axis cs:0,227.722030457742)
--(axis cs:0,209.875531967062)
--(axis cs:1,285.817220029576)
--(axis cs:2,349.202178791827)
--(axis cs:3,394.305448929235)
--(axis cs:4,422.898833446958)
--(axis cs:5,464.540622575707)
--(axis cs:6,513.525563063535)
--(axis cs:7,598.009525047612)
--(axis cs:8,670.902187782591)
--(axis cs:9,757.491237615401)
--(axis cs:10,882.265397462867)
--(axis cs:11,817.946921001042)
--(axis cs:12,866.180515553458)
--(axis cs:13,936.520177886154)
--(axis cs:14,1043.92329734243)
--(axis cs:15,1001.58523592059)
--(axis cs:16,921.606786540488)
--(axis cs:17,988.143973178714)
--(axis cs:18,965.250956221494)
--(axis cs:19,1139.78700938749)
--(axis cs:20,1117.35872715405)
--(axis cs:21,1066.07213451816)
--(axis cs:22,1124.87815271425)
--(axis cs:23,1091.34641065521)
--(axis cs:24,1384.83697599485)
--(axis cs:25,1454.91748557668)
--(axis cs:26,1525.6089386872)
--(axis cs:27,1401.40803006108)
--(axis cs:28,1467.42479772689)
--(axis cs:29,1382.71713281626)
--(axis cs:30,1736.58079349737)
--(axis cs:31,1714.6702522448)
--(axis cs:32,1614.53159026973)
--(axis cs:33,1654.72799496946)
--(axis cs:34,1910.17651634152)
--(axis cs:35,1856.82823537613)
--(axis cs:36,1711.94060669706)
--(axis cs:37,1685.89468415066)
--(axis cs:38,1828.14466279384)
--(axis cs:39,1757.42627021753)
--(axis cs:40,1889.37481643322)
--(axis cs:41,1974.09961609983)
--(axis cs:42,1844.50260420196)
--(axis cs:43,1984.568087169)
--(axis cs:44,1866.32338512515)
--(axis cs:45,1877.96838161711)
--(axis cs:46,1867.17233915005)
--(axis cs:47,1917.69712771838)
--(axis cs:48,2079.43453796887)
--(axis cs:49,2118.10633177856)
--(axis cs:50,1949.441498118)
--(axis cs:51,2160.24116303145)
--(axis cs:52,2304.51623524009)
--(axis cs:53,2345.92013705215)
--(axis cs:54,2239.11996142609)
--(axis cs:55,2407.49625853624)
--(axis cs:56,2458.95440936618)
--(axis cs:57,2443.00760003822)
--(axis cs:58,2512.02501879617)
--(axis cs:59,2398.1225803406)
--(axis cs:60,2355.12386116597)
--(axis cs:61,2530.07477038071)
--(axis cs:62,2388.79802882584)
--(axis cs:63,2500.58224747075)
--(axis cs:64,2431.50299756701)
--(axis cs:65,2330.38838288474)
--(axis cs:66,2572.09423208636)
--(axis cs:67,2551.47961271629)
--(axis cs:68,2494.37042439063)
--(axis cs:69,2546.73699768322)
--(axis cs:70,2611.74097917574)
--(axis cs:71,2598.09251931321)
--(axis cs:72,2479.55463487707)
--(axis cs:73,2570.31408262122)
--(axis cs:74,2521.74410083289)
--(axis cs:75,2605.92009315991)
--(axis cs:76,2548.06212527614)
--(axis cs:77,2566.75360369484)
--(axis cs:78,2477.27017755365)
--(axis cs:79,2576.4554472569)
--(axis cs:80,2533.76161723854)
--(axis cs:81,2525.64736970629)
--(axis cs:82,2633.69040740173)
--(axis cs:83,2579.02318245879)
--(axis cs:84,2711.17952521325)
--(axis cs:85,2612.99982160372)
--(axis cs:86,2522.52041236093)
--(axis cs:87,2726.54143043797)
--(axis cs:88,2601.57848436782)
--(axis cs:89,2777.41284382442)
--(axis cs:90,2573.78146452686)
--(axis cs:91,2679.58288720098)
--(axis cs:92,2703.79370457214)
--(axis cs:93,2638.18628412671)
--(axis cs:94,2652.38152196842)
--(axis cs:95,2628.7620890711)
--(axis cs:96,2588.32931440127)
--(axis cs:97,2623.1934094996)
--(axis cs:98,2788.95523333105)
--(axis cs:99,2671.40677742093)
--(axis cs:99,3262.43275819683)
--(axis cs:99,3262.43275819683)
--(axis cs:98,3244.04690998451)
--(axis cs:97,3239.2665744976)
--(axis cs:96,3104.34548664752)
--(axis cs:95,3181.1229740625)
--(axis cs:94,3210.97405750689)
--(axis cs:93,3195.36425715939)
--(axis cs:92,3204.56070067402)
--(axis cs:91,3206.42748522484)
--(axis cs:90,3160.29076287402)
--(axis cs:89,3242.33865944129)
--(axis cs:88,3186.85220688904)
--(axis cs:87,3253.50737081588)
--(axis cs:86,3123.59828332119)
--(axis cs:85,3231.54575495129)
--(axis cs:84,3224.87294381501)
--(axis cs:83,3198.71024867113)
--(axis cs:82,3223.22362518655)
--(axis cs:81,3244.77502235475)
--(axis cs:80,3186.30473096954)
--(axis cs:79,3210.23120365877)
--(axis cs:78,3146.0584650786)
--(axis cs:77,3158.87170304111)
--(axis cs:76,3196.80565575705)
--(axis cs:75,3205.55097922861)
--(axis cs:74,3179.97621678295)
--(axis cs:73,3175.7667059508)
--(axis cs:72,3165.70246618233)
--(axis cs:71,3229.56445473828)
--(axis cs:70,3218.78527230191)
--(axis cs:69,3182.20922808048)
--(axis cs:68,3189.7765770761)
--(axis cs:67,3178.26141943306)
--(axis cs:66,3148.7625586046)
--(axis cs:65,3024.80681338309)
--(axis cs:64,3041.96443816972)
--(axis cs:63,3136.25061408492)
--(axis cs:62,3052.78956626012)
--(axis cs:61,3173.4307510517)
--(axis cs:60,2948.70264370373)
--(axis cs:59,3058.99885563853)
--(axis cs:58,3121.34253619156)
--(axis cs:57,3114.67526422422)
--(axis cs:56,3104.55858107123)
--(axis cs:55,3018.75582097635)
--(axis cs:54,2902.91936109346)
--(axis cs:53,2992.64162871832)
--(axis cs:52,2956.18759642073)
--(axis cs:51,2865.15516382055)
--(axis cs:50,2766.20581784133)
--(axis cs:49,2875.28122673456)
--(axis cs:48,2796.89046324474)
--(axis cs:47,2760.66150739818)
--(axis cs:46,2668.05527517456)
--(axis cs:45,2728.65810010659)
--(axis cs:44,2740.19621197951)
--(axis cs:43,2682.88961312506)
--(axis cs:42,2657.84185842841)
--(axis cs:41,2656.17582139675)
--(axis cs:40,2742.60238431836)
--(axis cs:39,2571.49905141297)
--(axis cs:38,2593.69158742054)
--(axis cs:37,2584.23453113049)
--(axis cs:36,2617.28214600017)
--(axis cs:35,2654.89651229823)
--(axis cs:34,2667.19112939005)
--(axis cs:33,2574.63086839973)
--(axis cs:32,2483.60973872216)
--(axis cs:31,2601.03351702995)
--(axis cs:30,2506.15608899502)
--(axis cs:29,2343.07384864521)
--(axis cs:28,2304.96787663098)
--(axis cs:27,2292.15529402097)
--(axis cs:26,2372.59437756157)
--(axis cs:25,2256.44267255073)
--(axis cs:24,2141.83541533023)
--(axis cs:23,1897.27085405685)
--(axis cs:22,1942.09737202474)
--(axis cs:21,1870.59736208017)
--(axis cs:20,1949.51217337427)
--(axis cs:19,2015.69552839496)
--(axis cs:18,1831.76116560222)
--(axis cs:17,1871.83627175198)
--(axis cs:16,1772.13231377581)
--(axis cs:15,1849.50312928504)
--(axis cs:14,1879.07831867185)
--(axis cs:13,1733.82447574992)
--(axis cs:12,1668.7043186726)
--(axis cs:11,1385.33987470105)
--(axis cs:10,1629.44943066627)
--(axis cs:9,1438.06148081955)
--(axis cs:8,1251.43382683279)
--(axis cs:7,1049.58879930422)
--(axis cs:6,847.607067586556)
--(axis cs:5,959.71509260495)
--(axis cs:4,698.967249802097)
--(axis cs:3,586.503830781468)
--(axis cs:2,415.781116026208)
--(axis cs:1,317.268968263353)
--(axis cs:0,227.722030457742)
--cycle;

\path [draw=color1, fill=color1, opacity=0.25]
(axis cs:0,239.09266210003)
--(axis cs:0,211.720531107257)
--(axis cs:1,287.637279012059)
--(axis cs:2,334.642054786259)
--(axis cs:3,372.162944830398)
--(axis cs:4,418.447764430014)
--(axis cs:5,435.839734039336)
--(axis cs:6,479.192828672625)
--(axis cs:7,526.637222219673)
--(axis cs:8,543.081293567677)
--(axis cs:9,618.310630202345)
--(axis cs:10,797.245104229667)
--(axis cs:11,980.778765444125)
--(axis cs:12,1098.165335251)
--(axis cs:13,991.869076093764)
--(axis cs:14,1081.77505367119)
--(axis cs:15,1342.01824394845)
--(axis cs:16,1406.36346110326)
--(axis cs:17,1427.13918106317)
--(axis cs:18,1389.9393280227)
--(axis cs:19,1581.41161049708)
--(axis cs:20,1657.60279355574)
--(axis cs:21,1576.44175615081)
--(axis cs:22,1729.10832379026)
--(axis cs:23,1810.99194700438)
--(axis cs:24,1978.30361973865)
--(axis cs:25,1865.71023954859)
--(axis cs:26,1817.9954509546)
--(axis cs:27,2111.63945895673)
--(axis cs:28,1888.11028914223)
--(axis cs:29,1935.8784872313)
--(axis cs:30,1977.90665487893)
--(axis cs:31,1993.3294095408)
--(axis cs:32,2135.19873713945)
--(axis cs:33,2159.41845903121)
--(axis cs:34,2173.77720660126)
--(axis cs:35,2188.92140189748)
--(axis cs:36,2141.40873340782)
--(axis cs:37,2209.44955789347)
--(axis cs:38,2097.18072104825)
--(axis cs:39,2406.05984837792)
--(axis cs:40,2494.63373970567)
--(axis cs:41,2449.4739695906)
--(axis cs:42,2392.84739691753)
--(axis cs:43,2393.1886770232)
--(axis cs:44,2361.19662817384)
--(axis cs:45,2439.15847712378)
--(axis cs:46,2425.89516642974)
--(axis cs:47,2464.56817951875)
--(axis cs:48,2528.9562839915)
--(axis cs:49,2477.46070804649)
--(axis cs:50,2598.51835642234)
--(axis cs:51,2531.25849216414)
--(axis cs:52,2690.44165007084)
--(axis cs:53,2411.05417180105)
--(axis cs:54,2580.1609528475)
--(axis cs:55,2514.26114430025)
--(axis cs:56,2588.50087100087)
--(axis cs:57,2647.3413531079)
--(axis cs:58,2515.48693110951)
--(axis cs:59,2670.86389997132)
--(axis cs:60,2760.07413431976)
--(axis cs:61,2585.39610863304)
--(axis cs:62,2398.35052574413)
--(axis cs:63,2796.10547285042)
--(axis cs:64,2596.89207263726)
--(axis cs:65,2769.56484989331)
--(axis cs:66,2660.3646932997)
--(axis cs:67,2790.16346640832)
--(axis cs:68,2727.94005727807)
--(axis cs:69,2896.49534051971)
--(axis cs:70,2806.05072753525)
--(axis cs:71,2847.58549443975)
--(axis cs:72,3023.39878434854)
--(axis cs:73,2894.9231370435)
--(axis cs:74,2965.44217063797)
--(axis cs:75,2995.52367110154)
--(axis cs:76,2766.75533539316)
--(axis cs:77,2922.6815557512)
--(axis cs:78,2982.75019035706)
--(axis cs:79,3087.04598359034)
--(axis cs:80,3108.95527112317)
--(axis cs:81,3196.44528603243)
--(axis cs:82,3115.13331392565)
--(axis cs:83,3143.99937832873)
--(axis cs:84,3020.53439306395)
--(axis cs:85,3091.33734169104)
--(axis cs:86,3089.76815658002)
--(axis cs:87,3113.89564014901)
--(axis cs:88,3115.80952794156)
--(axis cs:89,2973.32370394859)
--(axis cs:90,3109.59175743788)
--(axis cs:91,3142.2473048684)
--(axis cs:92,2988.57333933115)
--(axis cs:93,3072.51155246756)
--(axis cs:94,3159.30681344022)
--(axis cs:95,3177.89819541437)
--(axis cs:96,3193.63475561346)
--(axis cs:97,3128.35421541333)
--(axis cs:98,3139.43028162319)
--(axis cs:99,3111.90443053608)
--(axis cs:99,3287.49353730031)
--(axis cs:99,3287.49353730031)
--(axis cs:98,3327.84957690821)
--(axis cs:97,3324.62445205457)
--(axis cs:96,3357.61555592821)
--(axis cs:95,3328.17974879271)
--(axis cs:94,3344.4546871775)
--(axis cs:93,3303.8543237281)
--(axis cs:92,3228.53245141335)
--(axis cs:91,3311.81254564671)
--(axis cs:90,3366.16945842366)
--(axis cs:89,3307.18212698024)
--(axis cs:88,3330.96337803601)
--(axis cs:87,3320.99908935197)
--(axis cs:86,3333.64796739213)
--(axis cs:85,3276.75765533987)
--(axis cs:84,3236.04539084219)
--(axis cs:83,3325.93296533914)
--(axis cs:82,3283.36058684801)
--(axis cs:81,3331.95752679497)
--(axis cs:80,3304.85806553695)
--(axis cs:79,3242.85247446411)
--(axis cs:78,3257.53798160108)
--(axis cs:77,3205.29143158972)
--(axis cs:76,3126.74522599327)
--(axis cs:75,3279.45899412813)
--(axis cs:74,3209.20043566889)
--(axis cs:73,3165.46232927992)
--(axis cs:72,3241.0476359858)
--(axis cs:71,3172.73519387983)
--(axis cs:70,3147.00302334799)
--(axis cs:69,3256.6068847937)
--(axis cs:68,3108.15383320702)
--(axis cs:67,3253.64337564737)
--(axis cs:66,3052.95280911013)
--(axis cs:65,3193.40625258198)
--(axis cs:64,3057.657786328)
--(axis cs:63,3202.58789803314)
--(axis cs:62,2930.6412677413)
--(axis cs:61,3023.00119504152)
--(axis cs:60,3145.11784428244)
--(axis cs:59,3148.73957985087)
--(axis cs:58,2984.92586789844)
--(axis cs:57,2987.97735521034)
--(axis cs:56,3056.99898104822)
--(axis cs:55,3021.79231315163)
--(axis cs:54,3022.79645586167)
--(axis cs:53,2942.01081761216)
--(axis cs:52,3125.20259187213)
--(axis cs:51,2924.39961731912)
--(axis cs:50,3079.49618101176)
--(axis cs:49,2975.17167984573)
--(axis cs:48,3014.02962508797)
--(axis cs:47,2912.08579841114)
--(axis cs:46,2995.15666033538)
--(axis cs:45,2982.19454970696)
--(axis cs:44,2862.19377051531)
--(axis cs:43,3013.15653619479)
--(axis cs:42,2861.1642369508)
--(axis cs:41,2887.65503021575)
--(axis cs:40,2952.38928167055)
--(axis cs:39,2947.08968740958)
--(axis cs:38,2738.54901729278)
--(axis cs:37,2858.42806744588)
--(axis cs:36,2859.38351627542)
--(axis cs:35,2758.41971809244)
--(axis cs:34,2760.21967973224)
--(axis cs:33,2678.70415386779)
--(axis cs:32,2706.86395311616)
--(axis cs:31,2601.17039447672)
--(axis cs:30,2797.91632130486)
--(axis cs:29,2698.19870293566)
--(axis cs:28,2542.6849816822)
--(axis cs:27,2622.4490153346)
--(axis cs:26,2410.23652197373)
--(axis cs:25,2524.42909273524)
--(axis cs:24,2546.31171384082)
--(axis cs:23,2522.09798433952)
--(axis cs:22,2382.82296920253)
--(axis cs:21,2296.11487817591)
--(axis cs:20,2411.5340446659)
--(axis cs:19,2282.24511049195)
--(axis cs:18,2220.68877801328)
--(axis cs:17,2298.20315645154)
--(axis cs:16,2246.33235969528)
--(axis cs:15,2142.23953041674)
--(axis cs:14,1772.33795501925)
--(axis cs:13,1683.42381037656)
--(axis cs:12,1838.33163011171)
--(axis cs:11,1704.56913736543)
--(axis cs:10,1543.59908882033)
--(axis cs:9,1121.44623139938)
--(axis cs:8,911.304274412792)
--(axis cs:7,853.957352062102)
--(axis cs:6,965.560160487932)
--(axis cs:5,697.537629643489)
--(axis cs:4,679.78903050133)
--(axis cs:3,456.213581064)
--(axis cs:2,408.430430327867)
--(axis cs:1,316.502410320029)
--(axis cs:0,239.09266210003)
--cycle;
\path [draw=red, fill=red, opacity=0.25]
(axis cs:0,414.960931896271)
--(axis cs:0,255.368195611738)
--(axis cs:1,371.530516726261)
--(axis cs:2,409.56282308926)
--(axis cs:3,532.830908090564)
--(axis cs:4,670.686136472664)
--(axis cs:5,701.847970984499)
--(axis cs:6,823.317458174824)
--(axis cs:7,1150.46014829899)
--(axis cs:8,1209.0526771647)
--(axis cs:9,1150.1877735414)
--(axis cs:10,1313.64826683316)
--(axis cs:11,1406.19926455223)
--(axis cs:12,1201.56734361096)
--(axis cs:13,1339.08776255258)
--(axis cs:14,1432.28544596489)
--(axis cs:15,1543.62286334763)
--(axis cs:16,1595.05532663128)
--(axis cs:17,1424.35923842654)
--(axis cs:18,1631.09829717765)
--(axis cs:19,1886.70214800109)
--(axis cs:20,1931.20125716027)
--(axis cs:21,2068.81801393757)
--(axis cs:22,2169.5169296662)
--(axis cs:23,2014.68264923807)
--(axis cs:24,2209.69604351612)
--(axis cs:25,2392.76198637985)
--(axis cs:26,2437.55938852516)
--(axis cs:27,2368.55098754624)
--(axis cs:28,2338.98435369071)
--(axis cs:29,2440.23164794098)
--(axis cs:30,2421.20531753678)
--(axis cs:31,2591.94596271528)
--(axis cs:32,2442.4981335194)
--(axis cs:33,2650.09287389109)
--(axis cs:34,2540.97834973099)
--(axis cs:35,2360.97884851333)
--(axis cs:36,2566.79292999357)
--(axis cs:37,2642.78454048845)
--(axis cs:38,2575.39644545829)
--(axis cs:39,2796.80417735143)
--(axis cs:40,2733.06699885285)
--(axis cs:41,2815.45527550751)
--(axis cs:42,2661.81356200334)
--(axis cs:43,2710.66696116148)
--(axis cs:44,2633.82143586758)
--(axis cs:45,2755.41960151371)
--(axis cs:46,2674.22006554035)
--(axis cs:47,2655.0115169996)
--(axis cs:48,2702.70283425263)
--(axis cs:49,2722.80350347822)
--(axis cs:50,2723.66614056103)
--(axis cs:51,2744.57684441126)
--(axis cs:52,2731.48069755514)
--(axis cs:53,2924.90947934048)
--(axis cs:54,2680.0441649202)
--(axis cs:55,2749.96527065252)
--(axis cs:56,2633.95702182742)
--(axis cs:57,2792.69137395854)
--(axis cs:58,2719.13388864369)
--(axis cs:59,2921.51992948738)
--(axis cs:60,2818.30065219265)
--(axis cs:61,2720.33143017281)
--(axis cs:62,2872.58459177209)
--(axis cs:63,2776.12741204842)
--(axis cs:64,2910.66190252356)
--(axis cs:65,2870.22100865738)
--(axis cs:66,2926.51479574147)
--(axis cs:67,2912.26822129203)
--(axis cs:68,2847.40885802907)
--(axis cs:69,2872.17056417855)
--(axis cs:70,3020.52188025672)
--(axis cs:71,2917.51167835975)
--(axis cs:72,2813.78386076515)
--(axis cs:73,3037.13654047266)
--(axis cs:74,2858.01701090461)
--(axis cs:75,2868.76751012797)
--(axis cs:76,2807.19328478741)
--(axis cs:77,2925.56356817864)
--(axis cs:78,2622.30712667243)
--(axis cs:79,2641.68007078214)
--(axis cs:80,2921.20371491419)
--(axis cs:81,2997.12796161488)
--(axis cs:82,2967.10069447376)
--(axis cs:83,3039.64771972314)
--(axis cs:84,2974.09205106653)
--(axis cs:85,3099.84674351378)
--(axis cs:86,2998.5595856643)
--(axis cs:87,3023.34891395588)
--(axis cs:88,3100.49511097631)
--(axis cs:89,3088.99603700932)
--(axis cs:90,3194.30696114204)
--(axis cs:91,3045.23601916304)
--(axis cs:92,3228.88729502462)
--(axis cs:93,3227.07702824117)
--(axis cs:94,3133.99081719719)
--(axis cs:95,3143.24224566744)
--(axis cs:96,2963.98018179996)
--(axis cs:97,3219.61845284484)
--(axis cs:98,3234.80579822636)
--(axis cs:99,3170.76999019491)
--(axis cs:99,3382.01857432417)
--(axis cs:99,3382.01857432417)
--(axis cs:98,3462.35209354511)
--(axis cs:97,3395.12315734327)
--(axis cs:96,3285.05188155153)
--(axis cs:95,3391.08553770697)
--(axis cs:94,3409.00900739667)
--(axis cs:93,3380.40113129361)
--(axis cs:92,3384.5934734768)
--(axis cs:91,3397.83354610159)
--(axis cs:90,3432.12420489471)
--(axis cs:89,3372.98357076467)
--(axis cs:88,3299.96734256025)
--(axis cs:87,3359.45810873487)
--(axis cs:86,3343.86406520054)
--(axis cs:85,3355.29928675113)
--(axis cs:84,3315.92535855601)
--(axis cs:83,3291.01990668173)
--(axis cs:82,3337.03880743479)
--(axis cs:81,3265.04124413329)
--(axis cs:80,3252.79749652512)
--(axis cs:79,3244.63896957041)
--(axis cs:78,3174.25844335139)
--(axis cs:77,3322.17849344261)
--(axis cs:76,3284.8109812605)
--(axis cs:75,3300.42150900682)
--(axis cs:74,3372.89902190093)
--(axis cs:73,3368.32507138297)
--(axis cs:72,3247.03523798441)
--(axis cs:71,3300.31639432067)
--(axis cs:70,3323.61181774623)
--(axis cs:69,3325.94426400546)
--(axis cs:68,3250.37258313601)
--(axis cs:67,3344.19828553624)
--(axis cs:66,3319.87010373735)
--(axis cs:65,3246.2672087892)
--(axis cs:64,3253.37146413811)
--(axis cs:63,3249.44508600649)
--(axis cs:62,3203.39078773472)
--(axis cs:61,3223.54138219147)
--(axis cs:60,3286.75565631085)
--(axis cs:59,3234.13543493003)
--(axis cs:58,3257.79548981464)
--(axis cs:57,3284.05539459466)
--(axis cs:56,3170.9565127109)
--(axis cs:55,3259.07634335792)
--(axis cs:54,3152.94573864134)
--(axis cs:53,3265.06046935624)
--(axis cs:52,3243.61583405131)
--(axis cs:51,3116.89269839838)
--(axis cs:50,3221.4809732385)
--(axis cs:49,3133.73117169156)
--(axis cs:48,3180.190880476)
--(axis cs:47,3219.47377239806)
--(axis cs:46,3245.59182174113)
--(axis cs:45,3262.92420951796)
--(axis cs:44,3105.3741598451)
--(axis cs:43,3196.10575830401)
--(axis cs:42,3083.03205689995)
--(axis cs:41,3060.40251648945)
--(axis cs:40,3150.81342875935)
--(axis cs:39,3085.49141467924)
--(axis cs:38,2990.28513672156)
--(axis cs:37,3042.11984769543)
--(axis cs:36,2948.59125253432)
--(axis cs:35,2930.33877744475)
--(axis cs:34,3001.53346074345)
--(axis cs:33,3114.67191685633)
--(axis cs:32,2740.03873949114)
--(axis cs:31,3143.37190307802)
--(axis cs:30,3025.90481573691)
--(axis cs:29,2943.50286296719)
--(axis cs:28,2929.28203077952)
--(axis cs:27,2868.48773274209)
--(axis cs:26,3095.44309276771)
--(axis cs:25,2980.06937404856)
--(axis cs:24,2829.84180064681)
--(axis cs:23,2736.69584629947)
--(axis cs:22,2713.14581764637)
--(axis cs:21,2614.77839744265)
--(axis cs:20,2513.57949121086)
--(axis cs:19,2379.21468767762)
--(axis cs:18,1964.49404184352)
--(axis cs:17,1588.52199379132)
--(axis cs:16,2059.0012353387)
--(axis cs:15,2150.86427703633)
--(axis cs:14,1772.77198723148)
--(axis cs:13,1538.23118068975)
--(axis cs:12,1412.8170546426)
--(axis cs:11,1892.63367865085)
--(axis cs:10,1863.33028645175)
--(axis cs:9,1732.01554413383)
--(axis cs:8,1542.8045766738)
--(axis cs:7,1429.00289344061)
--(axis cs:6,1103.0320441481)
--(axis cs:5,986.44270970071)
--(axis cs:4,929.335630493556)
--(axis cs:3,737.289195270255)
--(axis cs:2,658.112283136451)
--(axis cs:1,566.841115033216)
--(axis cs:0,414.960931896271)
--cycle;

\addplot [very thick, blue]
table {%
0 237.931735389702
1 322.472970139589
2 451.273857151336
3 505.079849554591
4 570.425084130764
5 628.637809572201
6 725.24353945383
7 846.639415610278
8 869.772684513567
9 959.370205882819
10 983.349312953241
11 942.387282934777
12 997.508140789156
13 1097.17321338088
14 1233.27454394199
15 1395.62197819305
16 1350.79958424902
17 1372.78844878541
18 1505.72969057938
19 1436.58589259917
20 1500.77013672647
21 1668.59870084191
22 1763.07392988038
23 1853.40572810253
24 1856.31644224164
25 1887.5752769376
26 1924.00864575028
27 2072.90404595479
28 2198.78900522001
29 2028.27038178958
30 2100.08844763848
31 2049.14771749023
32 2154.68684962112
33 2033.26313590358
34 2224.02532882855
35 2163.50805644894
36 2306.14360366328
37 2212.82492602401
38 2214.41426470155
39 2185.0884033527
40 2154.1933421289
41 2226.25622634528
42 2252.03301951749
43 2232.48339857207
44 2263.71850794979
45 2302.03044963316
46 2496.42461483153
47 2398.01286490388
48 2428.20996906347
49 2498.58591926702
50 2571.73428793628
51 2554.58511422419
52 2563.76496407432
53 2520.52565744671
54 2554.25649481901
55 2686.04363361908
56 2893.61247273056
57 2736.67217784616
58 2781.62913016377
59 2605.28180555761
60 2631.06676904569
61 2569.23275319845
62 2732.3052870731
63 2797.39492179242
64 2818.85521198851
65 2781.76785634758
66 2879.38103452981
67 2904.80248684655
68 2782.41522955475
69 3004.49752385751
70 2992.46203984535
71 2985.39209592864
72 3064.12687735567
73 2921.3082110199
74 2960.36990372417
75 3058.7674102103
76 3056.9004461463
77 2985.7530827191
78 2995.75309129124
79 2992.54638033736
80 3034.98783876428
81 3049.98646676303
82 3052.61587896769
83 3037.88057331666
84 3000.61855912443
85 3099.09285333379
86 3055.07907264819
87 2995.48270573561
88 2985.40021808764
89 3060.92037308721
90 3048.30627653162
91 3038.82785709955
92 3018.95756997906
93 3004.53548024181
94 2980.12143942958
95 3037.42810535295
96 3132.60013854931
97 3022.79275523882
98 3066.445517045
99 3066.24571179109
};
\addlegendentry{Without Transfer}
\addplot [very thick, color2]
table {%
0 234.131992483107
1 321.03606869544
2 362.893113954523
3 402.459689792302
4 460.400164037066
5 525.542996067457
6 620.399708480875
7 739.729680579394
8 845.81267741215
9 1016.64772504558
10 1152.02445887606
11 1410.59833868811
12 1457.92165738605
13 1443.11436471447
14 1427.17461608674
15 1549.07744821334
16 1735.4731645947
17 1746.92866527449
18 1915.79435763171
19 2026.00383492007
20 2311.87201391001
21 2214.56443716503
22 2277.82778384258
23 2316.47434973149
24 2523.40245803995
25 2441.82369607579
26 2579.70104378925
27 2613.46583486945
28 2571.07551825965
29 2468.07355400795
30 2530.35883085227
31 2483.97371758371
32 2495.72276500618
33 2336.06180560902
34 2480.91169159269
35 2377.98576044651
36 2406.08023225498
37 2447.25795727892
38 2601.75876440289
39 2631.63075055075
40 2655.88983066491
41 2447.90222622939
42 2486.01273895177
43 2530.21266791302
44 2476.63118133501
45 2566.26107759926
46 2650.20828845563
47 2592.48721444321
48 2653.09073717754
49 2813.85620596764
50 2642.07844796042
51 2629.68148078459
52 2685.52920641512
53 2791.48323948469
54 2783.89756674614
55 2693.22478578334
56 2857.47419664587
57 2685.49000510649
58 2756.2843417912
59 2833.52699674755
60 2701.03985530282
61 2728.39990983533
62 2744.93664611523
63 2901.36670049718
64 2893.16415778573
65 2814.26065296079
66 2811.31957666009
67 2884.83369624286
68 2832.78962296807
69 2858.30533740213
70 2917.12884005286
71 2812.01591140243
72 2899.41279995548
73 2779.11516298136
74 2933.92423824017
75 2928.66223406315
76 2899.40303585319
77 2961.61787962997
78 2932.70125330824
79 2884.21610232663
80 2881.68761358798
81 2944.5799419522
82 2948.47598164486
83 2917.4917810153
84 2898.39569203452
85 2868.65765613216
86 2861.88684685487
87 2881.45887880091
88 2861.39309624531
89 3046.68676263232
90 2923.40730225364
91 2794.3786203762
92 2853.90181511419
93 2878.39087533363
94 2777.00886631139
95 2974.90480413077
96 2839.29128876053
97 2809.74102139608
98 2859.38727518158
99 2884.70425232548
};
\addlegendentry{Auxiliary Tasks}
\addplot [very thick, color0]
table {%
0 219.061340605247
1 301.163900047175
2 381.091829892861
3 479.395580511302
4 545.904404265766
5 680.840144306536
6 667.074765761007
7 796.782798266542
8 931.045208698225
9 1061.8239371643
10 1225.89192115794
11 1082.72523371428
12 1245.92139332148
13 1316.52284143777
14 1425.52266877328
15 1399.77768132533
16 1330.94308579582
17 1414.53060043639
18 1405.35069734869
19 1573.06699955376
20 1508.43752068834
21 1463.19186377138
22 1523.13852895479
23 1471.99791385333
24 1744.09208038033
25 1843.32124149335
26 1946.03393668919
27 1830.72836409248
28 1900.09075937541
29 1852.94700065138
30 2122.4604520419
31 2168.97950339756
32 2034.9665627439
33 2099.77437642036
34 2305.36855226521
35 2256.0558044488
36 2166.09776534871
37 2129.21651758185
38 2198.81457018414
39 2186.2272940857
40 2334.97536749386
41 2302.29790078905
42 2247.76943802953
43 2356.68173094857
44 2294.60581711507
45 2272.04698939923
46 2288.65136561724
47 2353.07059808272
48 2446.74590539288
49 2494.91022413668
50 2375.76691377442
51 2514.68692742383
52 2640.58887904298
53 2688.99563352395
54 2586.48472251313
55 2728.99133927662
56 2809.78526298145
57 2799.22531092001
58 2832.34720380956
59 2758.95172343033
60 2674.5548179528
61 2878.67495654019
62 2744.35219520253
63 2848.14499877838
64 2747.56265968513
65 2694.19985860334
66 2903.97406997612
67 2880.32888268003
68 2892.67314625031
69 2910.10525174955
70 2955.02507176583
71 2949.12217325709
72 2844.60106059576
73 2883.71796105156
74 2864.13677702506
75 2955.67036358164
76 2903.60233685099
77 2903.25695703632
78 2843.37304291522
79 2932.10572772133
80 2884.10644721066
81 2904.93837047896
82 2948.29613810578
83 2916.81726937771
84 2999.11238487957
85 2955.03089417616
86 2849.41947920067
87 3019.55305374855
88 2909.09618813515
89 3027.99380573669
90 2903.58472325313
91 2964.58486301538
92 2978.3784619618
93 2945.39160393618
94 2945.89119555586
95 2910.44253235164
96 2872.57417553864
97 2960.89556188767
98 3036.44935664425
99 2968.64532098467
};
\addlegendentry{Transfer $\hat{P}$}
\addplot [very thick, color1]
table {%
0 224.88124430751
1 300.67158774339
2 368.97247681908
3 409.392903926151
4 535.515786520612
5 555.043459466319
6 697.294320744766
7 675.25697220769
8 710.333642676005
9 835.418468549251
10 1134.07670722544
11 1322.13181133533
12 1465.49746185823
13 1322.79459195396
14 1418.47405061225
15 1731.90590474655
16 1827.73364911847
17 1849.44331343709
18 1803.35370870958
19 1931.65931981278
20 2039.29160354118
21 1933.8216462452
22 2068.38341160675
23 2186.07271504586
24 2274.5339996618
25 2213.04808162659
26 2113.41674153801
27 2377.84557674891
28 2232.1200604377
29 2345.9325268317
30 2402.3269218903
31 2324.14043470024
32 2434.09031593115
33 2425.20353049857
34 2458.14004402774
35 2475.65872157754
36 2511.80884802901
37 2551.38534483768
38 2444.57521010675
39 2691.87680851711
40 2746.59161029606
41 2686.43986476855
42 2635.96002165703
43 2729.74109427041
44 2627.80177462997
45 2733.05348758267
46 2737.69130886922
47 2719.98586884941
48 2787.13723625528
49 2747.99070394097
50 2861.01391334039
51 2736.88057744714
52 2902.20817067659
53 2681.58907255367
54 2825.25299975661
55 2778.84543545421
56 2814.73171465196
57 2827.64085443902
58 2765.05574547895
59 2934.17513754038
60 2962.7344267099
61 2821.2332990438
62 2672.3639034413
63 3018.41105621633
64 2833.5809775853
65 3012.79789078225
66 2855.32100842442
67 3031.82966587855
68 2928.47080936891
69 3090.34440138194
70 2987.80858391832
71 3019.64100092997
72 3139.55597107092
73 3031.33387327019
74 3096.8871918555
75 3143.27759575238
76 2958.84673128968
77 3070.42539359731
78 3135.47706065902
79 3166.65292617933
80 3212.53816535319
81 3268.87675218375
82 3208.28423090578
83 3246.29964186056
84 3129.42723042841
85 3184.01980919098
86 3222.37697885185
87 3222.99795156143
88 3224.93220348156
89 3145.24576673875
90 3248.74275340061
91 3231.12413216367
92 3115.24960185555
93 3199.43106194588
94 3254.55790120365
95 3254.66998383566
96 3282.59466508604
97 3238.99476350254
98 3238.86818414724
99 3201.03408556986
};
\addlegendentry{Transfer $\hat{R}$}
\addplot [ultra thick, red]
table {%
0 332.677352552458
1 467.319734873716
2 526.060874258156
3 636.7039999177
4 804.261470769805
5 841.177053864242
6 972.232062618837
7 1288.63164132268
8 1366.57834454135
9 1434.1617104641
10 1569.46200488611
11 1639.18977493724
12 1305.33697592067
13 1440.88705573537
14 1605.92951936684
15 1852.65328221935
16 1820.65616054676
17 1502.67067995127
18 1793.80272896999
19 2108.88585873651
20 2227.56243088185
21 2334.44793398653
22 2445.52252562663
23 2387.08832504077
24 2536.53869285706
25 2700.36344864166
26 2773.82394379951
27 2625.42294765864
28 2627.45120314464
29 2684.09898696947
30 2723.60534151739
31 2871.71721769156
32 2604.2535786867
33 2880.32844755928
34 2776.97127060736
35 2674.83625564422
36 2770.42959732551
37 2867.55696183805
38 2802.12049389381
39 2944.1978416564
40 2954.4153460078
41 2929.25227481245
42 2870.59356756033
43 2974.74622725944
44 2879.65446769359
45 3017.54900387429
46 2968.28211539878
47 2966.29090364757
48 2948.26129455336
49 2928.80989494904
50 3001.37263505464
51 2927.69279660813
52 3004.14053384568
53 3114.82644556787
54 2933.78528884063
55 3018.60800216694
56 2902.2826308235
57 3055.2760742693
58 3010.9930603373
59 3068.79954130156
60 3073.18649009099
61 2996.72989294648
62 3053.7945867926
63 3025.59090693938
64 3083.12724147487
65 3060.21830749619
66 3125.58477032726
67 3138.43954559623
68 3072.31962537609
69 3118.39565111841
70 3173.35704155799
71 3113.87860064804
72 3049.91517299243
73 3204.29458223754
74 3150.19523780355
75 3089.17596495787
76 3062.10499496882
77 3129.79838345799
78 2907.69753609973
79 2952.7141011197
80 3095.48261276413
81 3136.64346484784
82 3157.37724474446
83 3174.96920989332
84 3146.10811201738
85 3234.29276762242
86 3177.22378178472
87 3198.18105574482
88 3200.08042002844
89 3243.1164512246
90 3316.85857732913
91 3238.34078566591
92 3309.4291828963
93 3308.8086323021
94 3278.0235747053
95 3281.48077269216
96 3130.14681930886
97 3303.55475276252
98 3351.66584455183
99 3280.34269480174
};
\addlegendentry{Our Method}
\end{axis}

\end{tikzpicture}

%% file: figs/walker2d_ablation.tex
\begin{tikzpicture}[scale=0.8]

\definecolor{color0}{rgb}{1,0.549019607843137,0}
\definecolor{color1}{rgb}{0.55,0.27,0.07}
\definecolor{color2}{rgb}{0.580392156862745,0.403921568627451,0.741176470588235}
\begin{axis}[
legend cell align={left},
legend style={
  fill opacity=0.8,
  draw opacity=1,
  text opacity=1,
  at={(0.97,0.03)},
  anchor=south east,
  draw=white!80!black
},
tick align=outside,
tick pos=left,
title={{Walker2d}},
x grid style={white!69.0196078431373!black},
xlabel={{Steps}},
xmajorgrids,
xmin=-4.95, xmax=103.95,
xtick style={color=black},
xtick={0,20,40,60,80,100},
xticklabels={0k,200k,400k,600k,800k,1000k},
y grid style={white!69.0196078431373!black},
ymajorgrids,
ylabel={{Return}},
ymin=115.808622308108, ymax=4283.96550845008,
ytick style={color=black}
]

\path [draw=color2, fill=color2, opacity=0.25]
(axis cs:0,370.375706469879)
--(axis cs:0,342.173368712645)
--(axis cs:1,462.493904964554)
--(axis cs:2,567.763321910152)
--(axis cs:3,808.302838079512)
--(axis cs:4,1089.86375496446)
--(axis cs:5,1117.82074532575)
--(axis cs:6,1283.65583272598)
--(axis cs:7,1625.98438667063)
--(axis cs:8,1942.45901285881)
--(axis cs:9,2019.8363617525)
--(axis cs:10,2027.81750798964)
--(axis cs:11,2134.04908171328)
--(axis cs:12,2138.66475923353)
--(axis cs:13,2314.39054285271)
--(axis cs:14,2430.38076721936)
--(axis cs:15,2504.66888628962)
--(axis cs:16,2603.29102780053)
--(axis cs:17,2798.72887494413)
--(axis cs:18,2867.93862451974)
--(axis cs:19,2798.9353273827)
--(axis cs:20,2861.09718099884)
--(axis cs:21,2862.73974988716)
--(axis cs:22,3028.71807335413)
--(axis cs:23,2996.73572051634)
--(axis cs:24,2945.49823769387)
--(axis cs:25,3174.69939669545)
--(axis cs:26,3188.72732409485)
--(axis cs:27,3299.60312049873)
--(axis cs:28,3318.25701153442)
--(axis cs:29,3367.22389293066)
--(axis cs:30,3348.85737689708)
--(axis cs:31,3374.98190578393)
--(axis cs:32,3338.39710447003)
--(axis cs:33,3294.86090693508)
--(axis cs:34,3315.76498877397)
--(axis cs:35,3344.01407516874)
--(axis cs:36,3341.42475146667)
--(axis cs:37,3379.31724188496)
--(axis cs:38,3359.24215805595)
--(axis cs:39,3334.8623696127)
--(axis cs:40,3447.3170068277)
--(axis cs:41,3482.09987230918)
--(axis cs:42,3396.48703908898)
--(axis cs:43,3501.60654154344)
--(axis cs:44,3454.37716229063)
--(axis cs:45,3484.45361525781)
--(axis cs:46,3493.87476061183)
--(axis cs:47,3418.14975220421)
--(axis cs:48,3525.84929567807)
--(axis cs:49,3555.80451660028)
--(axis cs:50,3446.28362360832)
--(axis cs:51,3548.66301666026)
--(axis cs:52,3598.07381830421)
--(axis cs:53,3404.63803328268)
--(axis cs:54,3447.30508683823)
--(axis cs:55,3449.65589451937)
--(axis cs:56,3535.40333887785)
--(axis cs:57,3416.11625829072)
--(axis cs:58,3553.46732849232)
--(axis cs:59,3491.87395009043)
--(axis cs:60,3417.59077759741)
--(axis cs:61,3455.69096862899)
--(axis cs:62,3353.07664466514)
--(axis cs:63,3503.94509946346)
--(axis cs:64,3373.12205891938)
--(axis cs:65,3469.29699995993)
--(axis cs:66,3499.37118905065)
--(axis cs:67,3454.38180354336)
--(axis cs:68,3538.55050766347)
--(axis cs:69,3489.66723235562)
--(axis cs:70,3333.14982143965)
--(axis cs:71,3595.38997444552)
--(axis cs:72,3431.65495823163)
--(axis cs:73,3532.10040880755)
--(axis cs:74,3532.13368287241)
--(axis cs:75,3395.21339879768)
--(axis cs:76,3416.2659327915)
--(axis cs:77,3435.28728914476)
--(axis cs:78,3384.10203226492)
--(axis cs:79,3403.25507463606)
--(axis cs:80,3379.96633837191)
--(axis cs:81,3356.44804517931)
--(axis cs:82,3521.56561111427)
--(axis cs:83,3481.55792838004)
--(axis cs:84,3498.37972027638)
--(axis cs:85,3348.16759229853)
--(axis cs:86,3496.51972803828)
--(axis cs:87,3386.65845178629)
--(axis cs:88,3322.81066580138)
--(axis cs:89,3377.71938248222)
--(axis cs:90,3410.24812308713)
--(axis cs:91,3337.8312744384)
--(axis cs:92,3350.01209086102)
--(axis cs:93,3427.09457979101)
--(axis cs:94,3397.42182301287)
--(axis cs:95,3457.35282874762)
--(axis cs:96,3443.014264011)
--(axis cs:97,3255.58761216463)
--(axis cs:98,3402.31203903516)
--(axis cs:99,3551.42348707499)
--(axis cs:99,3991.59285906436)
--(axis cs:99,3991.59285906436)
--(axis cs:98,3922.50371724036)
--(axis cs:97,3895.60758761068)
--(axis cs:96,3909.73541027615)
--(axis cs:95,3944.54092955533)
--(axis cs:94,3915.08488422152)
--(axis cs:93,3888.25848480454)
--(axis cs:92,3897.85239946243)
--(axis cs:91,3870.07349514349)
--(axis cs:90,3945.82036743784)
--(axis cs:89,3878.4694909515)
--(axis cs:88,3966.21306224804)
--(axis cs:87,3909.67670579407)
--(axis cs:86,3941.1545530123)
--(axis cs:85,3872.95116690488)
--(axis cs:84,3905.67717315887)
--(axis cs:83,3944.28954588994)
--(axis cs:82,3950.04471255965)
--(axis cs:81,3911.91612348903)
--(axis cs:80,3833.30290538064)
--(axis cs:79,3952.35886424045)
--(axis cs:78,3931.47283933706)
--(axis cs:77,3928.59434527728)
--(axis cs:76,3915.41034166115)
--(axis cs:75,3813.80906921131)
--(axis cs:74,3930.98191971403)
--(axis cs:73,3948.69156034768)
--(axis cs:72,3892.42890470976)
--(axis cs:71,3946.47090397837)
--(axis cs:70,3913.85187817861)
--(axis cs:69,3884.56962093558)
--(axis cs:68,3890.15115635557)
--(axis cs:67,3873.84322756943)
--(axis cs:66,3904.18044044327)
--(axis cs:65,3924.43133125279)
--(axis cs:64,3798.06479310637)
--(axis cs:63,3917.60538596156)
--(axis cs:62,3789.46628027542)
--(axis cs:61,3907.37623289664)
--(axis cs:60,3880.72606522257)
--(axis cs:59,3855.1715439021)
--(axis cs:58,3835.54074709961)
--(axis cs:57,3843.15587889968)
--(axis cs:56,3845.49372882243)
--(axis cs:55,3909.06736845202)
--(axis cs:54,3814.75199258378)
--(axis cs:53,3826.12918710761)
--(axis cs:52,3893.87963816858)
--(axis cs:51,3855.05221764679)
--(axis cs:50,3862.40025503683)
--(axis cs:49,3800.22754714129)
--(axis cs:48,3853.17342027581)
--(axis cs:47,3806.32550047094)
--(axis cs:46,3756.09252717373)
--(axis cs:45,3797.23026069552)
--(axis cs:44,3787.60267682417)
--(axis cs:43,3739.80661250088)
--(axis cs:42,3717.14325672657)
--(axis cs:41,3753.32444381442)
--(axis cs:40,3636.3443885278)
--(axis cs:39,3664.26225235737)
--(axis cs:38,3704.89345934524)
--(axis cs:37,3641.76183045561)
--(axis cs:36,3677.08626961294)
--(axis cs:35,3645.7023437611)
--(axis cs:34,3592.79732976273)
--(axis cs:33,3541.2638290813)
--(axis cs:32,3493.30838946763)
--(axis cs:31,3565.84411914016)
--(axis cs:30,3575.9053979094)
--(axis cs:29,3544.76888152162)
--(axis cs:28,3555.73325723988)
--(axis cs:27,3532.96465023261)
--(axis cs:26,3449.92365144134)
--(axis cs:25,3426.28346225615)
--(axis cs:24,3204.03974607287)
--(axis cs:23,3240.58962974916)
--(axis cs:22,3285.96166679443)
--(axis cs:21,3161.83033990988)
--(axis cs:20,3099.86124710482)
--(axis cs:19,3129.71106321335)
--(axis cs:18,3081.4652805466)
--(axis cs:17,3026.2286253186)
--(axis cs:16,3020.48931986398)
--(axis cs:15,2988.45325380023)
--(axis cs:14,2863.3212833067)
--(axis cs:13,2720.63878542292)
--(axis cs:12,2581.81087185593)
--(axis cs:11,2575.57970436256)
--(axis cs:10,2474.02440326987)
--(axis cs:9,2411.17098800992)
--(axis cs:8,2282.2798553823)
--(axis cs:7,2093.74688361311)
--(axis cs:6,1888.17273740721)
--(axis cs:5,1690.71040749618)
--(axis cs:4,1666.15628659667)
--(axis cs:3,1206.27742982685)
--(axis cs:2,678.452142974261)
--(axis cs:1,513.306303849133)
--(axis cs:0,370.375706469879)
--cycle;

\path [draw=color0, fill=color0, opacity=0.25]
(axis cs:0,367.856351844657)
--(axis cs:0,331.7895250355)
--(axis cs:1,542.749390564342)
--(axis cs:2,606.082189314856)
--(axis cs:3,765.680882108329)
--(axis cs:4,739.473775594633)
--(axis cs:5,934.945118186632)
--(axis cs:6,1227.77732343697)
--(axis cs:7,1323.49246722229)
--(axis cs:8,1449.48972938041)
--(axis cs:9,1388.02286792268)
--(axis cs:10,1539.59390793494)
--(axis cs:11,1535.90829342556)
--(axis cs:12,1633.89670415609)
--(axis cs:13,1748.12431558777)
--(axis cs:14,1815.14391503377)
--(axis cs:15,1936.74617636667)
--(axis cs:16,2065.98852865256)
--(axis cs:17,2200.8488039621)
--(axis cs:18,2185.23452101243)
--(axis cs:19,2403.35326641465)
--(axis cs:20,2329.98686137017)
--(axis cs:21,2519.7401999646)
--(axis cs:22,2637.04596981122)
--(axis cs:23,2557.53546128018)
--(axis cs:24,2555.00047032547)
--(axis cs:25,2680.36823678592)
--(axis cs:26,2802.87111423073)
--(axis cs:27,2697.96673955377)
--(axis cs:28,2711.39290152001)
--(axis cs:29,2844.09155735304)
--(axis cs:30,2879.76164004849)
--(axis cs:31,3011.55513939494)
--(axis cs:32,2854.71904693517)
--(axis cs:33,2983.26766226702)
--(axis cs:34,3043.13635692258)
--(axis cs:35,2895.72456778264)
--(axis cs:36,2968.80519850959)
--(axis cs:37,3171.00324900028)
--(axis cs:38,3030.19983269196)
--(axis cs:39,3342.795455005)
--(axis cs:40,3155.91372183805)
--(axis cs:41,3023.5085856405)
--(axis cs:42,3163.66759203856)
--(axis cs:43,3354.67982706243)
--(axis cs:44,3064.69467140258)
--(axis cs:45,3403.99876581115)
--(axis cs:46,3214.02286334236)
--(axis cs:47,3278.3091359412)
--(axis cs:48,3250.515802417)
--(axis cs:49,3390.54001034724)
--(axis cs:50,3344.26403859633)
--(axis cs:51,3357.37141472994)
--(axis cs:52,3523.17760028126)
--(axis cs:53,3473.12817140661)
--(axis cs:54,3440.7024694207)
--(axis cs:55,3551.87945136366)
--(axis cs:56,3541.47472851054)
--(axis cs:57,3503.33545989816)
--(axis cs:58,3487.25132674486)
--(axis cs:59,3562.63180271737)
--(axis cs:60,3537.37776956138)
--(axis cs:61,3538.04958107584)
--(axis cs:62,3630.5308767835)
--(axis cs:63,3589.12340291787)
--(axis cs:64,3515.11321481038)
--(axis cs:65,3603.90255458416)
--(axis cs:66,3435.58973744496)
--(axis cs:67,3500.30258546558)
--(axis cs:68,3421.77050898191)
--(axis cs:69,3514.75924974354)
--(axis cs:70,3474.04551708519)
--(axis cs:71,3571.14641045148)
--(axis cs:72,3463.20902583544)
--(axis cs:73,3470.32501927436)
--(axis cs:74,3655.22532247621)
--(axis cs:75,3512.11509734924)
--(axis cs:76,3467.87770227842)
--(axis cs:77,3439.27834784486)
--(axis cs:78,3523.84120227783)
--(axis cs:79,3562.26239513997)
--(axis cs:80,3363.86576569324)
--(axis cs:81,3506.01897477167)
--(axis cs:82,3444.74052443256)
--(axis cs:83,3545.45030117038)
--(axis cs:84,3498.25301003705)
--(axis cs:85,3649.54119325088)
--(axis cs:86,3585.87602503182)
--(axis cs:87,3422.34677737147)
--(axis cs:88,3570.2922288671)
--(axis cs:89,3639.79705122712)
--(axis cs:90,3564.93990205729)
--(axis cs:91,3659.53093738392)
--(axis cs:92,3615.20662868987)
--(axis cs:93,3553.34504085755)
--(axis cs:94,3698.35614124499)
--(axis cs:95,3695.63608905316)
--(axis cs:96,3786.93825029959)
--(axis cs:97,3738.41640245298)
--(axis cs:98,3762.04833625575)
--(axis cs:99,3777.65881120799)
--(axis cs:99,3963.85100467454)
--(axis cs:99,3963.85100467454)
--(axis cs:98,3969.49938231468)
--(axis cs:97,3995.29783060159)
--(axis cs:96,3986.71579775184)
--(axis cs:95,3954.51884428005)
--(axis cs:94,3978.44240989225)
--(axis cs:93,3950.78197875521)
--(axis cs:92,3953.44460757998)
--(axis cs:91,3956.87060810095)
--(axis cs:90,3943.44058813081)
--(axis cs:89,3941.24567665442)
--(axis cs:88,3933.16085505911)
--(axis cs:87,3939.9619813746)
--(axis cs:86,3934.19839974407)
--(axis cs:85,3940.78829566094)
--(axis cs:84,3910.77634343594)
--(axis cs:83,3921.74837563452)
--(axis cs:82,3904.33084269199)
--(axis cs:81,3919.0403852433)
--(axis cs:80,3926.59586086718)
--(axis cs:79,3930.76186765801)
--(axis cs:78,3902.17527545988)
--(axis cs:77,3873.72759856422)
--(axis cs:76,3905.15382480952)
--(axis cs:75,3907.84402842145)
--(axis cs:74,3894.15898783792)
--(axis cs:73,3921.77371099833)
--(axis cs:72,3923.34722170849)
--(axis cs:71,3936.5462491727)
--(axis cs:70,3893.33003190487)
--(axis cs:69,3886.87435268664)
--(axis cs:68,3838.03936576636)
--(axis cs:67,3852.32258770516)
--(axis cs:66,3880.41354024706)
--(axis cs:65,3882.2067014306)
--(axis cs:64,3910.30194514937)
--(axis cs:63,3894.98404138551)
--(axis cs:62,3879.12993274511)
--(axis cs:61,3930.74254850122)
--(axis cs:60,3773.65869645193)
--(axis cs:59,3828.24780697308)
--(axis cs:58,3812.37080497638)
--(axis cs:57,3885.50064803991)
--(axis cs:56,3817.39191808463)
--(axis cs:55,3827.03098123976)
--(axis cs:54,3845.18184686881)
--(axis cs:53,3810.49901209327)
--(axis cs:52,3790.60443878978)
--(axis cs:51,3784.93587323426)
--(axis cs:50,3812.90933999158)
--(axis cs:49,3728.7414007127)
--(axis cs:48,3621.76931351139)
--(axis cs:47,3677.73450739621)
--(axis cs:46,3657.4741726251)
--(axis cs:45,3664.94767042935)
--(axis cs:44,3624.49811674399)
--(axis cs:43,3697.18818852939)
--(axis cs:42,3644.75476683708)
--(axis cs:41,3571.53128075957)
--(axis cs:40,3524.82569456544)
--(axis cs:39,3582.22758823641)
--(axis cs:38,3522.24580457439)
--(axis cs:37,3528.41726464333)
--(axis cs:36,3447.26266300422)
--(axis cs:35,3427.54849374586)
--(axis cs:34,3456.28573039965)
--(axis cs:33,3450.0097682493)
--(axis cs:32,3382.05380952829)
--(axis cs:31,3385.24677153073)
--(axis cs:30,3347.84476086246)
--(axis cs:29,3295.20780419495)
--(axis cs:28,3256.90515666412)
--(axis cs:27,3217.59138375875)
--(axis cs:26,3280.08719671192)
--(axis cs:25,3196.74488426592)
--(axis cs:24,3083.72757727848)
--(axis cs:23,3078.79634199645)
--(axis cs:22,3153.59984501238)
--(axis cs:21,3020.35243049585)
--(axis cs:20,2867.35321876242)
--(axis cs:19,2852.06841912104)
--(axis cs:18,2712.89192949541)
--(axis cs:17,2654.19625754519)
--(axis cs:16,2581.6463746823)
--(axis cs:15,2471.76074787688)
--(axis cs:14,2282.70745813486)
--(axis cs:13,2170.4961120927)
--(axis cs:12,2148.97358886023)
--(axis cs:11,2146.08730636037)
--(axis cs:10,1931.23947121497)
--(axis cs:9,1535.37271125176)
--(axis cs:8,1837.95586289953)
--(axis cs:7,1503.64983788043)
--(axis cs:6,1482.94122336036)
--(axis cs:5,1274.79752362669)
--(axis cs:4,991.034662723162)
--(axis cs:3,947.81732026937)
--(axis cs:2,861.803551641503)
--(axis cs:1,600.238556602336)
--(axis cs:0,367.856351844657)
--cycle;

\path [draw=color1, fill=color1, opacity=0.25]
(axis cs:0,384.487725901216)
--(axis cs:0,349.262693091479)
--(axis cs:1,491.638093860051)
--(axis cs:2,576.229846551738)
--(axis cs:3,648.955526798417)
--(axis cs:4,1171.86833251252)
--(axis cs:5,1164.56572405792)
--(axis cs:6,994.617881080697)
--(axis cs:7,1350.49271719613)
--(axis cs:8,1406.10209462812)
--(axis cs:9,1483.92128276943)
--(axis cs:10,1555.65791158369)
--(axis cs:11,1599.25063312176)
--(axis cs:12,1704.95719758964)
--(axis cs:13,1939.46044852897)
--(axis cs:14,1946.08751545455)
--(axis cs:15,1902.53132395036)
--(axis cs:16,2148.73786896812)
--(axis cs:17,2252.23667163637)
--(axis cs:18,2322.66643230518)
--(axis cs:19,2474.02725904349)
--(axis cs:20,2435.05281508411)
--(axis cs:21,2437.1410729301)
--(axis cs:22,2672.34660626664)
--(axis cs:23,2677.80475655762)
--(axis cs:24,2672.58462189355)
--(axis cs:25,2671.79357742089)
--(axis cs:26,2664.85951336713)
--(axis cs:27,2856.05226360745)
--(axis cs:28,2690.04852601567)
--(axis cs:29,2822.16609543725)
--(axis cs:30,2740.39984007289)
--(axis cs:31,2916.85434699396)
--(axis cs:32,2844.81999816954)
--(axis cs:33,2888.0232668731)
--(axis cs:34,2938.56327218846)
--(axis cs:35,3016.75490195306)
--(axis cs:36,3071.89136641257)
--(axis cs:37,3131.60464832291)
--(axis cs:38,3041.20744644538)
--(axis cs:39,3044.83580590062)
--(axis cs:40,3124.14121273659)
--(axis cs:41,3090.63278578131)
--(axis cs:42,3179.99655985625)
--(axis cs:43,3266.72034171356)
--(axis cs:44,3242.19002760564)
--(axis cs:45,3261.41261380786)
--(axis cs:46,3243.67917066495)
--(axis cs:47,3359.8412796876)
--(axis cs:48,3292.91878318438)
--(axis cs:49,3201.12241666796)
--(axis cs:50,3306.50967967612)
--(axis cs:51,3300.52637606659)
--(axis cs:52,3392.08508084944)
--(axis cs:53,3266.94217887651)
--(axis cs:54,3372.85076584166)
--(axis cs:55,3431.4435510105)
--(axis cs:56,3396.09164820547)
--(axis cs:57,3506.38685128661)
--(axis cs:58,3394.51655900642)
--(axis cs:59,3517.76343486616)
--(axis cs:60,3557.14478614956)
--(axis cs:61,3438.19272761633)
--(axis cs:62,3587.9830999042)
--(axis cs:63,3557.68142199972)
--(axis cs:64,3611.52972376149)
--(axis cs:65,3550.73120964701)
--(axis cs:66,3494.54416039157)
--(axis cs:67,3490.84845319256)
--(axis cs:68,3448.0675194065)
--(axis cs:69,3578.35972400253)
--(axis cs:70,3610.01950744789)
--(axis cs:71,3538.29966196876)
--(axis cs:72,3634.75474045207)
--(axis cs:73,3659.41738579462)
--(axis cs:74,3588.10552976163)
--(axis cs:75,3605.56651176904)
--(axis cs:76,3322.7528837024)
--(axis cs:77,3538.19909065727)
--(axis cs:78,3644.72359020858)
--(axis cs:79,3513.56561580097)
--(axis cs:80,3655.09463365858)
--(axis cs:81,3619.02097746319)
--(axis cs:82,3584.88364220171)
--(axis cs:83,3496.24211321349)
--(axis cs:84,3562.93657373327)
--(axis cs:85,3654.58936834051)
--(axis cs:86,3502.53788154698)
--(axis cs:87,3552.94405713209)
--(axis cs:88,3504.71246100757)
--(axis cs:89,3527.33383563834)
--(axis cs:90,3517.64276496953)
--(axis cs:91,3560.80396112477)
--(axis cs:92,3556.10784819381)
--(axis cs:93,3486.65972531057)
--(axis cs:94,3417.46766854692)
--(axis cs:95,3551.87480738684)
--(axis cs:96,3555.70020222342)
--(axis cs:97,3651.76450390132)
--(axis cs:98,3546.69568834868)
--(axis cs:99,3532.57990316061)
--(axis cs:99,3896.42993899305)
--(axis cs:99,3896.42993899305)
--(axis cs:98,3915.69433442594)
--(axis cs:97,3926.70536737998)
--(axis cs:96,3866.52343325955)
--(axis cs:95,3887.48248953416)
--(axis cs:94,3854.54769616803)
--(axis cs:93,3895.4905798975)
--(axis cs:92,3914.72063632228)
--(axis cs:91,3908.55453530984)
--(axis cs:90,3889.45047412929)
--(axis cs:89,3911.74180848471)
--(axis cs:88,3870.26736111725)
--(axis cs:87,3919.72675550692)
--(axis cs:86,3867.14369832714)
--(axis cs:85,3899.86223507278)
--(axis cs:84,3841.07964615196)
--(axis cs:83,3882.56673196854)
--(axis cs:82,3863.42941241875)
--(axis cs:81,3911.21235947546)
--(axis cs:80,3924.92632813463)
--(axis cs:79,3824.64302902067)
--(axis cs:78,3902.31597054257)
--(axis cs:77,3895.44322796199)
--(axis cs:76,3830.86717107863)
--(axis cs:75,3899.94333830604)
--(axis cs:74,3848.38180419919)
--(axis cs:73,3896.03244512739)
--(axis cs:72,3903.95371773781)
--(axis cs:71,3900.27126165423)
--(axis cs:70,3845.52327367146)
--(axis cs:69,3913.69547687514)
--(axis cs:68,3814.21251608101)
--(axis cs:67,3781.67391359021)
--(axis cs:66,3811.58292325813)
--(axis cs:65,3843.76873987044)
--(axis cs:64,3824.82938047989)
--(axis cs:63,3837.13350786105)
--(axis cs:62,3855.63484544816)
--(axis cs:61,3749.91032879795)
--(axis cs:60,3800.34472114423)
--(axis cs:59,3791.63406352863)
--(axis cs:58,3744.70033051974)
--(axis cs:57,3762.34244083795)
--(axis cs:56,3771.78095293694)
--(axis cs:55,3777.13385506671)
--(axis cs:54,3726.27233571461)
--(axis cs:53,3614.38587422835)
--(axis cs:52,3699.75515543258)
--(axis cs:51,3705.00311157444)
--(axis cs:50,3651.51267330569)
--(axis cs:49,3630.56222655236)
--(axis cs:48,3616.81280112556)
--(axis cs:47,3706.01842907628)
--(axis cs:46,3658.49700533447)
--(axis cs:45,3645.5586413327)
--(axis cs:44,3623.33754975534)
--(axis cs:43,3641.57833359002)
--(axis cs:42,3591.46828534213)
--(axis cs:41,3521.78433907192)
--(axis cs:40,3574.78516798828)
--(axis cs:39,3481.32515015187)
--(axis cs:38,3486.14491661857)
--(axis cs:37,3561.78611203791)
--(axis cs:36,3458.42189320984)
--(axis cs:35,3359.01841089155)
--(axis cs:34,3436.92888545621)
--(axis cs:33,3384.96217967241)
--(axis cs:32,3363.56141339884)
--(axis cs:31,3459.07338864447)
--(axis cs:30,3196.82792674707)
--(axis cs:29,3253.47932709641)
--(axis cs:28,3251.74237669036)
--(axis cs:27,3347.44206477354)
--(axis cs:26,3218.17182710655)
--(axis cs:25,3287.77244645031)
--(axis cs:24,3199.28379440167)
--(axis cs:23,3193.0450379031)
--(axis cs:22,3204.42777774483)
--(axis cs:21,3124.37702856815)
--(axis cs:20,3017.12073815743)
--(axis cs:19,3046.49258482822)
--(axis cs:18,2872.93672987232)
--(axis cs:17,2831.69372290441)
--(axis cs:16,2804.05161179734)
--(axis cs:15,2461.77563511273)
--(axis cs:14,2609.8660111885)
--(axis cs:13,2604.25234193571)
--(axis cs:12,2316.63806422931)
--(axis cs:11,2220.56776885352)
--(axis cs:10,2050.75243191151)
--(axis cs:9,1816.48664903717)
--(axis cs:8,1791.46312867217)
--(axis cs:7,1742.98218414771)
--(axis cs:6,1406.71594975329)
--(axis cs:5,1489.60527681403)
--(axis cs:4,1365.71795123707)
--(axis cs:3,804.091125274206)
--(axis cs:2,762.151728266336)
--(axis cs:1,611.862393743629)
--(axis cs:0,384.487725901216)
--cycle;
\path [draw=red, fill=red, opacity=0.25]
(axis cs:0,407.533710470678)
--(axis cs:0,305.270298950925)
--(axis cs:1,537.532566200477)
--(axis cs:2,658.555372768738)
--(axis cs:3,944.768609914622)
--(axis cs:4,907.039772626741)
--(axis cs:5,1320.42208426518)
--(axis cs:6,1341.39208711135)
--(axis cs:7,1548.19228135746)
--(axis cs:8,1633.85555078248)
--(axis cs:9,1669.58278034828)
--(axis cs:10,1674.51217992106)
--(axis cs:11,1811.5678847737)
--(axis cs:12,1995.7528165432)
--(axis cs:13,2085.76752501497)
--(axis cs:14,2226.10315310418)
--(axis cs:15,2374.23989895114)
--(axis cs:16,2512.471867885)
--(axis cs:17,2634.42884041164)
--(axis cs:18,2750.94117121601)
--(axis cs:19,2686.36199158847)
--(axis cs:20,2810.09333474585)
--(axis cs:21,2740.49990260902)
--(axis cs:22,2853.47612562619)
--(axis cs:23,2845.02069031175)
--(axis cs:24,2782.01319168169)
--(axis cs:25,2718.88819529182)
--(axis cs:26,2956.51948302525)
--(axis cs:27,2840.40062604605)
--(axis cs:28,3137.95570042929)
--(axis cs:29,3145.88378721824)
--(axis cs:30,3065.50469678833)
--(axis cs:31,3190.69743896604)
--(axis cs:32,3385.08844823865)
--(axis cs:33,3401.53806408227)
--(axis cs:34,3489.19308003916)
--(axis cs:35,3395.07805026134)
--(axis cs:36,3430.44634490144)
--(axis cs:37,3446.638448274)
--(axis cs:38,3184.76465553071)
--(axis cs:39,3530.86059606125)
--(axis cs:40,3399.18400724134)
--(axis cs:41,3407.83478588424)
--(axis cs:42,3507.86713130122)
--(axis cs:43,3562.65084771154)
--(axis cs:44,3564.16108699171)
--(axis cs:45,3583.90921552952)
--(axis cs:46,3431.86540281412)
--(axis cs:47,3476.08872074995)
--(axis cs:48,3548.9733909679)
--(axis cs:49,3650.64280084877)
--(axis cs:50,3617.37194100812)
--(axis cs:51,3624.17268740678)
--(axis cs:52,3437.56898875566)
--(axis cs:53,3599.77523538695)
--(axis cs:54,3608.57216806922)
--(axis cs:55,3458.79617507407)
--(axis cs:56,3661.55564905219)
--(axis cs:57,3602.42811245603)
--(axis cs:58,3646.84518397833)
--(axis cs:59,3686.97885036873)
--(axis cs:60,3678.81387508645)
--(axis cs:61,3690.79691338848)
--(axis cs:62,3712.95556641836)
--(axis cs:63,3674.24342413673)
--(axis cs:64,3704.45447734433)
--(axis cs:65,3725.121310352)
--(axis cs:66,3767.71026515808)
--(axis cs:67,3560.77843711689)
--(axis cs:68,3715.07166425761)
--(axis cs:69,3603.12180897819)
--(axis cs:70,3630.68793467386)
--(axis cs:71,3700.61845954175)
--(axis cs:72,3560.12732371397)
--(axis cs:73,3671.94297242452)
--(axis cs:74,3666.52192767139)
--(axis cs:75,3722.9750825507)
--(axis cs:76,3635.27585189935)
--(axis cs:77,3664.04966900103)
--(axis cs:78,3772.23867325207)
--(axis cs:79,3832.76872698679)
--(axis cs:80,3666.05076812306)
--(axis cs:81,3655.61450543928)
--(axis cs:82,3805.57201923473)
--(axis cs:83,3715.78889329693)
--(axis cs:84,3891.44094003287)
--(axis cs:85,3948.91039684973)
--(axis cs:86,3959.63072474719)
--(axis cs:87,3851.195867319)
--(axis cs:88,3928.0970572451)
--(axis cs:89,3872.71557233753)
--(axis cs:90,3891.65612060109)
--(axis cs:91,3963.42992780049)
--(axis cs:92,3966.01519154662)
--(axis cs:93,3938.44623564186)
--(axis cs:94,3936.17382700384)
--(axis cs:95,3858.52355732583)
--(axis cs:96,3806.41364701511)
--(axis cs:97,3890.93449130692)
--(axis cs:98,3779.81574641154)
--(axis cs:99,3884.28074710913)
--(axis cs:99,4094.50383180726)
--(axis cs:99,4094.50383180726)
--(axis cs:98,4034.23441882702)
--(axis cs:97,4080.10576272014)
--(axis cs:96,4033.6490998368)
--(axis cs:95,4051.74749901585)
--(axis cs:94,4018.70808147539)
--(axis cs:93,4084.2440877742)
--(axis cs:92,4072.02577084672)
--(axis cs:91,4055.53540248013)
--(axis cs:90,4052.17713909033)
--(axis cs:89,4059.47260816945)
--(axis cs:88,4037.10789242985)
--(axis cs:87,4007.2874138841)
--(axis cs:86,4059.06528867849)
--(axis cs:85,4064.50979860463)
--(axis cs:84,4055.00056665817)
--(axis cs:83,3976.80555866988)
--(axis cs:82,4014.21971006857)
--(axis cs:81,4023.56722761994)
--(axis cs:80,3924.73335454392)
--(axis cs:79,3987.57400450726)
--(axis cs:78,4043.27653026183)
--(axis cs:77,3991.14150541446)
--(axis cs:76,3936.026370433)
--(axis cs:75,3992.55172075333)
--(axis cs:74,3984.01939752545)
--(axis cs:73,3964.18344886101)
--(axis cs:72,3965.02728077486)
--(axis cs:71,3935.01659309285)
--(axis cs:70,3923.43932677333)
--(axis cs:69,3927.46331746483)
--(axis cs:68,3995.61999551437)
--(axis cs:67,3926.25475228484)
--(axis cs:66,3962.33842326475)
--(axis cs:65,3970.33315759355)
--(axis cs:64,3951.25012123968)
--(axis cs:63,3916.76978494134)
--(axis cs:62,3934.72130734052)
--(axis cs:61,3917.14255610038)
--(axis cs:60,3884.6141572353)
--(axis cs:59,3916.35203797914)
--(axis cs:58,3861.09134854442)
--(axis cs:57,3888.76602216235)
--(axis cs:56,3873.61383559934)
--(axis cs:55,3777.27400595311)
--(axis cs:54,3865.60197677936)
--(axis cs:53,3855.46824216789)
--(axis cs:52,3729.49204148612)
--(axis cs:51,3836.82973624683)
--(axis cs:50,3828.53598087851)
--(axis cs:49,3881.21708166457)
--(axis cs:48,3753.23500591823)
--(axis cs:47,3647.44160941715)
--(axis cs:46,3640.62105980395)
--(axis cs:45,3770.1011719949)
--(axis cs:44,3777.83094066801)
--(axis cs:43,3743.72805239768)
--(axis cs:42,3770.2853342693)
--(axis cs:41,3633.54293874471)
--(axis cs:40,3651.15681826498)
--(axis cs:39,3757.50337269418)
--(axis cs:38,3604.00350663062)
--(axis cs:37,3664.59441056387)
--(axis cs:36,3689.77732626202)
--(axis cs:35,3598.36548381147)
--(axis cs:34,3722.63275268238)
--(axis cs:33,3701.46183431704)
--(axis cs:32,3686.47592271597)
--(axis cs:31,3575.38955450296)
--(axis cs:30,3491.75907400038)
--(axis cs:29,3490.26698904868)
--(axis cs:28,3543.68939853356)
--(axis cs:27,3364.04984194389)
--(axis cs:26,3403.22704694552)
--(axis cs:25,3186.49185473395)
--(axis cs:24,3178.98179844641)
--(axis cs:23,3260.31181559133)
--(axis cs:22,3226.63667444338)
--(axis cs:21,3141.72143145636)
--(axis cs:20,3174.82198521341)
--(axis cs:19,3017.71506730689)
--(axis cs:18,3116.05256954178)
--(axis cs:17,2957.6050845039)
--(axis cs:16,2894.23597262589)
--(axis cs:15,2732.55279992138)
--(axis cs:14,2602.02701567098)
--(axis cs:13,2356.9238586003)
--(axis cs:12,2267.02353213471)
--(axis cs:11,2235.23547596398)
--(axis cs:10,2022.13444485554)
--(axis cs:9,2073.72645041419)
--(axis cs:8,2001.80241216607)
--(axis cs:7,1777.18929485367)
--(axis cs:6,1646.78859718402)
--(axis cs:5,1674.02688751081)
--(axis cs:4,1293.19381542621)
--(axis cs:3,1208.99026802613)
--(axis cs:2,890.285880116113)
--(axis cs:1,844.927316475017)
--(axis cs:0,407.533710470678)
--cycle;
\addplot [very thick, blue]
table {%
0 348.786031833919
1 604.372654531799
2 965.047647917091
3 1141.69660008647
4 1618.28694840608
5 1596.55599625575
6 1952.32855563454
7 1669.33116436364
8 1908.51585553314
9 1900.95158684342
10 2003.17082548276
11 1920.20538511573
12 2158.11130943077
13 2574.24553846179
14 2768.0200348151
15 2809.54760060012
16 3019.81883820065
17 3040.16163239793
18 3165.75523530485
19 3089.41894932036
20 3020.07927677136
21 3196.32938510838
22 3179.34381523992
23 3201.23464495482
24 3218.10369254355
25 3320.89172537134
26 3262.21003847251
27 3328.78214367164
28 3255.55528762398
29 3266.20650105378
30 3292.7377156064
31 3286.60967093887
32 3310.47651794125
33 3319.26408629819
34 3381.10260887725
35 3427.21149494887
36 3318.83817062111
37 3367.99512185791
38 3364.03648793041
39 3249.15601363706
40 3345.69619432986
41 3438.65858767672
42 3526.53010425513
43 3533.54264765776
44 3535.20129057824
45 3543.55326213085
46 3460.49023333834
47 3569.82529089082
48 3531.85831803365
49 3576.53005790122
50 3512.77564113537
51 3495.10973956346
52 3575.85157529766
53 3579.1370592863
54 3657.94099634905
55 3639.41467911535
56 3575.70489982169
57 3665.80634234821
58 3605.73420160541
59 3562.71200214178
60 3707.05136255598
61 3630.09473767936
62 3597.70462679886
63 3697.21932683403
64 3553.88300570137
65 3702.16975230477
66 3658.5567645764
67 3565.43047857311
68 3622.30341915348
69 3673.54937018337
70 3625.02436670114
71 3638.86181351615
72 3713.39672556402
73 3719.81732209679
74 3710.57799145728
75 3671.88417165961
76 3741.85713274778
77 3754.28131394611
78 3674.68365474079
79 3689.63897085822
80 3749.7544466619
81 3651.1416426323
82 3721.57675631618
83 3704.91964361769
84 3738.35121783512
85 3711.33547133662
86 3781.49638300054
87 3721.35914679416
88 3642.34820588273
89 3748.02911624766
90 3753.77099811448
91 3725.36575573506
92 3720.12000647399
93 3806.74276355086
94 3745.54964887382
95 3727.84257022077
96 3755.68694817698
97 3793.25150246693
98 3760.91346089704
99 3721.92041298397
};
\addlegendentry{Without Transfer}

\addplot [very thick, color2]
table {%
0 354.998441434043
1 487.456722884113
2 626.102133430314
3 993.40045753288
4 1372.01951563316
5 1397.96000282088
6 1573.90160238139
7 1841.35700921823
8 2101.61315310599
9 2199.64924128657
10 2230.90797073997
11 2347.50641898431
12 2356.69280241553
13 2515.90860857775
14 2632.22424211953
15 2750.11966207629
16 2808.69113144779
17 2912.35361524246
18 2975.22211578254
19 2958.61284206689
20 2984.48480085725
21 3014.7413309254
22 3156.74319873248
23 3118.8636417434
24 3067.0334988589
25 3293.17207604071
26 3325.75704442137
27 3415.50112202548
28 3436.7377036597
29 3457.52185029762
30 3464.19217049385
31 3471.3964652664
32 3413.95641898105
33 3425.53935273139
34 3453.41536424215
35 3503.23789344732
36 3510.78302561517
37 3518.73891244199
38 3536.84612097209
39 3503.90422836195
40 3543.97775625686
41 3622.39805250382
42 3572.46891410924
43 3619.59573939181
44 3632.6220656587
45 3645.23807580612
46 3633.42062310596
47 3621.97394927098
48 3700.12703452548
49 3687.30363219434
50 3647.31219972896
51 3705.45555939774
52 3751.56789829024
53 3627.25575398223
54 3643.52206469824
55 3691.94424999499
56 3695.35874685519
57 3654.34838234083
58 3706.41411410129
59 3693.86103658809
60 3663.44286078916
61 3689.74513902695
62 3581.16593034799
63 3721.9190603133
64 3592.31798124848
65 3721.571876799
66 3702.98200443755
67 3684.58063347984
68 3730.82867440441
69 3701.84878212795
70 3654.4903965613
71 3779.52732889492
72 3673.90381752416
73 3758.43801988689
74 3751.78369856736
75 3616.23950755084
76 3676.986796104
77 3698.92386465317
78 3672.25748550819
79 3701.41440974953
80 3629.38762092454
81 3668.21939794884
82 3747.53253766881
83 3729.57885940172
84 3715.40260293304
85 3635.72823056108
86 3727.29902767929
87 3657.42127550056
88 3694.6121173739
89 3633.17755893516
90 3707.90009758295
91 3628.02589859599
92 3633.43765419393
93 3692.22895936264
94 3660.14248110664
95 3711.7033039817
96 3690.53925376232
97 3609.82385276299
98 3682.62394914161
99 3793.36941841728
};
\addlegendentry{Auxiliary Tasks}
\addplot [very thick, color0]
table {%
0 351.225076348745
1 569.656464521901
2 723.807667574152
3 858.418918809424
4 868.146844445558
5 1108.46271280393
6 1354.32613649095
7 1414.99339421318
8 1631.26076620632
9 1461.65246458663
10 1718.11190264894
11 1797.81226539717
12 1861.82266601557
13 1948.62632312377
14 2041.0429959157
15 2207.78753574577
16 2339.34717279421
17 2428.99058638904
18 2450.98588571585
19 2626.50840773999
20 2620.66712422881
21 2767.21527372937
22 2885.97128966208
23 2823.13823858791
24 2811.56383805857
25 2942.28399977312
26 3033.87614227804
27 2949.19379371141
28 3011.06121918644
29 3066.12376654483
30 3126.62263111907
31 3202.0602128438
32 3129.76496149842
33 3231.62652658393
34 3250.16876642129
35 3190.40730935011
36 3215.87033473234
37 3355.52813362476
38 3277.40381237662
39 3459.8001292671
40 3331.20765588148
41 3310.77783843531
42 3407.22291422917
43 3530.13043277758
44 3356.66775873996
45 3538.40026470877
46 3458.3283122696
47 3491.52480617566
48 3450.98997704571
49 3571.3900483483
50 3609.28982931993
51 3585.58069431566
52 3672.73270004548
53 3659.36154749126
54 3660.91167837684
55 3698.81469459137
56 3693.32792848466
57 3711.25176374404
58 3668.77814632654
59 3705.600107538
60 3672.53276536687
61 3749.22721578422
62 3768.9360337868
63 3765.77452411345
64 3721.11158655557
65 3756.8533854911
66 3681.76401414869
67 3692.8276300495
68 3652.79166070664
69 3713.96834977098
70 3705.86442916297
71 3769.8621300368
72 3714.19452851153
73 3717.76474013161
74 3779.84465156032
75 3724.14817211287
76 3709.48112161208
77 3683.47183613859
78 3736.24050711246
79 3765.53709253194
80 3679.53784699045
81 3740.68029773413
82 3711.83252760063
83 3756.79256476068
84 3737.1161079167
85 3810.99898608199
86 3783.56405471031
87 3710.87223881512
88 3780.73321772546
89 3802.5259068328
90 3781.48142106798
91 3823.8959378579
92 3793.6868132087
93 3775.61869840006
94 3847.71051695892
95 3831.28068762949
96 3893.12833371685
97 3866.18695501882
98 3870.52176602996
99 3871.04813466859
};
\addlegendentry{Transfer $\hat{P}$}
\addplot [very thick, color1]
table {%
0 365.336373355671
1 547.341350139598
2 654.810319077989
3 723.508785747305
4 1264.63985415444
5 1339.7964072474
6 1210.96412402802
7 1551.7751815685
8 1589.81480832584
9 1650.24811363773
10 1807.51492981558
11 1892.90670847867
12 2002.9004570512
13 2279.23506161303
14 2280.1483223308
15 2176.34767264894
16 2462.43672588921
17 2544.24572428884
18 2580.93933284141
19 2760.1010678933
20 2729.77558734906
21 2793.16480061559
22 2939.25611084345
23 2927.00956424984
24 2938.88049733465
25 2983.20623853321
26 2935.73695653633
27 3114.24700653549
28 2983.92686912795
29 3039.70833991062
30 2977.04438391621
31 3203.9541824593
32 3124.397826126
33 3148.60729496761
34 3185.87787717116
35 3196.13050592074
36 3263.80188099878
37 3344.07399404707
38 3268.21281051214
39 3267.23760052325
40 3339.05812343802
41 3303.88866418287
42 3387.93130734758
43 3464.55034202526
44 3451.60317126598
45 3453.44592515539
46 3461.88852385934
47 3537.14719158854
48 3464.26472986906
49 3425.90432926427
50 3483.35274061673
51 3517.70297902166
52 3550.18353475372
53 3448.27167058061
54 3546.23020618316
55 3611.25077273753
56 3599.88741082403
57 3643.24623056329
58 3577.31379740546
59 3659.1382705128
60 3691.63574407997
61 3591.01722152251
62 3734.24215077499
63 3706.57027144357
64 3728.33711605663
65 3707.0061454641
66 3666.54184985837
67 3636.53446184749
68 3654.61632131207
69 3760.43541012494
70 3736.8561224868
71 3738.22660832316
72 3778.87946489977
73 3788.93155537624
74 3722.13573694782
75 3770.66038840517
76 3617.38098216948
77 3731.64638163162
78 3780.06446501383
79 3676.99726085066
80 3802.02786190012
81 3779.54055416435
82 3731.49842708838
83 3704.96201333357
84 3709.00101931319
85 3790.99499377687
86 3700.99245277981
87 3746.7896932478
88 3710.21978590707
89 3734.9312935421
90 3729.15741670387
91 3756.20577035302
92 3764.05711436577
93 3702.0838784166
94 3659.50965707824
95 3734.93572146092
96 3714.9061973995
97 3795.58507556963
98 3744.64206985825
99 3733.0901786772
};
\addlegendentry{Transfer $\hat{R}$}
\addplot [ultra thick, red]
table {%
0 352.216536377132
1 668.803874193032
2 767.330976891019
3 1080.28835719827
4 1103.37156435703
5 1497.58734943897
6 1496.54393290749
7 1660.0384738353
8 1802.17997596836
9 1861.51654799798
10 1839.87962857442
11 2004.70091501243
12 2127.23195934199
13 2217.7279442298
14 2416.0409063133
15 2544.50231872939
16 2690.08645394612
17 2793.14551178875
18 2936.15087533605
19 2845.23227013396
20 2991.98059227674
21 2928.27851617842
22 3045.18086266371
23 3054.64980252015
24 2991.74772226943
25 2949.32102142374
26 3193.29223381274
27 3123.5535071955
28 3345.99625024637
29 3322.88181448398
30 3301.89817727821
31 3401.52459854713
32 3536.34733403601
33 3558.2441488628
34 3605.876582258
35 3504.53834872181
36 3575.19499243544
37 3555.5638813397
38 3418.40306346795
39 3652.78910578485
40 3530.37428050205
41 3524.825595245
42 3638.55984102108
43 3658.77951739737
44 3676.25503719853
45 3686.71610129983
46 3540.54339892823
47 3563.88035132998
48 3660.982560073
49 3776.48810509096
50 3727.28700125474
51 3738.69714686046
52 3594.50127071093
53 3736.55218993051
54 3756.09085847934
55 3628.42518013429
56 3775.90844720814
57 3761.57504069011
58 3766.54583356675
59 3809.24539968738
60 3785.68206787608
61 3824.81154549792
62 3827.09251812164
63 3808.78938620425
64 3841.56309126754
65 3867.35653142337
66 3875.51430263319
67 3762.95899324328
68 3868.11617117093
69 3784.91402568404
70 3797.93182028425
71 3826.59712351863
72 3798.60081531006
73 3833.63944346486
74 3836.15406515176
75 3866.35478835839
76 3803.19882782448
77 3844.36884046504
78 3925.57817062743
79 3919.5789011625
80 3807.52602500469
81 3872.99648786921
82 3923.64292448871
83 3857.23807969684
84 3976.51615745677
85 4009.74607755748
86 4010.7696615503
87 3938.22400235581
88 3982.79954975696
89 3972.18767697171
90 3977.78264037173
91 4011.36536878356
92 4023.36611127198
93 4015.93938203033
94 3977.60068088209
95 3963.15603348812
96 3931.55160639234
97 3995.59967511403
98 3927.14447251775
99 4000.74046488771
};
\addlegendentry{Our Method}
\end{axis}

\end{tikzpicture}

%% file: figs/ball_ablation_all.tex
\begin{tikzpicture}[scale=0.8]

\definecolor{color0}{rgb}{0,0,1}
\definecolor{color1}{rgb}{1,0,0}
\definecolor{color2}{rgb}{0.580392156862745,0.403921568627451,0.741176470588235}
\definecolor{color3}{rgb}{1,0.549019607843137,0}
\definecolor{color4}{rgb}{0.55,0.27,0.07}

\begin{axis}[
legend cell align={left},
legend style={
  fill opacity=0.8,
  draw opacity=1,
  text opacity=1,
  at={(0.97,0.03)},
  anchor=south east,
  draw=white!80!black
},
tick align=outside,
tick pos=left,
title={{3DBall}},
scaled ticks=false,
x grid style={white!69.0196078431373!black},
xlabel={{Steps}},
xmajorgrids,
xmin=-14400, xmax=302400,
xtick style={color=black},
xticklabels={0k,50k,100k,150k,200k,250k,300k,350k},
y grid style={white!69.0196078431373!black},
ylabel={{Return}},
ymajorgrids,
ymin=-4.05404311371744, ymax=103.882906501044,
ytick style={color=black},
ytick={-20,0,20,40,60,80,100,120},
yticklabels={\ensuremath{-}20,0,20,40,60,80,100,120}
]

\path [fill=color0, fill opacity=0.2]
(axis cs:0,0.911586044410865)
--(axis cs:0,0.893576828181744)
--(axis cs:12000,0.950052213847637)
--(axis cs:24000,0.987312424639861)
--(axis cs:36000,0.937743202825387)
--(axis cs:48000,0.98111319377025)
--(axis cs:60000,1.38828979444504)
--(axis cs:72000,1.6821152121226)
--(axis cs:84000,2.09489543334643)
--(axis cs:96000,5.51821187400818)
--(axis cs:108000,16.7050844659805)
--(axis cs:120000,28.6504294724464)
--(axis cs:132000,61.839427675883)
--(axis cs:144000,71.1358873888651)
--(axis cs:156000,82.7840106633504)
--(axis cs:168000,79.247117922465)
--(axis cs:180000,77.2332826207479)
--(axis cs:192000,72.8922494913737)
--(axis cs:204000,78.1251781082153)
--(axis cs:216000,81.2324106190999)
--(axis cs:228000,83.7450485204061)
--(axis cs:240000,76.4509918289185)
--(axis cs:252000,80.0072109603882)
--(axis cs:264000,79.5812978897095)
--(axis cs:276000,85.6326447703044)
--(axis cs:288000,87.8364172744751)
--(axis cs:288000,93.5437296015422)
--(axis cs:288000,93.5437296015422)
--(axis cs:276000,93.296606010437)
--(axis cs:264000,87.377580019633)
--(axis cs:252000,90.8259956512451)
--(axis cs:240000,84.8313724250794)
--(axis cs:228000,92.2400962600708)
--(axis cs:216000,89.7994200897217)
--(axis cs:204000,84.5465452957153)
--(axis cs:192000,81.01156300354)
--(axis cs:180000,86.601208946228)
--(axis cs:168000,87.2577630233765)
--(axis cs:156000,90.4309548492432)
--(axis cs:144000,86.4457708218892)
--(axis cs:132000,82.4021010195414)
--(axis cs:120000,54.4784513861338)
--(axis cs:108000,32.1223968280156)
--(axis cs:96000,9.70910165691376)
--(axis cs:84000,2.72736665960153)
--(axis cs:72000,2.23179793878396)
--(axis cs:60000,1.67258901568254)
--(axis cs:48000,1.04662085551023)
--(axis cs:36000,0.964142741501331)
--(axis cs:24000,1.01115789222717)
--(axis cs:12000,0.972143123229345)
--(axis cs:0,0.911586044410865)
--cycle;

\path [fill=color2, fill opacity=0.2]
(axis cs:0,0.905502462764581)
--(axis cs:0,0.885391654312611)
--(axis cs:12000,0.966521290083726)
--(axis cs:24000,1.00435026784738)
--(axis cs:36000,0.932912302533786)
--(axis cs:48000,0.932445608874162)
--(axis cs:60000,1.26287882602215)
--(axis cs:72000,2.15145396868388)
--(axis cs:84000,2.6549033610026)
--(axis cs:96000,8.27388117281596)
--(axis cs:108000,28.4651692875226)
--(axis cs:120000,52.8148584591548)
--(axis cs:132000,60.6141692867279)
--(axis cs:144000,80.1024012680054)
--(axis cs:156000,84.4604769083659)
--(axis cs:168000,88.4866076533)
--(axis cs:180000,85.0167690658569)
--(axis cs:192000,79.9824682108561)
--(axis cs:204000,82.2037711970011)
--(axis cs:216000,82.2846532465617)
--(axis cs:228000,86.3395181121826)
--(axis cs:240000,85.5953360188802)
--(axis cs:252000,80.8599378534953)
--(axis cs:264000,85.0106308695475)
--(axis cs:276000,88.1110554962158)
--(axis cs:288000,81.7496233393351)
--(axis cs:288000,87.842561726888)
--(axis cs:288000,87.842561726888)
--(axis cs:276000,95.5240037155151)
--(axis cs:264000,92.2460413970947)
--(axis cs:252000,87.2063626505534)
--(axis cs:240000,91.739370241801)
--(axis cs:228000,93.1681739425659)
--(axis cs:216000,87.9079510294596)
--(axis cs:204000,90.2273802439372)
--(axis cs:192000,87.8411529642741)
--(axis cs:180000,90.802943069458)
--(axis cs:168000,93.9773802973429)
--(axis cs:156000,94.2027341613769)
--(axis cs:144000,89.8478268508911)
--(axis cs:132000,84.8754140650431)
--(axis cs:120000,77.9275668433507)
--(axis cs:108000,51.3851270041466)
--(axis cs:96000,17.3506291985512)
--(axis cs:84000,5.12177464898427)
--(axis cs:72000,3.21762132986387)
--(axis cs:60000,1.50324456230799)
--(axis cs:48000,0.984190978348255)
--(axis cs:36000,0.96343360632658)
--(axis cs:24000,1.03399463154872)
--(axis cs:12000,0.99281308277448)
--(axis cs:0,0.905502462764581)
--cycle;

\path [fill=color3, fill opacity=0.2]
(axis cs:0,0.883225904961427)
--(axis cs:0,0.859483804682891)
--(axis cs:12000,0.943223104178905)
--(axis cs:24000,0.969433712323507)
--(axis cs:36000,0.92955630671978)
--(axis cs:48000,0.951388427535693)
--(axis cs:60000,1.62061495105426)
--(axis cs:72000,2.8163078562816)
--(axis cs:84000,5.05053282674154)
--(axis cs:96000,15.0860100111961)
--(axis cs:108000,34.3169784334501)
--(axis cs:120000,59.097813469251)
--(axis cs:132000,72.5467786814372)
--(axis cs:144000,89.129275255839)
--(axis cs:156000,79.964326944987)
--(axis cs:168000,83.2830441665649)
--(axis cs:180000,78.9848361422221)
--(axis cs:192000,83.2655541788737)
--(axis cs:204000,81.9271672312419)
--(axis cs:216000,81.2018585179647)
--(axis cs:228000,86.6242374827067)
--(axis cs:240000,85.5693049367269)
--(axis cs:252000,79.4840527801514)
--(axis cs:264000,84.4132767105103)
--(axis cs:276000,88.8119815724691)
--(axis cs:288000,78.2512013880412)
--(axis cs:288000,85.3612252756755)
--(axis cs:288000,85.3612252756755)
--(axis cs:276000,94.190894955953)
--(axis cs:264000,91.9075223083496)
--(axis cs:252000,86.6295628356934)
--(axis cs:240000,92.5088448435466)
--(axis cs:228000,93.7518971964518)
--(axis cs:216000,87.6470216929118)
--(axis cs:204000,88.5242448628744)
--(axis cs:192000,91.7699772796631)
--(axis cs:180000,89.1489678777059)
--(axis cs:168000,91.7363477834066)
--(axis cs:156000,89.1528789812724)
--(axis cs:144000,95.7775318094889)
--(axis cs:132000,89.3075712280273)
--(axis cs:120000,83.1186680908203)
--(axis cs:108000,52.1497931283315)
--(axis cs:96000,28.3462480015755)
--(axis cs:84000,7.12110461870829)
--(axis cs:72000,3.57770788594087)
--(axis cs:60000,1.94961797014872)
--(axis cs:48000,1.01305635631084)
--(axis cs:36000,0.946282773633798)
--(axis cs:24000,0.990201221803824)
--(axis cs:12000,0.96629344367981)
--(axis cs:0,0.883225904961427)
--cycle;

\path [fill=color4, fill opacity=0.2]
(axis cs:0,0.88139027150472)
--(axis cs:0,0.86060740228494)
--(axis cs:12000,0.897389663040638)
--(axis cs:24000,0.978434031466643)
--(axis cs:36000,0.92770034134388)
--(axis cs:48000,0.852181868771712)
--(axis cs:60000,0.981657914261023)
--(axis cs:72000,1.54076967740059)
--(axis cs:84000,2.2723570318222)
--(axis cs:96000,4.44641343259811)
--(axis cs:108000,20.0529719130993)
--(axis cs:120000,46.6824677054087)
--(axis cs:132000,51.1384569028219)
--(axis cs:144000,79.3364067014058)
--(axis cs:156000,88.6471265767415)
--(axis cs:168000,89.6331681696574)
--(axis cs:180000,86.6542479527791)
--(axis cs:192000,88.5627005742391)
--(axis cs:204000,92.2642791213989)
--(axis cs:216000,90.4127185821533)
--(axis cs:228000,88.1311870473226)
--(axis cs:240000,94.8828336079915)
--(axis cs:252000,91.9616381352742)
--(axis cs:264000,92.361337735494)
--(axis cs:276000,91.751135559082)
--(axis cs:288000,92.3945277201335)
--(axis cs:288000,97.2609303792318)
--(axis cs:288000,97.2609303792318)
--(axis cs:276000,97.8016079406738)
--(axis cs:264000,97.0839737141927)
--(axis cs:252000,96.2726387481689)
--(axis cs:240000,98.9766815185547)
--(axis cs:228000,95.7259859466553)
--(axis cs:216000,95.8932098693848)
--(axis cs:204000,95.7680113601685)
--(axis cs:192000,94.0881837514241)
--(axis cs:180000,94.0629704818726)
--(axis cs:168000,96.1677467880249)
--(axis cs:156000,96.9878536148071)
--(axis cs:144000,92.5953575236003)
--(axis cs:132000,69.496928565979)
--(axis cs:120000,70.3417202097575)
--(axis cs:108000,43.95457969745)
--(axis cs:96000,9.91770686904589)
--(axis cs:84000,3.00842940425873)
--(axis cs:72000,1.85459575299422)
--(axis cs:60000,1.12380301326513)
--(axis cs:48000,0.874550246298313)
--(axis cs:36000,0.966980618814627)
--(axis cs:24000,0.996892310341199)
--(axis cs:12000,0.925474374175072)
--(axis cs:0,0.88139027150472)
--cycle;

\path [fill=color1, fill opacity=0.2]
(axis cs:0,0.871857167919477)
--(axis cs:0,0.860776517748833)
--(axis cs:12000,0.914289421141148)
--(axis cs:24000,0.95932713073492)
--(axis cs:36000,0.921242848654588)
--(axis cs:48000,0.889719801882903)
--(axis cs:60000,1.11461558610201)
--(axis cs:72000,2.03723392283916)
--(axis cs:84000,2.74784810765584)
--(axis cs:96000,5.64052388978005)
--(axis cs:108000,25.4286829140981)
--(axis cs:120000,46.3380833806197)
--(axis cs:132000,62.919475666364)
--(axis cs:144000,84.2873067881266)
--(axis cs:156000,88.0819868774414)
--(axis cs:168000,89.7528482055664)
--(axis cs:180000,87.7862248484294)
--(axis cs:192000,90.0258816553752)
--(axis cs:204000,89.1812913309733)
--(axis cs:216000,90.3935373102824)
--(axis cs:228000,92.2884673461914)
--(axis cs:240000,93.388493057251)
--(axis cs:252000,93.7499954223633)
--(axis cs:264000,92.2617032750448)
--(axis cs:276000,91.5755005340576)
--(axis cs:288000,91.516523958842)
--(axis cs:288000,96.4274523620605)
--(axis cs:288000,96.4274523620605)
--(axis cs:276000,94.7065149205526)
--(axis cs:264000,97.2292050933838)
--(axis cs:252000,98.5404088567098)
--(axis cs:240000,98.0282959187825)
--(axis cs:228000,97.2676238708496)
--(axis cs:216000,96.3480383758545)
--(axis cs:204000,95.355135093689)
--(axis cs:192000,95.1367835184733)
--(axis cs:180000,93.917110748291)
--(axis cs:168000,96.0188167114258)
--(axis cs:156000,93.0501984481811)
--(axis cs:144000,92.1519879277547)
--(axis cs:132000,82.9547321907679)
--(axis cs:120000,72.8582103347778)
--(axis cs:108000,46.2636419054667)
--(axis cs:96000,8.23124446709951)
--(axis cs:84000,3.76994888810317)
--(axis cs:72000,2.49293602403005)
--(axis cs:60000,1.29158747235934)
--(axis cs:48000,0.923121931711833)
--(axis cs:36000,0.942605700612068)
--(axis cs:24000,0.975823855757713)
--(axis cs:12000,0.930647446235021)
--(axis cs:0,0.871857167919477)
--cycle;
\addplot [very thick, color0]
table {%
0 0.901905838777622
12000 0.962037223396699
24000 0.999146792906523
36000 0.951815141230822
48000 1.01409313456416
60000 1.52559088852008
72000 1.94680462131103
84000 2.40893825781345
96000 7.38685420749982
108000 24.3069916057189
120000 41.4718132746061
132000 73.4398098092397
144000 78.9852495770772
156000 86.7667355789185
168000 82.8747308479309
180000 82.1588153673808
192000 76.738635986201
204000 81.3939222195943
216000 85.7553117299398
228000 88.1452280423482
240000 80.7764732204437
252000 85.3530803944906
264000 83.6140804400126
276000 89.553214999644
288000 90.7340763198852
};
\addlegendentry{Without Transfer}

\addplot [very thick, color2]
table {%
0 0.895802424434821
12000 0.979414781109492
24000 1.01822760201891
36000 0.948577052050829
48000 0.958733204432329
60000 1.37656532059113
72000 2.69477556509972
84000 3.70319500122468
96000 12.681743108209
108000 39.698486129268
120000 65.7639903133074
132000 73.354210416158
144000 85.2202657460531
156000 89.3859767936707
168000 91.2461878267924
180000 87.9476788393656
192000 83.8790335863749
204000 86.2696961906433
216000 85.1066269032796
228000 89.6051114163717
240000 88.9006399584452
252000 84.0515214922587
264000 88.5175182601929
276000 91.9629084582011
288000 84.9264506701151
};
\addlegendentry{Auxiliary Tasks}
\addplot [very thick, color3]
table {%
0 0.87093829579552
12000 0.954461890629927
24000 0.979961372739077
36000 0.937324230092764
48000 0.981206701046228
60000 1.7854477313598
72000 3.20104330263138
84000 6.1046913353761
96000 21.3634891139507
108000 43.327407117637
120000 72.0688820735614
132000 81.3282992485682
144000 92.6493984835307
156000 84.6867981441498
168000 87.5211837226868
180000 84.0500229882558
192000 87.9100887560527
204000 85.2127063667297
216000 84.5571378092448
228000 90.3248475883484
240000 89.2384677899679
252000 83.1312809277852
264000 88.3535327097575
276000 91.4740384435018
288000 82.0333514373779
};
\addlegendentry{Transfer $\hat{P}$}
\addplot [very thick, color4]
table {%
0 0.870934468992551
12000 0.91124678465724
24000 0.987519289684296
36000 0.946019084688028
48000 0.862140349630515
60000 1.04412622292638
72000 1.69992994931539
84000 2.62818058981101
96000 6.86198756399155
108000 31.0493813880126
120000 58.3522651816686
132000 60.6335226650238
144000 86.6576751715342
156000 93.0558447364807
168000 93.0219746119181
180000 90.563238965861
192000 91.3316284627279
204000 93.9608005152384
216000 93.4840740112305
228000 91.8142399416606
240000 97.0505021827698
252000 94.1016805768331
264000 94.8691317276001
276000 95.0102150054932
288000 94.8567907740275
};
\addlegendentry{Transfer $\hat{R}$}

\addplot [ultra thick, color1]
table {%
0 0.866553709467252
12000 0.923119978447755
24000 0.967357336485386
36000 0.931596616234382
48000 0.906495008933544
60000 1.20265740027428
72000 2.25544934192499
84000 3.25691274438302
96000 6.92106072714329
108000 35.5233276727041
120000 61.0897104690393
132000 73.4303872014682
144000 88.4033905367533
156000 90.5667071271261
168000 92.9890858202616
180000 90.9145073636373
192000 92.7364072451274
204000 92.2308878267924
216000 93.5267481755574
228000 94.9024087539673
240000 95.6501070798238
252000 96.3939725914001
264000 94.8362269927978
276000 93.2026348889669
288000 94.1352893002828
};
\addlegendentry{Our Method}
\end{axis}

\end{tikzpicture}

%% file: figs/coeff_compare.tex
\begin{tikzpicture}[scale=0.8]

\definecolor{color0}{rgb}{0.12156862745098,0.466666666666667,0.705882352941177}
\definecolor{color1}{rgb}{1,0.498039215686275,0.0549019607843137}

\begin{axis}[
legend cell align={left},
legend style={fill opacity=0.8, draw opacity=1, text opacity=1, draw=white!80!black},
tick align=outside,
tick pos=left,
x grid style={white!69.0196078431373!black},
xmajorgrids,
xmin=-0.15, xmax=25.15,
xtick style={color=black},
xlabel={Regularization Weight $\lambda$},
y grid style={white!69.0196078431373!black},
ymajorgrids,
ymin=42.51114025, ymax=83.5967214166667,
ytick style={color=black},
ylabel={Average Return},
]
\path [fill=color0, fill opacity=0.2]
(axis cs:1,68.8835166666667)
--(axis cs:1,48.96444)
--(axis cs:2,44.9647116666667)
--(axis cs:4,58.59723)
--(axis cs:6,50.0959033333333)
--(axis cs:8,56.99296)
--(axis cs:10,56.087765)
--(axis cs:12,44.3786666666667)
--(axis cs:14,49.1427466666667)
--(axis cs:16,60.9039216666667)
--(axis cs:18,66.04868)
--(axis cs:20,56.6198883333333)
--(axis cs:22,50.0327633333333)
--(axis cs:24,53.5275133333333)
--(axis cs:24,70.0231633333333)
--(axis cs:24,70.0231633333333)
--(axis cs:22,67.85762)
--(axis cs:20,74.5615133333333)
--(axis cs:18,81.729195)
--(axis cs:16,80.5915783333333)
--(axis cs:14,69.4391683333333)
--(axis cs:12,63.1279883333333)
--(axis cs:10,65.3951133333333)
--(axis cs:8,69.693275)
--(axis cs:6,63.69618)
--(axis cs:4,71.870995)
--(axis cs:2,56.038055)
--(axis cs:1,68.8835166666667)
--cycle;

\addplot [very thick, color0, mark=*, mark size=3, mark options={solid}]
table {%
1 59.481312
2 50.6649345
4 65.1702346666667
6 57.2592315
8 63.306004
10 60.7404753333333
12 53.886415
14 58.8863686666667
16 70.7575176666667
18 74.4304055
20 65.4625285
22 59.0641981666667
24 61.9044141666667
};
\addlegendentry{Transfer with Various $\lambda$'s}
\addplot [ultra thick, color1]
table {%
1 45.334
2 45.334
4 45.334
6 45.334
8 45.334
10 45.334
12 45.334
14 45.334
16 45.334
18 45.334
20 45.334
22 45.334
24 45.334
};
\addlegendentry{Without Transfer}
\end{axis}

\end{tikzpicture}